\documentclass[11pt, a4paper, logo]{expevolve}

\usepackage[authoryear, sort&compress, round]{natbib}
\usepackage{xurl}

\usepackage[most]{tcolorbox}
\usepackage{subfigure}
\usepackage{tikz}
\tikzset{every label/.append style={font=\footnotesize}} % label size used by the appendix diagrams
\usepackage{mathtools}
\usepackage{multirow}
\usepackage{algorithm,algorithmic}
\usepackage{xspace}

\definecolor{kleinblue}{rgb}{0,0.18,0.65}
\hypersetup{colorlinks=true,citecolor=kleinblue, linkcolor=kleinblue, urlcolor=kleinblue}

\newtheorem{theorem}{Theorem}[section]
\newtheorem{definition}[theorem]{Definition}

\def\eqref#1{equation~\ref{#1}}
\newcommand{\E}{\mathbb{E}}

\newenvironment{myquotation}{\setlength{\leftmargini}{0em}\quotation}{\endquotation}

\usepackage[hyperpageref]{backref}

\renewcommand*{\backrefalt}[4]{
  \ifcase #1 \relax
  \or
    (Cited on page #2)
  \else
    (Cited on pages #2)
  \fi
}

\newcommand{\ours}{\texttt{ExpEvolve}\xspace}
\newcommand{\coat}{\texttt{COAT}\xspace}
\newcommand{\shinka}{\texttt{ShinkaEvolve}\xspace}
\newcommand{\alevolve}{\texttt{AlphaEvolve}\xspace}
\newcommand{\scs}{\texttt{SCS}\xspace}
\newcommand{\openrsi}{\texttt{OpenRSI}\xspace}

\usetikzlibrary{shapes, positioning, arrows.meta, fit, calc, backgrounds, shadows}

\definecolor{boxMem}{RGB}{225,245,254}
\definecolor{boxAct}{RGB}{255,249,196}
\definecolor{boxEnv}{RGB}{232,245,233}
\definecolor{boxWorld}{RGB}{245,245,245}
\definecolor{lineColor}{RGB}{80,80,80}
\definecolor{memBlue}{RGB}{220,235,255}
\definecolor{progGreen}{RGB}{220,255,220}
\definecolor{obsGray}{RGB}{240,240,240}
\definecolor{borderGray}{RGB}{100,100,100}

\DeclareFontFamily{T1}{optimistic}{}
\DeclareFontShape{T1}{optimistic}{m}{n}{<-> s * [0.88] assets/optimistic}{}
\DeclareFontShape{T1}{optimistic}{b}{n}{<-> s * [0.88] assets/optimistic}{}
\pdfmapline{+optimistic < assets/Optimistic.ttf <T1-WGL4.enc}
\newcommand{\wordmarkfont}{\fontfamily{optimistic}\selectfont}

\title{ExpEvolve: Structuring and Reusing Experience for Open-Ended Discovery}

\author[1,2]{Yongqiang Chen\textsuperscript{*$\dagger$}}
\author[3]{Chenxi Liu\textsuperscript{*}}
\author[4]{Deyu Zou\textsuperscript{*}}
\author[1]{Zhenhao Chen}
\author[5,1]{Tongliang Liu}
\author[3]{Bo Han}
\author[1,2]{Kun Zhang}

\affil[1]{MBZUAI}
\affil[2]{Carnegie Mellon University}
\affil[3]{TMLR Group, Hong Kong Baptist University}
\affil[4]{The Chinese University of Hong Kong}
\affil[5]{SAIC Centre, The University of Sydney}

\firstpagenote{This paper is an extended version of \emph{CausalEvolve} (\href{https://arxiv.org/abs/2603.14575}{arXiv:2603.14575}).}
\authornote{\textsuperscript{*}Core contributors. \textsuperscript{$\dagger$}Project lead.\\[2pt]
   {\color{gray}\texttt{yqchen24@gmail.com}}}

\begin{abstract}
Evolutionary coding agents such as AlphaEvolve are one of the notable successes in using Large Language Models (LLMs) to build AI Scientists. These agents tackle open-ended scientific problems by iteratively improving and evolving programs, leveraging the prior knowledge and reasoning capabilities of LLMs. Despite their success, mutations over programs are mostly low-level upon surface heuristics, without a persistent unit of control. In contrast, science compounds through ideas. In this work, we propose ExpEvolve, which injects idea-level extraction and control into evolution: it discovers and leverages the high-level factors from programs that explain scoring differences among different programs. We characterize theoretical conditions on when discovery with idea factors improves sample efficiency. We also build a practical framework to show that the extracted high-level idea factors are helpful for guiding exploitation from experience and exploration when discovery comes to a plateau.
Across $11$ open-ended mathematical tasks, \ours significantly improves the final performance and sample efficiency up to $5.0\times$. The improvements also transfer to $4$ other LLMs and to $6$ machine learning tasks from MLS-Bench. Analysis shows that \ours can discover idea-level factors to provide better exploration and exploitation for evolutionary discovery.
\end{abstract}

\begin{document}

\maketitle

\section{Introduction}
Large language models (LLMs) are now widely used to build agents that operate on increasingly long-horizon tasks~\citep{plaat2025agentic}. One of the most ambitious applications is AI scientist agents~\citep{ZHENG2025FromAT}, in which LLMs automate the scientific discovery process, including literature surveys~\citep{asai2026openscholar}, hypothesis generation~\citep{wang2024scimon,yang2024large}, data-driven analysis~\citep{majumder2025discoverybench}, experiment design~\citep{Li2025CanLL}, or even the whole discovery workflow~\citep{lu2024ai,yamada2025ai,gottweis2026coscientist,Mitchener2025KosmosAA}. LLM-driven discovery has shown strong results in a variety of domains, such as mathematics~\citep{openai2026unitdistance,openai2026tenadvances}, de novo protein design~\citep{anthropic2026protein}, algorithm discovery~\citep{novikov2025alphaevolve,lange2025shinkaevolve}, and the research that develops LLMs themselves, an early form of recursive self-improvement (RSI)~\citep{openai2026researchacceleration,favaro2026rsi,lyu2026mlsbench,karpathy2026autoresearch}.
At the heart of the discovery loop is the \textit{iteration over scientific hypotheses}, where scientific ideas compound throughout the evaluation of promising hypotheses~\citep{wallace2026conjecture}.
Therefore, it is essential to \textit{reuse the discovery experience} to generate high-quality hypotheses, e.g., avoiding getting stuck in repetitive mistakes or suboptimal solutions and encoding promising exploratory proposals.

In this work, we study experience reuse in open-ended discovery with evolutionary coding agents, where hypotheses are represented as \textit{solution programs} to the underlying scientific problem, and experiences are accumulated as \textit{evaluated programs and their scores} given by an explicit evaluator~\citep{novikov2025alphaevolve}.
Typically, existing evolutionary coding agents such as \shinka directly incorporate successful programs or rely on LLMs to summarize reusable insights and incorporate them as memory for future mutations with LLMs~\citep{lange2025shinkaevolve,Gideoni2026SimpleBA,karpathy2026autoresearch}.
However, program mutation contains an open-ended search space including low-level details such as hyperparameters and high-level insights such as optimization strategies. LLM-driven guidance is prone to focusing on low-level details~\citep{karpathy2026autoresearch,pelleriti2026evotrace}, which may not bring genuine new scientific discoveries.
In contrast, human scientists tend to recognize the high-level contrasts among programs, accumulate and exploit scientific ideas from experience~\citep{sci_revolution,sep-abduction}.
The gap raises a challenging research question:
\begin{myquotation}\centering
    \textit{How can we reuse experience for evolutionary discovery like human scientists?}
\end{myquotation}

To tackle the question, we formulate evolutionary discovery as a Partially Observable Markov Decision Process
(POMDP)~\citep{kaelbling1998planning} (Sec.~\ref{sec:formulation}).
In each turn, the agent needs to propose meaningful solutions to uncover the underlying causal mechanism for solving the scientific problem, i.e., maximizing the evaluation score~\citep{wallace1981causality,ClarkOutline}.
Under the formulation, we introduce the notion of \textit{ideas} as high-level mechanistic properties of the programs that determine the evaluation scores, e.g., the construction family or the optimization strategy (Def.~\ref{def:idea}).
Without uncovering the ideas, the agent has to treat programs as unrelated and cannot tell whether a revision changes an idea or only a low-level detail, so discovery degenerates into exhaustive search (Thm.~\ref{thm:structured_reuse}).
Meanwhile, the identification of ideas also requires \textit{sufficient exploration} to contrast programs with and without each idea (Thm.~\ref{thm:degenerated_exploration}). Hence, an effective discovery agent needs to extract ideas from experience, and keep exploring programs where the ideas vary.

\begin{figure}
    \centering
    \includegraphics[width=\linewidth]{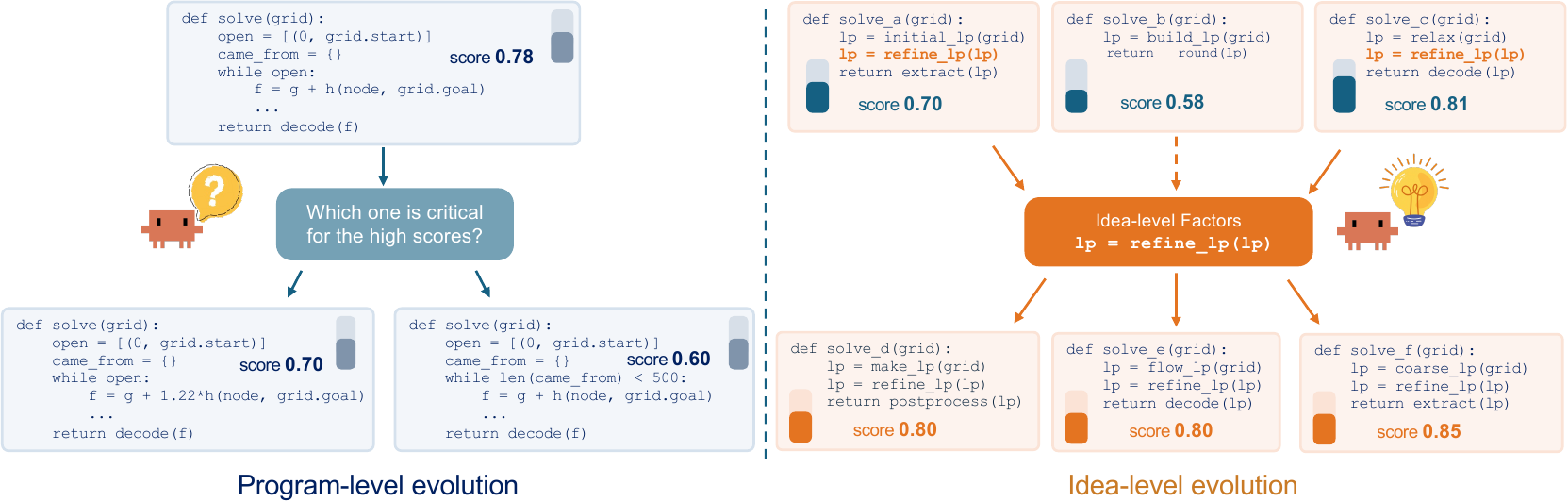}
    \caption{Comparison between experience reuse at program-level and idea-level. All programs solve the same grid task. \textbf{Left:} Experience reuse at program-level is prone to low-level edits. Program-level evolution derives new programs from a high-scoring parent, such as changing the heuristic weight or bounding the search loop. Both edited programs score lower. \textbf{Right:} Experience reuse at idea-level compares structurally different programs through a shared procedure-level factor, \texttt{lp = refine\_lp(lp)}. Programs that contain the factor score higher than the program that omits it. This contrast singles out the factor, which is then inserted into new programs of different structures.}
    \label{fig:illustration}
\end{figure}

Motivated by the analysis, we propose \ours, which recognizes, structures, and leverages the high-level idea factors from discovery experience (Fig.~\ref{fig:illustration}).
Specifically, \ours realizes the discovery of ideas via identifying high-level factors that help explain the program scoring differences. Built upon \coat~\citep{coat2025discoveringreasoningcausalityhidden}, the discovered idea factors are expected to predict which program performs better given any pair of programs in the experience.
Then, the admitted idea factors will be incorporated as an action of a multi-armed bandit that controls the exploration and exploitation of different ideas during discovery.
To enable sufficient exploration, \ours also allocates slots for random exploration~\citep{Gideoni2026SimpleBA}, and idea-guided exploration when the evolution comes to a plateau. The idea-guided exploration enforces LLMs to explore diverse ideas other than those widely used in top-scoring programs.
The idea-based control in \ours allows efficient exploitation of promising ideas while continuing to explore programs with varied ideas.

We evaluate \ours on $11$ open-ended mathematical discovery tasks with a budget of $200$ evaluated programs, using \scs as the backbone and the main baseline (Sec.~\ref{sec:exp}). Under two configurations that split the budget into $20$ generations of $10$ programs and $10$ generations of $20$ programs, \ours improves the final performance over \scs and \shinka in most settings with smaller budgets up to $5.0\times$.
Further analyses show that the discovered idea factors order the archived programs better than their code length on all $11$ tasks. We further integrate \ours into \openrsi~\citep{yang2026frontis} on six tasks from MLS-Bench~\citep{lyu2026mlsbench}, where it improves the mean final best-so-far score of \openrsi on all six tasks.

\section{Related Work}
\label{sec:related_work}
In this section, we discuss the closest related works and defer a more detailed one in Appendix~\ref{appdx:related_work}.

\textbf{AI Scientist agents.} With the significant advancement in LLM capacity and the development of agentic systems, building LLM-based AI scientists has gained increasing attention~\citep{Zhou2025AutonomousAF}.
One research line automates part or the whole of the pipelines in scientific activities through agentic workflows~\citep{lu2026aiscientist,yamada2025ai,gottweis2026coscientist,ghareeb2026robin,Mitchener2025KosmosAA,internscience2026internagent,lyu2026evoscientist}.
Evolutionary coding agents are one of the most representative AI scientist agents that integrate LLMs into evolutionary search over programs scored by an explicit evaluator~\citep{romeraparedes2024funsearch,novikov2025alphaevolve}.
AI scientist agents, including evolutionary coding agents, have made impact in a variety of domains, such as mathematics~\citep{openai2026tenadvances}, materials science~\citep{ghafarollahi2025sciagents},  biology~\citep{anthropic2026protein}, and the development of LLMs themselves, e.g., accelerating LLM training and research toward recursive self-improvement~\citep{karpathy2026autoresearch,lyu2026mlsbench,openai2026researchacceleration,favaro2026rsi,zhang2025dgm}.
As generation scales up, validating the generated hypotheses becomes the main cost~\citep{wallace2026conjecture}, and ideas that look promising often degrade once executed~\citep{si2026gap}.
The expanding paradigm of LLM-driven scientific discovery makes the reuse of experience \textit{increasingly more essential}, which is central to our work.

\textbf{Experience reuse in evolutionary coding agents.}
Existing evolutionary coding agents typically reuse experience via conditioning on evaluated programs~\citep{novikov2025alphaevolve,Gideoni2026SimpleBA} or summarizing textual insights from programs into memory for future mutations~\citep{lange2025shinkaevolve,liu2024eoh,ye2024reevo,volkov2026evomem}.
The useful experience can also be exploited for future search control, e.g., selection inspirations across diverse strategy clusters~\citep{lange2025shinkaevolve,luo2026seaevo}.
Different from existing works that concentrate on the unit of experience as texts or programs, \ours discovers and exploits high-level idea factors to guide program mutations.

\section{Theoretical Formulation of Evolutionary Discovery}
\label{sec:formulation}

In this section, we formalize evolutionary discovery and introduce the notion of ideas.
\definecolor{BlueFill}{HTML}{EAF1F9}\definecolor{BlueLine}{HTML}{2B6CB0}
\definecolor{WarmFill}{HTML}{FDF3E7}\definecolor{WarmLine}{HTML}{B4560F}
\definecolor{GreenFill}{HTML}{E9F6EE}\definecolor{GreenLine}{HTML}{1F6B45}
\definecolor{GrayFill}{HTML}{F1F4F8}\definecolor{GrayLine}{HTML}{4A5568}
\definecolor{Hairline}{HTML}{C6CEDA}\definecolor{PanelBg}{HTML}{FBFCFE}
\definecolor{TextColor}{HTML}{171E29}
\providecommand{\OvPanelFont}{\fontsize{10.1}{12}\selectfont\bfseries}
\providecommand{\OvGroupFont}{\fontsize{7.8}{9.2}\selectfont\bfseries}
\providecommand{\OvTitleFont}{\fontsize{7.8}{9.4}\selectfont\bfseries}
\providecommand{\OvBodyFont}{\fontsize{6.6}{7.9}\selectfont}
\providecommand{\OvEdgeFont}{\fontsize{6.4}{7.5}\selectfont\itshape}
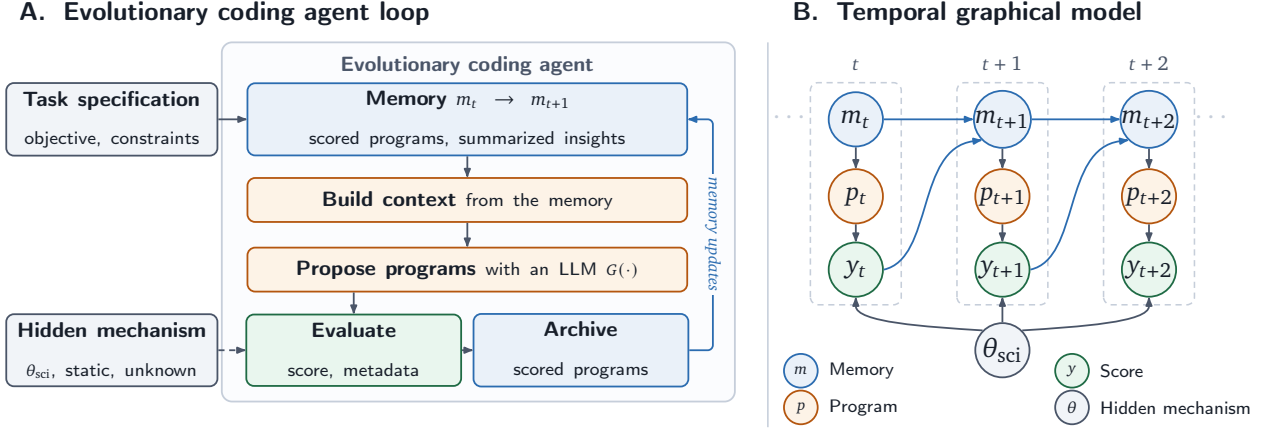
\begin{figure}[t]
\centering
{\hyphenpenalty=10000\exhyphenpenalty=10000
\resizebox{\linewidth}{!}{%
\begin{tikzpicture}[x=1cm,y=1cm,
  every node/.style={text=TextColor, font=\sffamily},
  box/.style={rounded corners=3pt, line width=0.7pt, inner xsep=4pt, inner ysep=3pt, align=center},
  memBox/.style={box, draw=BlueLine, fill=BlueFill},
  actBox/.style={box, draw=WarmLine, fill=WarmFill},
  evalBox/.style={box, draw=GreenLine, fill=GreenFill},
  extBox/.style={box, draw=GrayLine, fill=GrayFill},
  head/.style={-{Latex[length=1.7mm,width=1.15mm,round]}},
  flow/.style={head, line width=0.7pt, draw=GrayLine, rounded corners=2pt},
  feedback/.style={head, line width=0.7pt, draw=BlueLine, rounded corners=2pt},
  constraint/.style={head, line width=0.7pt, draw=GrayLine, dash pattern=on 2.4pt off 1.8pt},
  slice/.style={draw=Hairline, dash pattern=on 2.4pt off 1.8pt, rounded corners=3pt, line width=0.6pt},
  circleNode/.style={circle, minimum size=7.4mm, inner sep=0pt, line width=0.7pt, align=center},
  memNode/.style={circleNode, draw=BlueLine, fill=BlueFill},
  progNode/.style={circleNode, draw=WarmLine, fill=WarmFill},
  outNode/.style={circleNode, draw=GreenLine, fill=GreenFill},
  thetaNode/.style={circleNode, draw=GrayLine, fill=GrayFill},
  elab/.style={font=\OvEdgeFont, text=GrayLine, fill=white, inner sep=1.2pt}
]
\node[anchor=west] at (0.05, 9.50) {\OvPanelFont A.\hspace{3pt} Evolutionary coding agent loop};
\node[anchor=west] at (10.60, 9.50) {\OvPanelFont B.\hspace{3pt} Temporal graphical model};
\fill[PanelBg, rounded corners=4pt] (2.95, 4.20) rectangle (9.95, 9.10);
\draw[draw=Hairline, line width=0.7pt, rounded corners=4pt] (2.95, 4.20) rectangle (9.95, 9.10);
\node[anchor=north, text=GrayLine] at (6.30, 9.05) {\OvGroupFont Evolutionary coding agent};
\node[extBox, minimum width=2.80cm, minimum height=0.95cm, text width=2.60cm] (task) at (1.45, 8.05) {
  {\OvTitleFont Task specification}\\[1.5pt]
  {\OvBodyFont objective, constraints}
};
\node[extBox, minimum width=2.80cm, minimum height=0.95cm, text width=2.60cm] (thetaA) at (1.45, 4.90) {
  {\OvTitleFont Hidden mechanism}\\[1.5pt]
  {\OvBodyFont $\theta_{\mathrm{sci}}$, static, unknown}
};
\node[memBox, minimum width=6.00cm, minimum height=0.95cm, text width=5.70cm] (memory) at (6.30, 8.05) {
  {\OvTitleFont Memory $m_t \rightarrow m_{t+1}$}\\[1.5pt]
  {\OvBodyFont scored programs, summarized insights}
};
\node[actBox, minimum width=6.00cm, minimum height=0.60cm, text width=5.70cm] (context) at (6.30, 6.95) {
  {\OvTitleFont Build context} {\OvBodyFont from the memory}
};
\node[actBox, minimum width=6.00cm, minimum height=0.60cm, text width=5.70cm] (propose) at (6.30, 6.00) {
  {\OvTitleFont Propose programs} {\OvBodyFont with an LLM $G(\cdot)$}
};
\node[evalBox, minimum width=2.90cm, minimum height=0.95cm, text width=2.65cm] (evaluate) at (4.75, 4.90) {
  {\OvTitleFont Evaluate}\\[1.5pt]
  {\OvBodyFont score, metadata}
};
\node[memBox, minimum width=2.90cm, minimum height=0.95cm, text width=2.65cm] (archive) at (7.85, 4.90) {
  {\OvTitleFont Archive}\\[1.5pt]
  {\OvBodyFont scored programs}
};
\draw[flow] (memory.south) -- (context.north);
\draw[flow] (context.south) -- (propose.north);
\draw[flow] (propose.south -| evaluate.north) -- (evaluate.north);
\draw[flow] (evaluate.east) -- (archive.west);
\draw[flow] (task.east) -- (memory.west);
\draw[constraint] (thetaA.east) -- (evaluate.west);
\draw[feedback] (archive.east) -- (9.65, 4.90) -- (9.65, 8.05) -- (memory.east);
\node[elab, text=BlueLine, rotate=-90, fill=PanelBg] at (9.65, 6.48) {memory updates};
\draw[draw=Hairline, line width=0.6pt] (10.40, 3.85) -- (10.40, 9.10);
\foreach \x/\step in {11.60/$t$, 13.60/$t+1$, 15.60/$t+2$}{
  \draw[slice] (\x-0.62, 5.52) rectangle (\x+0.62, 8.55);
  \node[anchor=south, text=GrayLine] at (\x, 8.58) {\OvTitleFont \step};
}
\node[memNode]  (m0) at (11.60, 8.05) {$m_t$};
\node[progNode] (p0) at (11.60, 7.00) {$p_t$};
\node[outNode]  (y0) at (11.60, 6.00) {$y_t$};
\node[memNode]  (m1) at (13.60, 8.05) {$m_{t+1}$};
\node[progNode] (p1) at (13.60, 7.00) {$p_{t+1}$};
\node[outNode]  (y1) at (13.60, 6.00) {$y_{t+1}$};
\node[memNode]  (m2) at (15.60, 8.05) {$m_{t+2}$};
\node[progNode] (p2) at (15.60, 7.00) {$p_{t+2}$};
\node[outNode]  (y2) at (15.60, 6.00) {$y_{t+2}$};
\node[thetaNode] (thetaB) at (13.60, 4.90) {$\theta_{\mathrm{sci}}$};
\foreach \u/\v in {m0/p0, p0/y0, m1/p1, p1/y1, m2/p2, p2/y2}{\draw[flow] (\u) -- (\v);}
\draw[feedback] (m0.east) -- (m1.west);
\draw[feedback] (m1.east) -- (m2.west);
\draw[feedback] (y0.east) .. controls (12.66,6.15) and (12.42,7.65) .. (m1.south west);
\draw[feedback] (y1.east) .. controls (14.66,6.15) and (14.42,7.65) .. (m2.south west);
\draw[flow] (thetaB.north west) .. controls (12.55,5.22) and (11.60,5.28) .. (y0.south);
\draw[flow] (thetaB.north) -- (y1.south);
\draw[flow] (thetaB.north east) .. controls (14.65,5.22) and (15.60,5.28) .. (y2.south);
\node[text=Hairline] at (10.72, 8.05) {$\cdots$};
\node[text=Hairline] at (16.48, 8.05) {$\cdots$};
\node[memNode, minimum size=4.6mm, font=\OvBodyFont] at (10.85, 4.62) {$m$};
\node[anchor=west] at (11.10, 4.62) {\OvBodyFont Memory};
\node[progNode, minimum size=4.6mm, font=\OvBodyFont] at (10.85, 4.12) {$p$};
\node[anchor=west] at (11.10, 4.12) {\OvBodyFont Program};
\node[outNode, minimum size=4.6mm, font=\OvBodyFont] at (14.55, 4.62) {$y$};
\node[anchor=west] at (14.80, 4.62) {\OvBodyFont Score};
\node[thetaNode, minimum size=4.6mm, font=\OvBodyFont] at (14.55, 4.12) {$\theta$};
\node[anchor=west] at (14.80, 4.12) {\OvBodyFont Hidden mechanism};
\end{tikzpicture}%
}}
\caption{Evolutionary discovery as a POMDP. \textbf{(A)} In each iteration, an evolutionary coding agent builds the proposal context from its memory, e.g., parent and inspiration programs or summarized insights, and generates candidate programs with an LLM $G$. The evaluator scores every program, whose outcome depends on the hidden mechanism $\theta_{\mathrm{sci}}$, and the scored programs enter the archive and update the memory. \textbf{(B)} The corresponding temporal graphical model. The memory $m_t$ guides the next program $p_t$, and the score $y_t$ depends on both $p_t$ and the static but unknown mechanism $\theta_{\mathrm{sci}}$, which in turn updates $m_{t+1}$.}
\label{fig:overview}
\end{figure}

\subsection{Evolutionary Discovery as a POMDP}
\label{sec:pomdp}

Scientific discovery aims to uncover the causal mechanisms underlying a problem through interactions with the world~\citep{sci_revolution}. In evolutionary discovery, the agent interacts with the problem through an evaluator. We denote by $\theta_{\mathrm{sci}}\in\Theta$ the hidden mechanism of the problem, which is fixed during the search and unknown to the agent. At turn $t$, the agent proposes a program $p_t\in\mathcal{P}$ and observes a score $y_t$ with $\E[y_t\mid p_t]=F(p_t;\theta_{\mathrm{sci}})$, where $F:\mathcal{P}\times\Theta\to\mathbb{R}$ is the objective, e.g., the value of a mathematical construction, the efficiency of a kernel, or the performance of a trained model~\citep{novikov2025alphaevolve}. The goal is to uncover $\theta_{\mathrm{sci}}$ by optimizing for the evaluation score.

\textbf{POMDP.} As shown in Fig.~\ref{fig:overview}B, the interaction forms a Partially Observable Markov Decision Process (POMDP)~\citep{kaelbling1998planning} with a static hidden state $s_t=\theta_{\mathrm{sci}}$, actions $a_t=p_t$, observations $o_t=y_t$, an identity transition kernel, and the observation kernel $O(y_t\mid \theta_{\mathrm{sci}},p_t)$. Given a budget of $T$ evaluations, the agent chooses $p_0,\dots,p_{T-1}$ to find a program $\hat p$ with high $F(\hat p;\theta_{\mathrm{sci}})$.

\textbf{Evaluation as intervention.} $\theta_{\mathrm{sci}}$ can be parameterized as a structural causal model over variables of interest~\citep{spirtes2000causation}, in which a program specifies a configuration $x_p$ of design variables $X$. Each evaluation can be considered as the expected outcome of the intervention:
\begin{equation}\label{eq:evaluation_scm}
    F(p;\theta_{\mathrm{sci}}) = \E\big[\,Y \,\big|\, \mathrm{do}(X=x_p),\ \theta_{\mathrm{sci}}\big].
\end{equation}
In this work, we leverage the causal view to provide a more in-depth understanding, while neither the analysis nor the method requires recovering the causal graph. A full discussion on the relations with causality can be found in Appendix~\ref{appdx:formulation}.

\textbf{Experience-driven discovery.}
We define $b_t$ as the agent's Bayesian belief after accumulating \textit{experience} $h_t=\{(p_0,y_0),\dots,(p_{t-1},y_{t-1})\}$,
which is essentially a probability distribution on $\Theta$, $b_t(B) = \Pr(\Theta_{\mathrm{sci}}\in B \mid h_t)$ for $B\subseteq \Theta$.
In the ideal Bayesian formalism, the belief $b_t$ is a sufficient
statistic for decision-making~\citep{kaelbling1998planning}. In practice, the AI Scientist maintains an internal belief within LLMs, or an external one realized as memory $m_t = \Phi(h_t)$ for some (possibly learnable) summarization function $\Phi$~\citep{lange2025shinkaevolve}, to represent the \emph{approximate representation} of its knowledge about $\theta_{\mathrm{sci}}$ and the
landscape of $F(\cdot; \theta_{\mathrm{sci}})$.
Each evaluation step $(p_t, y_t)$ thus updates $m_t$, which in turn updates the
agent's belief about $\theta_{\mathrm{sci}}$ and determines the next action $p_{t+1}$.

\subsection{Reusing Experience with Ideas}
\label{sec:ideas}
If the evaluator fully reveals $\theta_{\mathrm{sci}}$ and the evaluation budget is unlimited, then enumerating every possible program can directly lead to discovery. However, in reality, the evaluation score only reveals partial information about $\theta_{\mathrm{sci}}$ and evaluations can be expensive, e.g., training a model or running a physical experiment, so the agent can afford only a limited number of them~\citep{wallace2026conjecture}.
Therefore, it is \textit{essential} to reuse experience from past evaluations.

Without properly reusing experience, a single evaluation cannot reveal further insights about other unseen programs, and the agent then has to evaluate every candidate program to find the best one.
Reusing experience requires some structure shared across programs, so that the evaluation of one program also informs the others. We formalize the contrast as follows.

\begin{theorem}[Structured experience reuse, informal]
\label{thm:structured_reuse}
Let $\mathcal{P}=\{p_1,\dots,p_K\}$, and let each evaluation return $y_t=F(p_t;\theta_{\mathrm{sci}})+\varepsilon_t$ with $\varepsilon_t\sim\mathcal{N}(0,\sigma^2)$ i.i.d.
\begin{enumerate}[label=(\alph*), leftmargin=*]
\item \emph{With reuse.} If each program has a known representation $\phi(p)\in\mathbb{R}^d$ with $\|\phi(p)\|_2\le 1$ and $F(p;\theta_{\mathrm{sci}})=\langle \phi(p),w^\star\rangle$ for an unknown $w^\star\in\mathbb{R}^d$, then under a coverage condition on the representation, there exists a policy that, with probability at least $1-\delta$, outputs a program within $\epsilon$ of the optimal value after $O\big(d\,\frac{\sigma^2}{\epsilon^2}\log\frac{K}{\delta}\big)$ evaluations.
\item \emph{Without reuse.} If programs are unrelated, i.e., $F(p_i;\theta_{\mathrm{sci}})=\mu_i$ for an arbitrary $\mu\in\mathbb{R}^K$, then any policy with the same guarantee for every $\mu$ needs $\Omega\big(K\,\frac{\sigma^2}{\epsilon^2}\log\frac{1}{\delta}\big)$ evaluations.
\end{enumerate}
\end{theorem}
The formal statement and proof are given in Appendix~\ref{appdx:thm1-static}.
Theorem~\ref{thm:structured_reuse} implies that reusing experience at the high-level design choices can significantly improve sample efficiency from $\Omega(K)$ to $O(d\log K)$, where $K\gg d$. In evolutionary discovery, $K$ counts the candidate program space, which grow combinatorially with the high-level design choices and further with the ways of implementing each choice, while $d$ is determined by the design choices that affect the score.
Most existing agents retain experience at the level of programs~\citep{novikov2025alphaevolve,lange2025shinkaevolve}, which ignores the high-level design choices.
In contrast, human scientists recognize and record what they learn from experiments in terms of reusable concepts, e.g., a construction or an optimization technique, which carry over to later experiments. We are hence inspired to introduce the notion of \textit{ideas}.

\begin{definition}[Idea]
\label{def:idea}
An idea is an interpretable mapping $z:\mathcal{P}\to\mathcal{V}$ from programs to a binary or categorical value. It describes a high-level mechanistic property that can be shared by syntactically different programs, such as the construction family, initialization, or optimization procedure. A set of ideas $\{z_1,\dots,z_r\}$ defines the idea representation $Z(p)=(z_1(p),\dots,z_r(p))$.
\end{definition}

A program realizes a set of ideas together with low-level details.
We call an idea representation $Z$ \emph{faithful} if the score depends on a program only through its ideas, i.e., $F(p;\theta_{\mathrm{sci}})=g(Z(p))$ for some function $g$ determined by $\theta_{\mathrm{sci}}$. Under a faithful idea representation, the hidden mechanism acts on a program only through its ideas, and Theorem~\ref{thm:structured_reuse} applies with $\phi(p)$ being the one-hot encoding of $Z(p)$, whose dimension is determined by the number of ideas.

\textbf{Identifiability of ideas.} However, the agent does not know in advance which representation is faithful, and the ideas have to be identified from the evaluated programs, while which programs are evaluated is decided by the search itself. When the search stops varying the candidate ideas, the exploration of ideas degenerates, as formalized below.

\begin{theorem}[Degenerated exploration of ideas]
\label{thm:degenerated_exploration}
Consider two candidate ideas $z_0$ and $z_1$, and two mechanisms $\theta_0$ and $\theta_1$ under which the score depends on $z_0$ and $z_1$, respectively, i.e., $F(p;\theta_i)=g(Z(p),z_i(p))$ for a common representation $Z$ and function $g$, with the same observation noise. Suppose the optimal programs under $\theta_0$ and $\theta_1$ are different and unique with a margin $\Delta>0$. Then for any policy that only evaluates programs on which $z_0$ and $z_1$ agree, there exists $i\in\{0,1\}$ such that for every budget $T$, the output $\hat p$ satisfies $\E\big[\max_{p\in\mathcal{P}}F(p;\theta_i)-F(\hat p;\theta_i)\big]\ge\Delta/2$.
\end{theorem}

The formal statement and proof are given in Appendix~\ref{appdx:prop2}. Theorem~\ref{thm:degenerated_exploration} can be considered as a consequence of the partial observability of $\theta_\mathrm{sci}$, since the scores reveal $\theta_\mathrm{sci}$ only through the evaluated programs. If two candidate ideas agree on all explored programs, the two mechanisms produce the same observations, and one cannot distinguish which idea determines the score.
This situation is common in evolutionary search: as the search concentrates around its current best programs, the ideas shared by these programs stop varying. Identifying ideas thus requires sufficient exploration to contrast programs with and without each idea, to isolate contributions of ideas in the program space.

\section{\ours: Discovering and Exploiting Ideas from Experience}
\label{sec:method}

Theoretical analysis in Sec.~\ref{sec:formulation} suggests that a discovery agent should extract ideas from its experience, and keep exploring programs where the ideas vary. We realize the theoretical principles in \ours.
The full procedure and implementation details are given in Alg.~\ref{alg:expevolve} and Appendix~\ref{appdx:method}.

\subsection{Discovering Ideas from Experience}
\label{sec:method_extract}

We leverage the discovery loop of the \coat framework~\citep{coat2025discoveringreasoningcausalityhidden} to find idea factors according to their theoretical properties.

\begin{figure}[ht]
    \centering
    \includegraphics[width=0.95\linewidth]{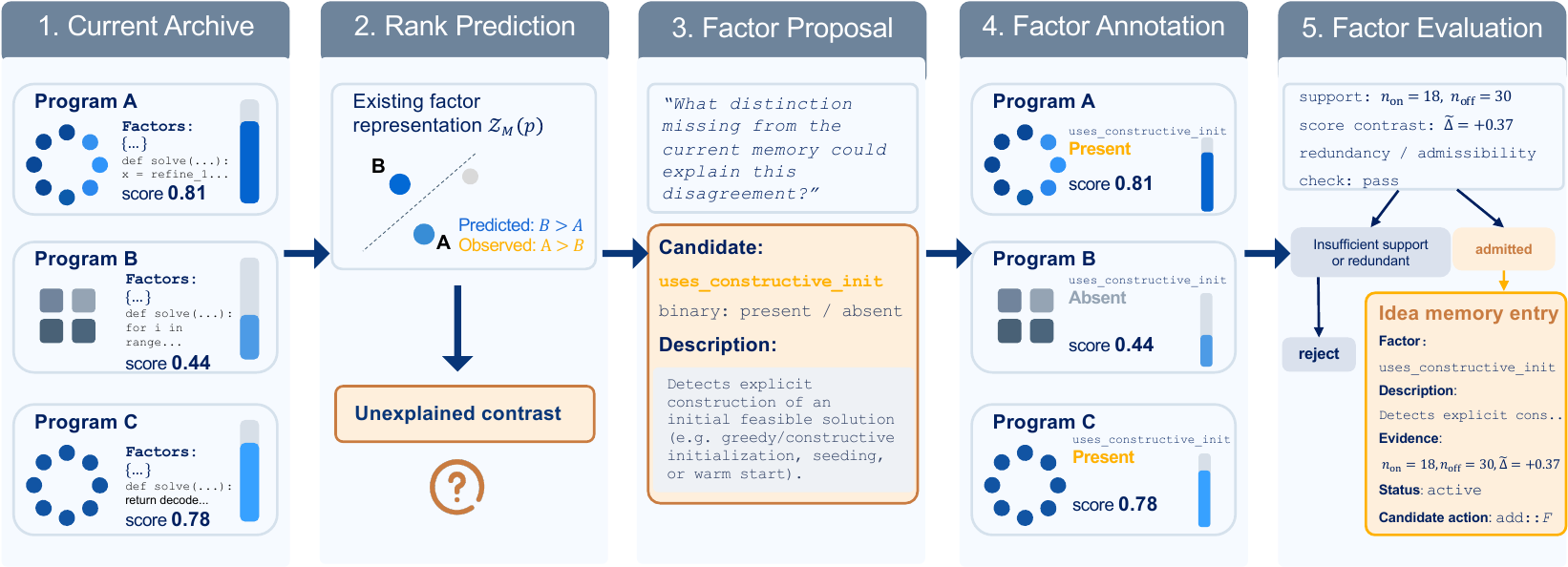}
    \caption{Illustration of idea factor discovery.}
    \label{fig:idea_discovery}
\end{figure}

\textbf{Idea memory.} During the discovery, \ours maintains an idea memory $\mathcal{M}_t$ which is a set of idea factors. Each idea factor is stored with a name, a description of its code pattern, a type, i.e., binary or categorical, and a value set $\mathcal{V}$.
For each program in the experience archive, the values of idea factors are annotated by an annotator LLM, which instantiates the idea representation $Z(p)$.

\textbf{Explainability of idea factors.} A faithful idea representation is expected to be able to predict the score of a program, i.e., $F(p;\theta_{\mathrm{sci}})=g(Z(p))$ (Sec.~\ref{sec:ideas}). However, predicting the exact scores from ideas is difficult, since the scores of different tasks lie on different scales, the unknown function $g$ can be highly nonlinear, and the archive only contains a limited number of noisy evaluations.
Therefore, we define the explainability of an idea representation by \textit{its pairwise predictability}, i.e., whether $Z(p)$ and $Z(p')$ predict which of two programs $p$ and $p'$ is better.
The ranking is already informative for optimizing the objective,  invariant to the scale of the scores, and more robust to noise than the scores.

\textbf{Discovering idea factors from unexplained pairs.} The idea discovery process in \ours is an iterative loop: each time the discovery is triggered, \ours fits a pairwise ranker on the annotated idea factors, and selects the pairs of archived programs with a large score gap that the current idea factors fail to explain. The unexplained pairs suggest the \textit{existence of missing idea factors} that separate the programs in each pair.
For the selected pairs, a separate proposer LLM reads the programs to propose candidate idea factors that distinguish them. Each candidate is then annotated on the whole archive, and is admitted into the idea memory if each of its values is supported by enough programs and its values separate the scores.

\begin{figure}[ht]
    \centering
    \includegraphics[width=0.95\linewidth]{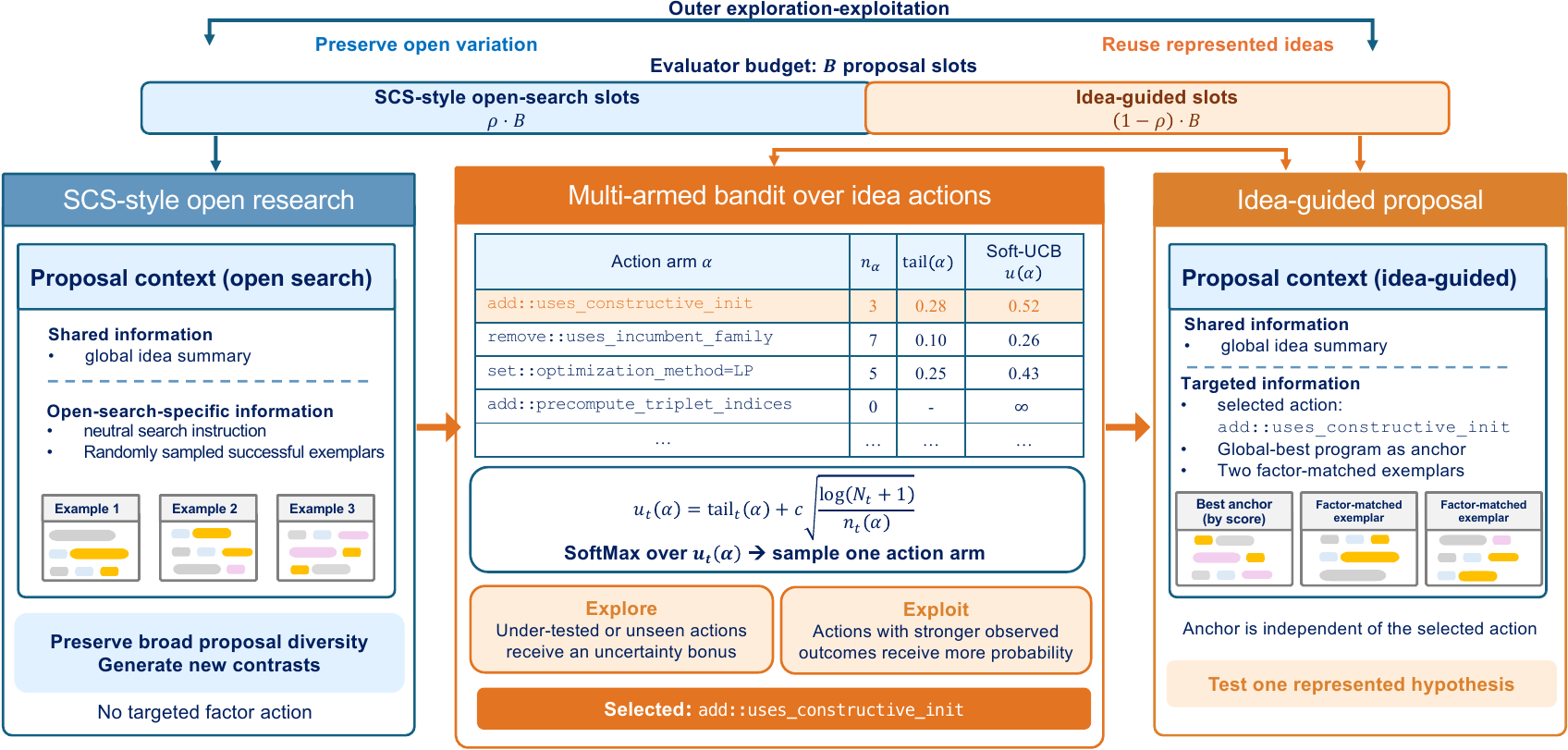}
    \caption{Illustration of two-level exploration-exploitation with discovered idea factors.}
    \label{fig:idea_MAB}
\end{figure}

\subsection{Exploiting Idea Factors for Discovery}
\label{sec:method_guide}
We build a two-level exploration and exploitation framework with a multi-armed bandit (MAB) to fully exploit discovered idea factors while exploring varying design spaces.

\textbf{Exploration and MAB exploitation.} In each generation, \ours keeps at least a fraction $\rho$ of the slots for random exploration that follows \scs~\citep{Gideoni2026SimpleBA} with randomly sampled example programs, which explores programs beyond the current ideas and provides new material for discovery.
The remaining slots are allocated to idea-guided actions by MAB.

\begin{figure}[ht]
    \centering
    \includegraphics[width=0.9\linewidth]{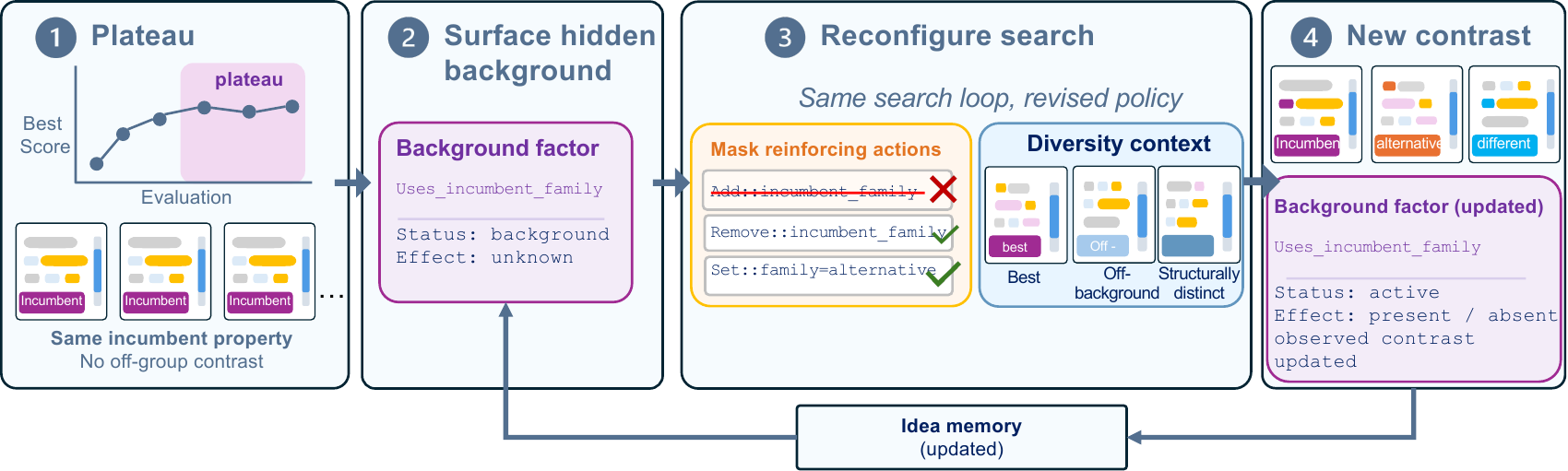}
    \caption{Illustration of idea-guided exploration at plateaus.}
    \label{fig:idea_plateau}
\end{figure}

\textbf{Idea factors as MAB actions.} To exploit an idea factor $z$, \ours compiles it into edit actions, e.g., $\mathsf{ADD}(z)$ and $\mathsf{REMOVE}(z)$ for a binary idea factor. An action adds to the context an instruction that describes $z$, together with demonstration programs in the target state, to ask the mutator LLM to realize the action in the new program. The MAB selects the action for each idea-guided slot. The reward of an action is the score of the program generated under the action relative to the random exploration in the same generation, i.e.,$
    r=\frac{y-\bar{y}_{\mathrm{rand}}}{\sigma_{\mathrm{rand}}},
$
where $y$ is the score of the generated program, and $\bar{y}_{\mathrm{rand}}$ and $\sigma_{\mathrm{rand}}$ are the mean and the standard deviation of the scores of random exploration in the same generation. The normalization removes the drift of scores across generations and the differences of score scales across tasks. The MAB selects the actions with an upper confidence bound (UCB) rule on the maximum reward of each action (Appendix~\ref{appdx:method_actions}).

\textbf{Exploration at plateaus.}\label{sec:method_plateau} Oftentimes, evolutionary discovery can concentrate around its best programs and a single strategy. Consequently, the ideas widely used in top-scoring programs stop varying, and their roles cannot be identified from the archive (Thm.~\ref{thm:degenerated_exploration}). When the best score does not improve for $\kappa$ generations, \ours switches to idea-guided exploration, which enforces the mutator LLM to explore diverse ideas other than those widely used in top-scoring programs. Specifically, a separate LLM probes the properties shared by the top-scoring programs as background idea factors, and an LLM judge determines which idea factors in the memory depend on the background. \ours then masks the actions of the dependent idea factors, instructs the mutator LLM to avoid them, and assigns more slots to random exploration with diversified demonstrations.

\section{Experiments}
\label{sec:exp}

We evaluate \ours on $11$ open-ended mathematical discovery tasks and further integrate it into an evolutionary agent for Machine Learning Engineering (MLE) tasks.

\subsection{Experimental Setup}
\label{sec:exp_setting}

\textbf{Tasks.} We evaluate on $11$ open-ended mathematical discovery tasks. Nine tasks are the problems studied by \scs~\citep{Gideoni2026SimpleBA}, which are taken from \alevolve~\citep{novikov2025alphaevolve}, including the first and second autocorrelation inequalities, the uncertainty inequality, the Erd\H{o}s minimum overlap problem, sums and differences of sets, the max--min distance ratio, the Heilbronn problem for triangles, the kissing number in $11$ dimensions, and circle packing. The other two tasks are the Hadamard matrix problem~\citep{Wang2025ThetaEvolveTL} and a variant of circle packing that recovers the radii by linear programming. We report the best score at the end of the budget, as the mean and standard deviation over seeds, and the best seed. Full details of the tasks are given in Appendix~\ref{appdx:tasks}.

\textbf{Budget.} Each run evaluates $200$ programs. We consider two configurations with the same budget, i.e., $10$ generations of $20$ programs as in \scs~\citep{Gideoni2026SimpleBA} and $20$ generations of $10$ programs, which presents more challenges to accumulating reusable knowledge from experience.

\textbf{Implementation.} We compare \ours with \scs~\citep{Gideoni2026SimpleBA}, which is the backbone of \ours and shows competitive results against more elaborate evolutionary coding agents. We further compare with \shinka~\citep{lange2025shinkaevolve}, which maintains an island-based archive and summarizes the insights of recent programs with a meta-summary.
Details are given in Appendix~\ref{appdx:baselines}.
We use DeepSeek-V4-Flash by default, and each setting is run with $3$ seeds. \ours uses a random-exploration fraction $\rho=0.60$ for $20$-gen. and $\rho=0.40$ for $10$-gen. to keep a similar number of random-exploration slots, i.e., $6$ and $8$. We donot tune other hyperparameters.

\subsection{Main Results}
\label{sec:exp_main}

\begin{table*}[t]
\vspace{-0.1in}
\centering
\caption{Final performance at the budget of $200$ evaluated programs. Mean reports the seed mean $\pm$ seed standard deviation over $3$ seeds, and Best reports the best seed. }
\vspace{-0.05in}
\label{tab:main_results}
\setlength{\tabcolsep}{3.2pt}
\renewcommand{\arraystretch}{1.12}
\resizebox{\linewidth}{!}{%
\begin{tabular}{@{}l*{10}{c}@{}}
\toprule
& \multicolumn{4}{c}{\textbf{20 generations\ $\times$ 10 programs}}
& \multicolumn{4}{c}{\textbf{10 generations\ $\times$ 20 programs}}
& \multicolumn{2}{c}{\multirow{2}{*}{\textbf{\shinka}}} \\
\cmidrule(lr){2-5}\cmidrule(lr){6-9}
\textbf{Task}
& \multicolumn{2}{c}{\textbf{\ours}} & \multicolumn{2}{c}{\textbf{\scs}}
& \multicolumn{2}{c}{\textbf{\ours}} & \multicolumn{2}{c}{\textbf{\scs}}
& & \\
\cmidrule(lr){2-3}\cmidrule(lr){4-5}\cmidrule(lr){6-7}\cmidrule(lr){8-9}\cmidrule(l){10-11}
& Mean & Best & Mean & Best & Mean & Best & Mean & Best & Mean & Best \\
\midrule
Hadamard $\uparrow$
& \textbf{0.9211 $\pm$ 0.000} & \textbf{0.9211}
& 0.6641 $\pm$ 0.169 & 0.8596
& \textbf{0.9259 $\pm$ 0.008} & \textbf{0.9357}
& 0.8349 $\pm$ 0.101 & 0.9211
& 0.6467 $\pm$ 0.139 & 0.8066 \\
Pack $\uparrow$
& \textbf{2.63542 $\pm$ 0.001} & \textbf{2.63598}
& 2.61697 $\pm$ 0.028 & 2.63572
& \textbf{2.63499 $\pm$ 0.002} & \textbf{2.63598}
& 2.62305 $\pm$ 0.011 & 2.63104
& 2.58201 $\pm$ 0.037 & 2.61408 \\
Linprog $\uparrow$
& \textbf{2.63542 $\pm$ 0.001} & \textbf{2.63598}
& \textbf{2.63542 $\pm$ 0.001} & \textbf{2.63598}
& \textbf{2.63542 $\pm$ 0.001} & \textbf{2.63598}
& 2.55161 $\pm$ 0.006 & 2.55608
& 2.51434 $\pm$ 0.026 & 2.54142 \\
SumDiff $\uparrow$
& \textbf{1.11609 $\pm$ 0.006} & \textbf{1.11937}
& 1.11087 $\pm$ 0.001 & 1.11143
& \textbf{1.12224 $\pm$ 0.001} & \textbf{1.12281}
& 1.11608 $\pm$ 0.005 & 1.12006
& 1.10954 $\pm$ 0.000 & 1.10954 \\
Heilbronn $\uparrow$
& \textbf{0.0353 $\pm$ 0.001} & \textbf{0.0359}
& 0.0313 $\pm$ 0.003 & 0.0340
& \textbf{0.0358 $\pm$ 0.001} & \textbf{0.0365}
& 0.0336 $\pm$ 0.001 & 0.0343
& 0.0303 $\pm$ 0.002 & 0.0324 \\
Min overlap $\downarrow$
& \textbf{0.3810 $\pm$ 0.000} & \textbf{0.3810}
& 0.3816 $\pm$ 0.001 & \textbf{0.3810}
& \textbf{0.3810 $\pm$ 0.000} & \textbf{0.3810}
& \textbf{0.3810 $\pm$ 0.000} & \textbf{0.3810}
& 0.3824 $\pm$ 0.000 & 0.3824 \\
Kissing $\uparrow$
& \textbf{526 $\pm$ 97} & \textbf{582}
& 413 $\pm$ 23 & 438
& \textbf{582 $\pm$ 0} & \textbf{582}
& \textbf{582 $\pm$ 0} & \textbf{582}
& 462 $\pm$ 90 & 566 \\
Second autocorr.\ $\uparrow$
& \textbf{0.8825 $\pm$ 0.000} & \textbf{0.8825}
& 0.8398 $\pm$ 0.016 & 0.8499
& \textbf{0.8825 $\pm$ 0.000} & \textbf{0.8825}
& \textbf{0.8825 $\pm$ 0.000} & \textbf{0.8825}
& 0.8745 $\pm$ 0.008 & 0.8825 \\
First autocorr.\ $\downarrow$
& \textbf{1.5113 $\pm$ 0.002} & \textbf{1.5098}
& 1.5235 $\pm$ 0.011 & 1.5151
& \textbf{1.5195 $\pm$ 0.006} & \textbf{1.5151}
& 1.5508 $\pm$ 0.011 & 1.5378
& 1.5540 $\pm$ 0.029 & 1.5203 \\
Max--min ratio $\downarrow$
& \textbf{12.9262 $\pm$ 0.064} & \textbf{12.8892}
& 12.9338 $\pm$ 0.077 & \textbf{12.8892}
& \textbf{12.8892 $\pm$ 0.000} & \textbf{12.8892}
& 12.9808 $\pm$ 0.033 & 12.9425
& 13.0000 $\pm$ 0.000 & 13.0000 \\
Hermite $\downarrow$
& \textbf{0.35473 $\pm$ 0.003} & \textbf{0.35263}
& 0.39431 $\pm$ 0.059 & 0.35945
& \textbf{0.35089 $\pm$ 0.002} & \textbf{0.34820}
& 0.35956 $\pm$ 0.005 & 0.35394
& 4.72779 $\pm$ 3.788 & 0.35429 \\
\bottomrule
\end{tabular}}
\end{table*}

\textbf{\ours consistently improves the final performance.} Table~\ref{tab:main_results} shows the final performance at the budget of $200$ programs. \ours improves the mean and best final performances across most settings over \scs and \shinka. Moreover, the improvements are robust to how the budget is parallelized. From $10$ generations of $20$ programs to $20$ generations of $10$ programs, the mean final score of \scs can drop dramatically because of limited capability in experience reuse, while \ours still maintains competitive performance. More details are given in Appendix~\ref{appdx:results}.

\begin{figure}[ht]
    \centering
    \includegraphics[width=\textwidth]{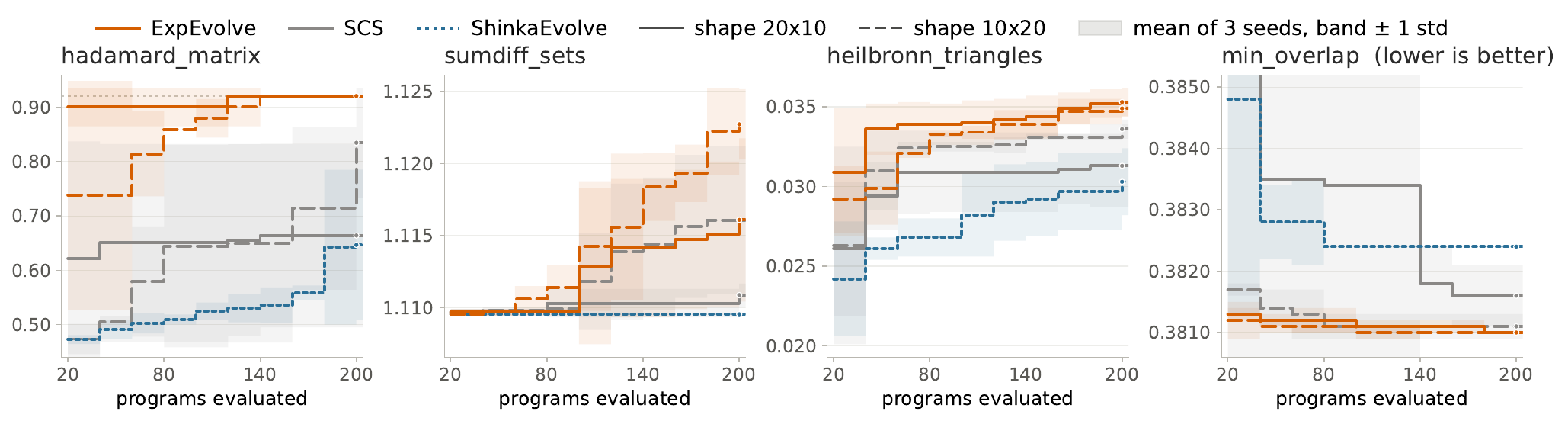}
    \caption{Best-so-far scores along the budget of $200$ evaluated programs. }
    \label{fig:curves}
\end{figure}

\textbf{\ours improves sample efficiency.} Fig.~\ref{fig:curves} shows the best-so-far scores along the budget on Hadamard, Sums and differences, Heilbronn, and Minimum overlap. On these four tasks, \ours needs $1.7\times$ to $26.2\times$ fewer evaluations than \scs under both configurations.

\begin{figure}[ht]
    \centering
    \includegraphics[width=\textwidth]{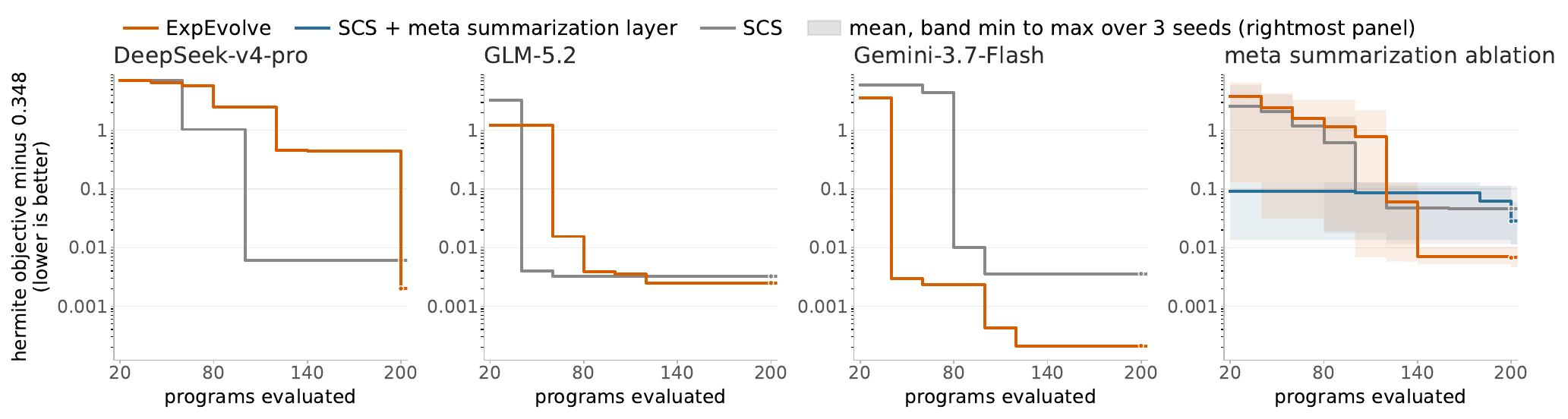}
    \caption{(a) Best-so-far scores of \scs and \ours with other LLMs. (b) Best-so-far scores of \scs, SCS+meta, and \ours. }
    \label{fig:meta_llm}
\end{figure}

\textbf{The improvements transfer to other LLMs.} Fig.~\ref{fig:meta_llm} (a) shows that the advantages of \ours over \scs transfers to other LLMs such as DeepSeek-v4-pro, GLM-5.2 and Gemini-3.7-Flash.

\textbf{Summarizing experience as memory helps the early search but saturates.} \shinka and SCS+meta retain the experience as a textual summary of recent programs, which provides additional context to the mutator LLM without controlling which ideas are evaluated. Fig.~\ref{fig:meta_llm} (b) shows that SCS+meta improves over \scs in the early stage. However, its improvement slows down afterwards while \ours keep improving. It again verifies the advantages of idea-guided evolution.

\begin{figure}[t]
    \centering
    \includegraphics[width=0.95\textwidth]{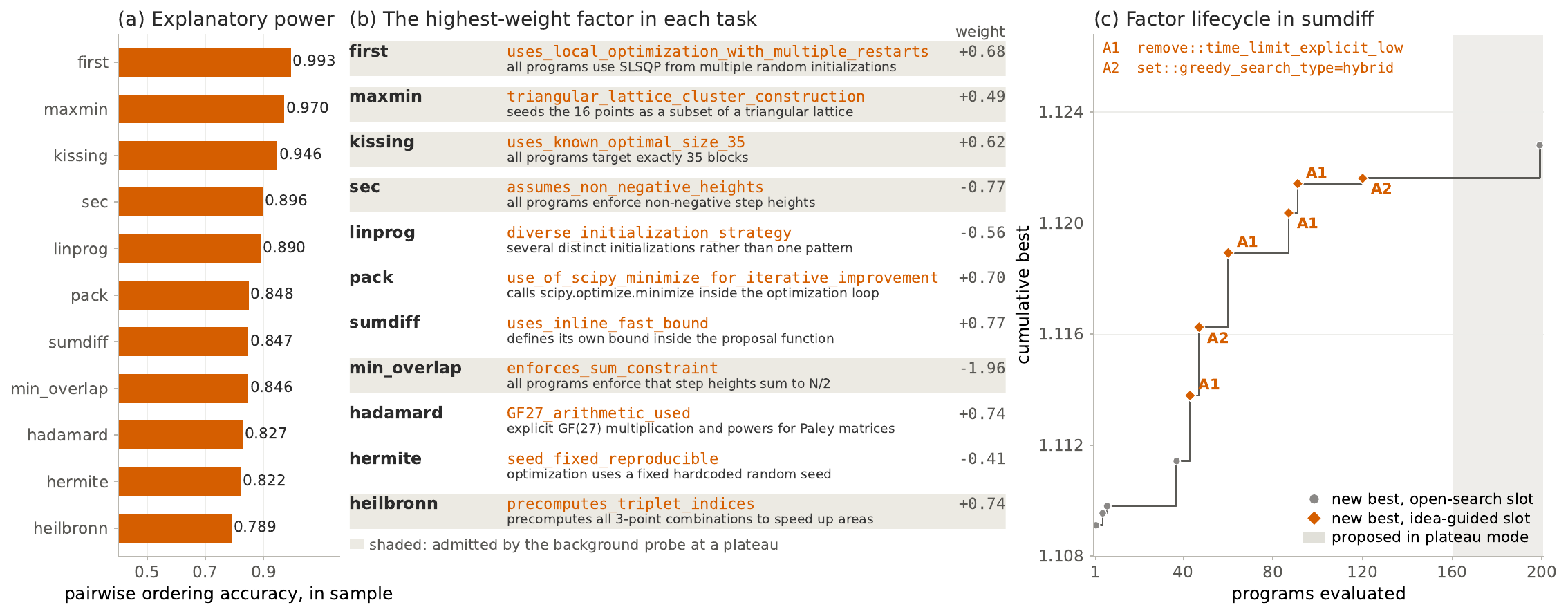}
    \caption{Quality of discovered idea factors. (a) Pairwise accuracy of the discovered idea factors. (b) The idea factor with the largest weight. (c) Illustration of usage of factors in sumdiff task.}
    \label{fig:factors}
\end{figure}

\textbf{The discovered idea factors explain the order of archived programs.} Following the definition of explainability in Sec.~\ref{sec:method_extract}, we measure the fraction of archived program pairs that discovered factors can order correctly. As shown in Fig.~\ref{fig:factors}(a), the discovered idea factors reach an accuracy of $0.79$ to $0.99$. Fig.~\ref{fig:factors}(b) shows the idea factor with the largest weight on each task.

\textbf{In an example run, idea-guided actions produce about half of the new best programs.} Fig.~\ref{fig:factors}(c) shows a run on Sums and differences with $10$ generations of $20$ programs. $6$ of the $11$ new best programs come from idea-guided slots. The $92$ idea-guided slots of the run select $7$ distinct actions, and all the new best programs from idea-guided slots come from $2$ of them.

\begin{figure*}[t]
    \centering
    \subfigure[Final BSF performance\label{fig:mle-final}]{
        \includegraphics[width=0.3\textwidth]{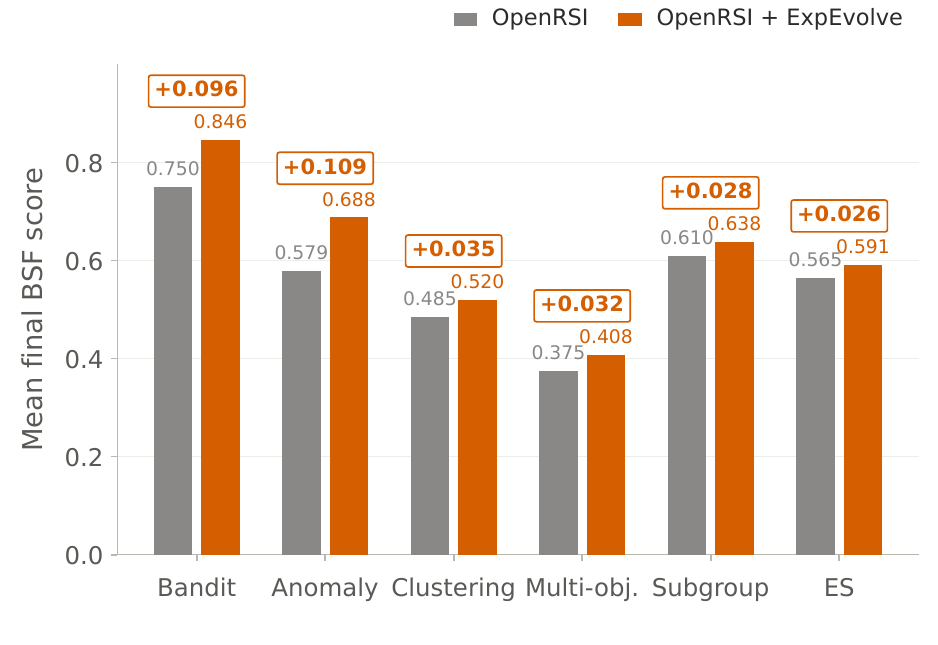}
    }
    \hfill
    \subfigure[Multi-objective optimization\label{fig:mle-guided-case}]{
        \includegraphics[width=0.3\textwidth]{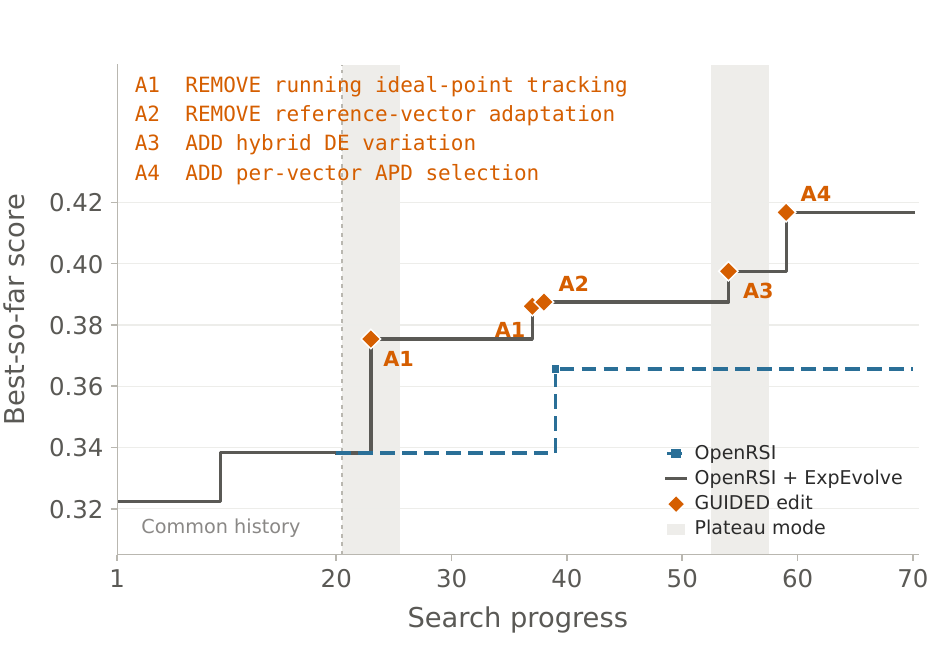}
    }
    \hfill
    \subfigure[Evolution strategy\label{fig:mle-plateau-case}]{
        \includegraphics[width=0.3\textwidth]{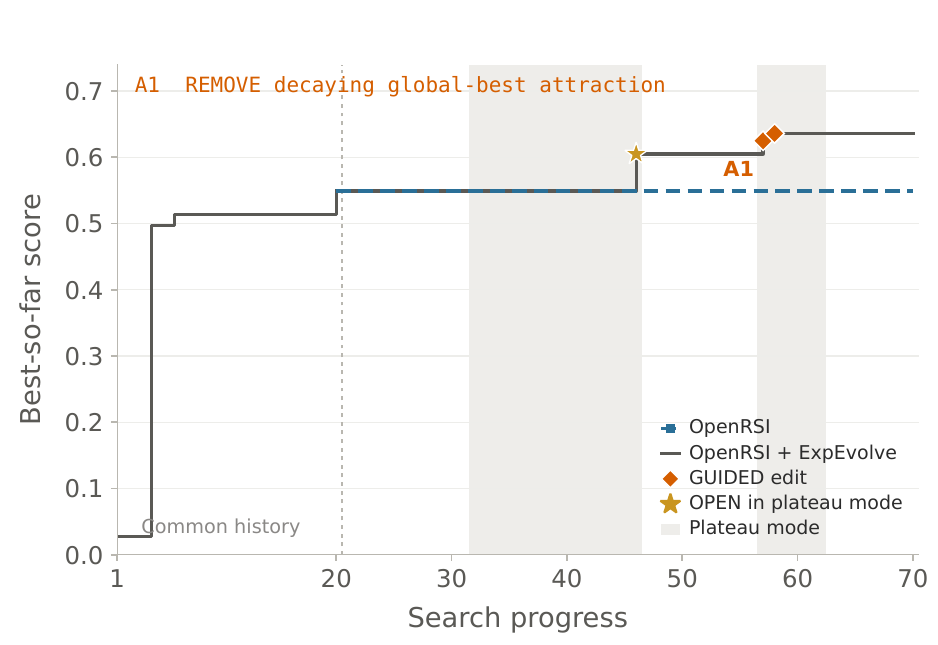}
    }
    \caption{
        Results on MLS-Bench.
        (a) Mean final scores.
        (b,c) Trajectories on two example tasks.
    }
    \label{fig:mle-main}
\end{figure*}

\subsection{Additional Experiments on Machine Learning Engineering Tasks}
\label{sec:mle-experiments}

We further study whether the effectiveness of \ours transfers to MLE by integrating it into OpenRSI~\citep{yang2026frontis} and evaluating it on $6$ tasks from MLS-Bench~\citep{lyu2026mlsbench}: \emph{Online bandit}, \emph{Anomaly detection}, \emph{Clustering}, \emph{Multi-objective optimization}, \emph{Subgroup calibration shift}, and \emph{Evolution strategy}.
Search leverages the scores and execution feedback from each task.

\textbf{Experimental setup.}
We use the evolutionary harness from OpenRSI, which organizes search through \textsc{Draft}, \textsc{Improve}, \textsc{Debug}, and \textsc{Crossover} operators.
For each task, we retain a fixed history containing 20 completed search slots from an earlier OpenRSI run. The OpenRSI baseline and OpenRSI equipped with \ours continue from this same history for \textbf{50} additional evaluator calls, with \textbf{3} seeds.  Both methods use \textbf{5} search slots per generation.
All program generation uses Qwen3.6-35B-A3B.
The same checkpoint performs factor discovery and annotation in \ours.
For each task and method, we average the final best-so-far (BSF) scores over three seeds.

\textbf{Idea-guided evolution transfers to MLE tasks.}
Fig.~\ref{fig:mle-final} compares OpenRSI with OpenRSI+\allowbreak\ours across the six tasks, where \ours shows superior performance on all six tasks, with the largest gains on Online bandit and Anomaly detection.

\textbf{Idea-guided actions produce better exploration and exploitation on MLE tasks.}
Fig.~\ref{fig:mle-guided-case} and Fig.~\ref{fig:mle-plateau-case} show two concrete trajectories in the tasks of multi-objective optimization and evolution strategy, respectively. We can find that \ours can effectively discover useful ideas such as removing running ideal-point tracking (A1), removing reference-vector adaptation (A2), and adding hybrid differential-evolution variation (A3) to effectively boost the performance. When the evolution gets stuck at plateaus, exploiting useful strategies proposed by \ours in multi-objective optimization, or exploration guided by idea factors in \ours can effectively break the plateaus and lead to improvements. More details about the implementation can be found in Appendix~\ref{appdx:mle-experiments}.

\section{Conclusions}
\label{sec:conclusion}

We introduced \ours, which structures the experience of evolutionary discovery into idea factors and reuses them to guide the search. \ours discovers idea factors from the pairs of programs that the current idea factors fail to explain, exploits them as the actions of a multi-armed bandit, and keeps exploring programs where the ideas vary, through random exploration and idea-guided exploration at plateaus. Our analysis shows that reusing a faithful idea representation reduces the number of evaluations from $\Omega(K)$ to $O(d\log K)$, and that identifying ideas requires exploring programs where the ideas vary. On $11$ mathematical discovery tasks and $6$ tasks from MLS-Bench, \ours improves the final performance and the sample efficiency of evolutionary discovery under the same evaluation budget.

\bibliography{ref_0_ai_sci,ref_1_llm,ref_2_causality,ref_3_others}

\begin{thebibliography}{116}
\providecommand{\natexlab}[1]{#1}
\providecommand{\url}[1]{\texttt{#1}}
\expandafter\ifx\csname urlstyle\endcsname\relax
  \providecommand{\doi}[1]{doi: #1}\else
  \providecommand{\doi}{doi: \begingroup \urlstyle{rm}\Url}\fi

\bibitem[Ahuja et~al.(2023)Ahuja, Mahajan, Wang, and
  Bengio]{ahuja2023interventional}
K.~Ahuja, D.~Mahajan, Y.~Wang, and Y.~Bengio.
\newblock Interventional causal representation learning.
\newblock In \emph{International Conference on Machine Learning (ICML)}, 2023.

\bibitem[{Anthropic}(2026)]{anthropic2026protein}
{Anthropic}.
\newblock Claude accelerates protein design and analytical chemistry.
\newblock
  \url{https://www.anthropic.com/research/Claude-accelerates-protein-design},
  Aug. 2026.
\newblock 354 binders against 14 of 15 targets from 1,320 designs, produced and
  tested independently by Adaptyv Bio and Twist Bioscience.

\bibitem[Asai et~al.(2026)Asai, He, Shao, Shi, Singh, Chang, Lo, Soldaini,
  Feldman, et~al.]{asai2026openscholar}
A.~Asai, J.~He, R.~Shao, W.~Shi, A.~Singh, J.~C. Chang, K.~Lo, L.~Soldaini,
  S.~Feldman, et~al.
\newblock Synthesizing scientific literature with retrieval-augmented language
  models.
\newblock \emph{Nature}, 2026.
\newblock \doi{10.1038/s41586-025-10072-4}.

\bibitem[Assump{\c{c}}{\~a}o et~al.(2025)Assump{\c{c}}{\~a}o, Ferreira, Campos,
  and Murai]{assumpcao2025codeevolve}
H.~Assump{\c{c}}{\~a}o, D.~Ferreira, L.~Campos, and F.~Murai.
\newblock {CodeEvolve}: an open source evolutionary coding agent for algorithm
  discovery and optimization.
\newblock \emph{arXiv preprint arXiv:2510.14150}, 2025.

\bibitem[Boiko et~al.(2023)Boiko, MacKnight, Kline, and
  Gomes]{boiko2023autonomous}
D.~A. Boiko, R.~MacKnight, B.~Kline, and G.~Gomes.
\newblock Autonomous chemical research with large language models.
\newblock \emph{Nature}, 624\penalty0 (7992):\penalty0 570--578, 2023.

\bibitem[Brouillard et~al.(2020)Brouillard, Lachapelle, Lacoste,
  Lacoste-Julien, and Drouin]{brouillard2020differentiable}
P.~Brouillard, S.~Lachapelle, A.~Lacoste, S.~Lacoste-Julien, and A.~Drouin.
\newblock Differentiable causal discovery from interventional data.
\newblock \emph{Advances in Neural Information Processing Systems},
  33:\penalty0 21865--21877, 2020.

\bibitem[Bubeck et~al.(2025)Bubeck, Coester, Eldan, Gowers, Lee, Lupsasca,
  Sawhney, Scherrer, Sellke, Spears, Unutmaz, Weil, Yin, and
  Zhivotovskiy]{Bubeck2025EarlySA}
S.~Bubeck, C.~Coester, R.~Eldan, T.~Gowers, Y.~T. Lee, A.~Lupsasca, M.~Sawhney,
  R.~Scherrer, M.~Sellke, B.~K. Spears, D.~Unutmaz, K.~Weil, S.~Yin, and
  N.~Zhivotovskiy.
\newblock Early science acceleration experiments with gpt-5.
\newblock \emph{ArXiv}, abs/2511.16072, 2025.

\bibitem[Cemri et~al.(2026)]{cemri2026adaevolve}
M.~Cemri et~al.
\newblock {AdaEvolve}: Adaptive {LLM} driven zeroth-order optimization.
\newblock \emph{arXiv preprint arXiv:2602.20133}, 2026.

\bibitem[Chalupka et~al.(2015)Chalupka, Perona, and
  Eberhardt]{chalupka2015visual}
K.~Chalupka, P.~Perona, and F.~Eberhardt.
\newblock Visual causal feature learning.
\newblock In \emph{Proceedings of the Thirty-First Conference on Uncertainty in
  Artificial Intelligence}, 2015.

\bibitem[Chan et~al.(2025)Chan, Chowdhury, Jaffe, Aung, Sherburn, Mays,
  Starace, Liu, Maksin, Patwardhan, Madry, and Weng]{chan2024mle-bench}
J.~S. Chan, N.~Chowdhury, O.~Jaffe, J.~Aung, D.~Sherburn, E.~Mays, G.~Starace,
  K.~Liu, L.~Maksin, T.~Patwardhan, A.~Madry, and L.~Weng.
\newblock {MLE}-bench: Evaluating machine learning agents on machine learning
  engineering.
\newblock In \emph{ICLR}, 2025.
\newblock arXiv:2410.07095.

\bibitem[Chen et~al.(2023)Chen, Huang, Zhou, Bian, Han, and Cheng]{feat}
Y.~Chen, W.~Huang, K.~Zhou, Y.~Bian, B.~Han, and J.~Cheng.
\newblock Understanding and improving feature learning for out-of-distribution
  generalization.
\newblock In \emph{Advances in Neural Information Processing Systems}, 2023.

\bibitem[Cheng et~al.(2025)Cheng, Liu, Pan, Li, Wang, Krentsel, Xia, Cemri,
  Park, Yang, Chen, Agrawal, Desai, Xing, Sen, Zaharia, and
  Stoica]{Cheng2025BarbariansAT}
A.~Cheng, S.~Liu, M.~Z. Pan, Z.~Li, B.~Wang, A.~Krentsel, T.~Xia, M.~Cemri,
  J.~Park, S.~Yang, J.~Chen, L.~A. Agrawal, A.~Desai, J.~Xing, K.~Sen,
  M.~Zaharia, and I.~Stoica.
\newblock Barbarians at the gate: How ai is upending systems research.
\newblock \emph{ArXiv}, abs/2510.06189, 2025.

\bibitem[Da~Costa et~al.(2008)Da~Costa, Fialho, Schoenauer, and
  Sebag]{dacosta2008aos}
L.~Da~Costa, {\'A}.~Fialho, M.~Schoenauer, and M.~Sebag.
\newblock Adaptive operator selection with dynamic multi-armed bandits.
\newblock In \emph{Proceedings of the Genetic and Evolutionary Computation
  Conference (GECCO)}, pages 913--920, 2008.

\bibitem[Dhawan et~al.(2024)Dhawan, Cotta, Ullrich, Krishnan, and
  Maddison]{dhawan2024end}
N.~Dhawan, L.~Cotta, K.~Ullrich, R.~G. Krishnan, and C.~J. Maddison.
\newblock End-to-end causal effect estimation from unstructured natural
  language data.
\newblock \emph{Advances in Neural Information Processing Systems},
  37:\penalty0 77165--77199, 2024.

\bibitem[Douven(2025)]{sep-abduction}
I.~Douven.
\newblock {Abduction}.
\newblock In E.~N. Zalta and U.~Nodelman, editors, \emph{The {Stanford}
  Encyclopedia of Philosophy}. Metaphysics Research Lab, Stanford University,
  {W}inter 2025 edition, 2025.

\bibitem[Favaro and Clark(2026)]{favaro2026rsi}
M.~Favaro and J.~Clark.
\newblock When {AI} builds itself: Our progress toward recursive
  self-improvement and its implications.
\newblock Anthropic Institute,
  \url{https://www.anthropic.com/institute/recursive-self-improvement}, June
  2026.

\bibitem[Feng et~al.(2026{\natexlab{a}})Feng, Lv, Ye, Wang, Huo, Yu, Xu, Zhang,
  Bai, He, and Li]{feng2025earth}
P.~Feng, Z.~Lv, J.~Ye, X.~Wang, X.~Huo, J.~Yu, W.~Xu, W.~Zhang, L.~Bai, C.~He,
  and W.~Li.
\newblock {Earth-Agent}: Unlocking the full landscape of earth observation with
  agents.
\newblock In \emph{ICLR}, 2026{\natexlab{a}}.
\newblock arXiv:2509.23141.

\bibitem[Feng et~al.(2026{\natexlab{b}})Feng, Ma, Yan,
  et~al.]{internscience2026internagent}
S.~Feng, R.~Ma, X.~Yan, et~al.
\newblock {InternAgent-1.5}: A unified agentic framework for long-horizon
  autonomous scientific discovery.
\newblock \emph{arXiv preprint arXiv:2602.08990}, 2026{\natexlab{b}}.

\bibitem[Georgiev et~al.(2025)Georgiev, G{\'o}mez-Serrano, Tao, and
  Wagner]{Georgiev2025MathematicalEA}
B.~Georgiev, J.~G{\'o}mez-Serrano, T.~Tao, and A.~Z. Wagner.
\newblock Mathematical exploration and discovery at scale.
\newblock \emph{ArXiv}, abs/2511.02864, 2025.

\bibitem[Ghafarollahi and Buehler(2025)]{ghafarollahi2025sciagents}
A.~Ghafarollahi and M.~J. Buehler.
\newblock Sciagents: automating scientific discovery through bioinspired
  multi-agent intelligent graph reasoning.
\newblock \emph{Advanced Materials}, 37\penalty0 (22):\penalty0 2413523, 2025.

\bibitem[Ghareeb et~al.(2026)]{ghareeb2026robin}
A.~E. Ghareeb et~al.
\newblock A multi-agent system for automating scientific discovery.
\newblock \emph{Nature}, 655\penalty0 (8122):\penalty0 497--505, 2026.
\newblock \doi{10.1038/s41586-026-10652-y}.

\bibitem[Gideoni et~al.(2026)Gideoni, Risi, and Gal]{Gideoni2026SimpleBA}
Y.~Gideoni, S.~Risi, and Y.~Gal.
\newblock Simple baselines are competitive with code evolution.
\newblock In \emph{ICLR 2026 Workshop on AI with Recursive Self-Improvement},
  2026.
\newblock arXiv:2602.16805.

\bibitem[Glymour()]{ClarkOutline}
C.~Glymour.
\newblock An outline of the history of methods of discovering causality.
\newblock URL
  \url{https://www.cmu.edu/dietrich/philosophy/docs/glymour/an-outline-of-the-history-of-methods-of-discovering-causality.pdf}.
\newblock Accessed: 2026-01-29.

\bibitem[Gottweis et~al.(2026)]{gottweis2026coscientist}
J.~Gottweis et~al.
\newblock Accelerating scientific discovery with co-scientist.
\newblock \emph{Nature}, 655\penalty0 (8122):\penalty0 487--496, 2026.
\newblock \doi{10.1038/s41586-026-10644-y}.

\bibitem[Greenland et~al.(1999)Greenland, Pearl, and
  Robins]{greenland1999causal}
S.~Greenland, J.~Pearl, and J.~M. Robins.
\newblock Causal diagrams for epidemiologic research.
\newblock \emph{Epidemiology}, 10\penalty0 (1):\penalty0 37--48, 1999.

\bibitem[Hoyer et~al.(2008)Hoyer, Janzing, Mooij, Peters, and
  Sch{\"o}lkopf]{hoyer2008nonlinear}
P.~Hoyer, D.~Janzing, J.~M. Mooij, J.~Peters, and B.~Sch{\"o}lkopf.
\newblock Nonlinear causal discovery with additive noise models.
\newblock \emph{Advances in neural information processing systems}, 21, 2008.

\bibitem[Huang et~al.(2020)Huang, Zhang, Zhang, Ramsey, Sanchez-Romero,
  Glymour, and Sch{\"o}lkopf]{huang2020causal}
B.~Huang, K.~Zhang, J.~Zhang, J.~Ramsey, R.~Sanchez-Romero, C.~Glymour, and
  B.~Sch{\"o}lkopf.
\newblock Causal discovery from heterogeneous/nonstationary data.
\newblock \emph{Journal of Machine Learning Research}, 21\penalty0
  (89):\penalty0 1--53, 2020.

\bibitem[Huang et~al.(2025{\natexlab{a}})Huang, Jin, Li, Li, Candes, and
  Leskovec]{huangautomated}
K.~Huang, Y.~Jin, R.~Li, M.~Y. Li, E.~Candes, and J.~Leskovec.
\newblock Automated hypothesis validation with agentic sequential
  falsifications.
\newblock In \emph{ICML}, 2025{\natexlab{a}}.

\bibitem[Huang et~al.(2025{\natexlab{b}})Huang, Chen, Zhang, Li, Zhou, Fang,
  Yang, Li, Shang, Xu, et~al.]{huang2025deep}
Y.~Huang, Y.~Chen, H.~Zhang, K.~Li, H.~Zhou, M.~Fang, L.~Yang, X.~Li, L.~Shang,
  S.~Xu, et~al.
\newblock Deep research agents: A systematic examination and roadmap.
\newblock \emph{arXiv preprint arXiv:2506.18096}, 2025{\natexlab{b}}.

\bibitem[Hyttinen et~al.(2013)Hyttinen, Eberhardt, and
  Hoyer]{hyttinen2013experiment}
A.~Hyttinen, F.~Eberhardt, and P.~O. Hoyer.
\newblock Experiment selection for causal discovery.
\newblock \emph{The Journal of Machine Learning Research}, 14\penalty0
  (1):\penalty0 3041--3071, 2013.

\bibitem[Jing et~al.(2025)Jing, Huang, Wang, Yao, Yu, Ma, Zhang, Du, and
  Yu]{jing2024dsbenchfardatascience}
L.~Jing, Z.~Huang, X.~Wang, W.~Yao, W.~Yu, K.~Ma, H.~Zhang, X.~Du, and D.~Yu.
\newblock {DSBench}: How far are data science agents from becoming data science
  experts?
\newblock In \emph{ICLR}, 2025.
\newblock arXiv:2409.07703.

\bibitem[Kaelbling et~al.(1998)Kaelbling, Littman, and
  Cassandra]{kaelbling1998planning}
L.~P. Kaelbling, M.~L. Littman, and A.~R. Cassandra.
\newblock Planning and acting in partially observable stochastic domains.
\newblock \emph{Artificial intelligence}, 101\penalty0 (1-2):\penalty0 99--134,
  1998.

\bibitem[Karpathy(2026)]{karpathy2026autoresearch}
A.~Karpathy.
\newblock autoresearch: session reports of an autonomous single-{GPU}
  training-loop agent.
\newblock GitHub discussions \#32 and \#43,
  \url{https://github.com/karpathy/autoresearch/discussions/32},
  \url{https://github.com/karpathy/autoresearch/discussions/43}, 2026.
\newblock Discussion \#43 is the 126-experiment log classified in the paper;
  discussion \#32 (8 March 2026) is an 89-experiment log with the same format.

\bibitem[Kaufmann et~al.(2016)Kaufmann, Capp{\'e}, and
  Garivier]{kaufmann2016complexity}
E.~Kaufmann, O.~Capp{\'e}, and A.~Garivier.
\newblock On the complexity of best-arm identification in multi-armed bandit
  models.
\newblock \emph{Journal of Machine Learning Research}, 17\penalty0
  (1):\penalty0 1--42, 2016.

\bibitem[Khrulkov et~al.(2025)Khrulkov, Galichin, et~al.]{khrulkov2025gigaevo}
V.~Khrulkov, A.~V. Galichin, et~al.
\newblock {GigaEvo}: An open source optimization framework powered by {LLMs}
  and evolution algorithms.
\newblock \emph{arXiv preprint arXiv:2511.17592}, 2025.

\bibitem[Kuhn and Hawkins(1963)]{sci_revolution}
T.~S. Kuhn and D.~Hawkins.
\newblock The structure of scientific revolutions.
\newblock \emph{American Journal of Physics}, 31:\penalty0 554--555, 1963.

\bibitem[Lange et~al.(2026)Lange, Imajuku, and Cetin]{lange2025shinkaevolve}
R.~T. Lange, Y.~Imajuku, and E.~Cetin.
\newblock {ShinkaEvolve}: Towards open-ended and sample-efficient program
  evolution.
\newblock In \emph{ICLR}, 2026.
\newblock arXiv:2509.19349.

\bibitem[Laurent et~al.(2024)Laurent, Janizek, Ruzo, Hinks, Hammerling,
  Narayanan, Ponnapati, White, and Rodriques]{laurent2024lab}
J.~M. Laurent, J.~D. Janizek, M.~Ruzo, M.~M. Hinks, M.~J. Hammerling,
  S.~Narayanan, M.~Ponnapati, A.~D. White, and S.~G. Rodriques.
\newblock Lab-bench: Measuring capabilities of language models for biology
  research.
\newblock \emph{arXiv preprint arXiv:2407.10362}, 2024.

\bibitem[Lehman and Stanley(2011)]{lehman2011novelty}
J.~Lehman and K.~O. Stanley.
\newblock Abandoning objectives: Evolution through the search for novelty
  alone.
\newblock \emph{Evolutionary Computation}, 19\penalty0 (2):\penalty0 189--223,
  2011.

\bibitem[Li et~al.(2025{\natexlab{a}})Li, Chen, Liu, Cai, Liu, Han, Zhang, and
  Xiong]{Li2025CanLL}
J.~Li, Y.~Chen, C.~Liu, Q.~Cai, T.~Liu, B.~Han, K.~Zhang, and H.~Xiong.
\newblock Can large language models help experimental design for causal
  discovery?
\newblock \emph{ArXiv}, abs/2503.01139, 2025{\natexlab{a}}.
\newblock URL \url{https://arxiv.org/abs/2503.01139}.

\bibitem[Li et~al.(2025{\natexlab{b}})Li, Wang, Zhang, Liu, Liu, Cao, Yu, and
  Chen]{lirevealing}
J.~Li, S.~Wang, Q.~Zhang, F.~Liu, T.~Liu, L.~Cao, S.~Yu, and F.~Chen.
\newblock Revealing multimodal causality with large language models.
\newblock In \emph{The Thirty-ninth Annual Conference on Neural Information
  Processing Systems}, 2025{\natexlab{b}}.

\bibitem[Li et~al.(2025{\natexlab{c}})Li, Xu, Guo, Zhao, Li, Yuan, Zhang,
  Jiang, Xin, Dang, Rong, Zhao, Feng, and Bing]{li2024chain}
L.~Li, W.~Xu, J.~Guo, R.~Zhao, X.~Li, Y.~Yuan, B.~Zhang, Y.~Jiang, Y.~Xin,
  R.~Dang, Y.~Rong, D.~Zhao, T.~Feng, and L.~Bing.
\newblock Chain of ideas: Revolutionizing research via novel idea development
  with {LLM} agents.
\newblock In \emph{Findings of the Association for Computational Linguistics:
  EMNLP 2025}, pages 8971--9004, 2025{\natexlab{c}}.
\newblock \doi{10.18653/v1/2025.findings-emnlp.477}.

\bibitem[Li et~al.(2024)Li, Vajipey, Goodman, and Fox]{li2024critical}
M.~Y. Li, V.~Vajipey, N.~D. Goodman, and E.~B. Fox.
\newblock Critical: Critic automation with language models.
\newblock \emph{arXiv preprint arXiv:2411.06590}, 2024.

\bibitem[Li and Liu(2025)]{li2025efficient}
X.-C. Li and T.~Liu.
\newblock Efficient and trustworthy causal discovery with latent variables and
  complex relations.
\newblock In \emph{The Thirteenth International Conference on Learning
  Representations}, 2025.

\bibitem[Li et~al.(2025{\natexlab{d}})Li, Wang, and Liu]{li2025recovery}
X.-C. Li, J.~Wang, and T.~Liu.
\newblock Recovery of causal graph involving latent variables via homologous
  surrogates.
\newblock In \emph{The Thirteenth International Conference on Learning
  Representations}, 2025{\natexlab{d}}.

\bibitem[Liang et~al.(2024)Liang, Zhang, Cao, Wang, Ding, Yang, Vodrahalli, He,
  Smith, Yin, et~al.]{liang2024can}
W.~Liang, Y.~Zhang, H.~Cao, B.~Wang, D.~Y. Ding, X.~Yang, K.~Vodrahalli, S.~He,
  D.~S. Smith, Y.~Yin, et~al.
\newblock Can large language models provide useful feedback on research papers?
  a large-scale empirical analysis.
\newblock \emph{NEJM AI}, 1\penalty0 (8):\penalty0 AIoa2400196, 2024.

\bibitem[Liu et~al.(2025)Liu, Chen, Liu, Gong, Cheng, Han, and
  Zhang]{coat2025discoveringreasoningcausalityhidden}
C.~Liu, Y.~Chen, T.~Liu, M.~Gong, J.~Cheng, B.~Han, and K.~Zhang.
\newblock Discovering and reasoning of causality in the hidden world with large
  language models, 2025.
\newblock URL \url{https://arxiv.org/abs/2402.03941}.

\bibitem[Liu et~al.(2024)Liu, Tong, Yuan, Lin, Luo, Wang, Lu, and
  Zhang]{liu2024eoh}
F.~Liu, X.~Tong, M.~Yuan, X.~Lin, F.~Luo, Z.~Wang, Z.~Lu, and Q.~Zhang.
\newblock Evolution of heuristics: Towards efficient automatic algorithm design
  using large language model.
\newblock In \emph{International Conference on Machine Learning (ICML)}, 2024.

\bibitem[Liu et~al.(2026)Liu, Yang, Xie, Ni, Gao, Li, Tang, Ouyang, Cambria,
  and Zhou]{liu2025researchbench}
Y.~Liu, Z.~Yang, T.~Xie, J.~Ni, B.~Gao, Y.~Li, S.~Tang, W.~Ouyang, E.~Cambria,
  and D.~Zhou.
\newblock {ResearchBench}: Benchmarking {LLMs} in scientific discovery via
  inspiration-based task decomposition.
\newblock In \emph{Findings of the Association for Computational Linguistics:
  ACL 2026}, pages 13187--13207, 2026.

\bibitem[Lu et~al.(2024)Lu, Lu, Lange, Foerster, Clune, and Ha]{lu2024ai}
C.~Lu, C.~Lu, R.~T. Lange, J.~Foerster, J.~Clune, and D.~Ha.
\newblock The ai scientist: Towards fully automated open-ended scientific
  discovery.
\newblock \emph{arXiv preprint arXiv:2408.06292}, 2024.

\bibitem[Lu et~al.(2026)Lu, Lu, Lange, Yamada, Hu, Foerster, Ha, and
  Clune]{lu2026aiscientist}
C.~Lu, C.~Lu, R.~T. Lange, Y.~Yamada, S.~Hu, J.~Foerster, D.~Ha, and J.~Clune.
\newblock Towards end-to-end automation of {AI} research.
\newblock \emph{Nature}, 651\penalty0 (8107):\penalty0 914--919, 2026.
\newblock \doi{10.1038/s41586-026-10265-5}.

\bibitem[Luo et~al.(2026)Luo, Huang, Luo, Liu, Deng, Li, Yao, Hu, Feng, and
  Liu]{luo2026seaevo}
S.~Luo, Y.~Huang, H.~Luo, F.~Liu, G.~Deng, L.~Li, Q.~Yao, Z.~Hu, J.~Feng, and
  Q.~Liu.
\newblock {SeaEvo}: Advancing algorithm discovery with strategy space
  evolution.
\newblock \emph{arXiv preprint arXiv:2604.24372}, 2026.

\bibitem[Lv et~al.(2026)Lv, Lu, Liu, Zhang, and Tang]{lv2026muevo}
H.~Lv, N.~Lu, S.~Liu, S.~Zhang, and K.~Tang.
\newblock {MuEvo}: {LLM}-driven evolution of multi-heuristic ensemble.
\newblock \emph{arXiv preprint arXiv:2608.03636}, 2026.

\bibitem[Lyu et~al.(2026{\natexlab{a}})Lyu, Yang, Huang, Zhang, Xu, Li, Han,
  Zhang, Zhang, Huang, Yang, Chen, Guo, Yang, Ai, Chai, Cao, Yang, Wang, Jiang,
  Gao, Tang, Shi, Du, Simchowitz, Jiao, Song, and Jin]{lyu2026mlsbench}
B.~Lyu, Y.~Yang, S.~Huang, J.~Zhang, Q.~Xu, X.~Li, X.~Han, Y.~Zhang, H.~Zhang,
  R.~Huang, K.~Yang, Z.~Chen, W.~Guo, J.~Yang, X.~Ai, W.~Chai, Y.~Cao, Z.~Yang,
  K.~Wang, D.~Jiang, H.-a. Gao, S.~Tang, C.~Shi, S.~S. Du, M.~Simchowitz,
  J.~Jiao, D.~Song, and C.~Jin.
\newblock {MLS-Bench}: A holistic and rigorous assessment of {AI} systems on
  building better {AI}.
\newblock \emph{arXiv preprint arXiv:2605.08678}, 2026{\natexlab{a}}.

\bibitem[Lyu et~al.(2026{\natexlab{b}})Lyu, Zhang, Yi, Zhao, Guo, Hu,
  Piotrowski, Kaliski, Urbani, Meng, Zhou, and Yan]{lyu2026evoscientist}
Y.~Lyu, X.~Zhang, X.~Yi, Y.~Zhao, S.~Guo, W.~Hu, J.~Piotrowski, J.~Kaliski,
  J.~Urbani, Z.~Meng, L.~Zhou, and X.~Yan.
\newblock {EvoScientist}: Towards multi-agent evolving {AI} scientists for
  end-to-end scientific discovery.
\newblock \emph{arXiv preprint arXiv:2603.08127}, 2026{\natexlab{b}}.

\bibitem[Majumder et~al.(2025)Majumder, Surana, Agarwal, Mishra, Meena,
  Prakhar, Vora, Khot, Sabharwal, and Clark]{majumder2025discoverybench}
B.~P. Majumder, H.~Surana, D.~Agarwal, B.~D. Mishra, A.~Meena, A.~Prakhar,
  T.~Vora, T.~Khot, A.~Sabharwal, and P.~Clark.
\newblock {DiscoveryBench}: Towards data-driven discovery with large language
  models.
\newblock In \emph{International Conference on Learning Representations
  (ICLR)}, 2025.

\bibitem[Mandal et~al.(2025)Mandal, Soni, Zaki, Smedskjaer, Wondraczek,
  Wondraczek, Gosvami, and Krishnan]{mandal2025evaluating}
I.~Mandal, J.~Soni, M.~Zaki, M.~M. Smedskjaer, K.~Wondraczek, L.~Wondraczek,
  N.~N. Gosvami, and N.~A. Krishnan.
\newblock Evaluating large language model agents for automation of atomic force
  microscopy.
\newblock \emph{Nature Communications}, 16\penalty0 (1):\penalty0 9104, 2025.

\bibitem[Mitchener et~al.(2025)Mitchener, Yiu, Chang, Bourdenx, Nadolski,
  Sulovari, Landsness, Barab{\'a}si, Narayanan, Evans, Reddy, Foiani, Kamal,
  Shriver, Cao, Wassie, Laurent, Melville-Green, Ramos, Bou, Roberts, Zagorac,
  Orr, Orr, Zwezdaryk, Ghareeb, McCoy, Gomes, Ashley, Duff, Buonassisi,
  Rainforth, Bateman, Skarlinski, Rodriques, Hinks, and
  White]{Mitchener2025KosmosAA}
L.~Mitchener, A.~Yiu, B.~Chang, M.~Bourdenx, T.~Nadolski, A.~Sulovari, E.~C.
  Landsness, D.~L. Barab{\'a}si, S.~Narayanan, N.~Evans, S.~Reddy, M.~S.
  Foiani, A.~Kamal, L.~P. Shriver, F.~Cao, A.~T. Wassie, J.~M. Laurent,
  E.~Melville-Green, M.~C. Ramos, A.~Bou, K.~F. Roberts, S.~Zagorac, T.~C. Orr,
  M.~E. Orr, K.~J. Zwezdaryk, A.~E. Ghareeb, L.~McCoy, B.~Gomes, E.~A. Ashley,
  K.~E. Duff, T.~Buonassisi, T.~Rainforth, R.~J. Bateman, M.~Skarlinski, S.~G.
  Rodriques, M.~M. Hinks, and A.~D. White.
\newblock Kosmos: An ai scientist for autonomous discovery.
\newblock \emph{ArXiv}, abs/2511.02824, 2025.

\bibitem[Mooij et~al.(2020)Mooij, Magliacane, and Claassen]{mooij2020joint}
J.~M. Mooij, S.~Magliacane, and T.~Claassen.
\newblock Joint causal inference from multiple contexts.
\newblock \emph{Journal of machine learning research}, 21\penalty0
  (99):\penalty0 1--108, 2020.

\bibitem[Mouret and Clune(2015)]{mouret2015mapelites}
J.-B. Mouret and J.~Clune.
\newblock Illuminating search spaces by mapping elites.
\newblock \emph{arXiv preprint arXiv:1504.04909}, 2015.

\bibitem[Novikov et~al.(2025)Novikov, V{\~u}, Eisenberger, Dupont, Huang,
  Wagner, Shirobokov, Kozlovskii, Ruiz, Mehrabian,
  et~al.]{novikov2025alphaevolve}
A.~Novikov, N.~V{\~u}, M.~Eisenberger, E.~Dupont, P.-S. Huang, A.~Z. Wagner,
  S.~Shirobokov, B.~Kozlovskii, F.~J. Ruiz, A.~Mehrabian, et~al.
\newblock Alphaevolve: A coding agent for scientific and algorithmic discovery.
\newblock \emph{arXiv preprint arXiv:2506.13131}, 2025.

\bibitem[{OpenAI}(2026{\natexlab{a}})]{openai2026researchacceleration}
{OpenAI}.
\newblock Research acceleration: The view inside {OpenAI}.
\newblock OpenAI blog, 6 September 2026, Sept. 2026{\natexlab{a}}.
\newblock Reports the automated research intern milestone and 3.1
  agent-workdays per human workday in the research organization.

\bibitem[{OpenAI}(2026{\natexlab{b}})]{openai2026tenadvances}
{OpenAI}.
\newblock Ten advances in mathematics and theoretical computer science.
\newblock \url{https://openai.com/index/ten-advances-in-mathematics/}, Aug.
  2026{\natexlab{b}}.
\newblock Results on ten open problems with Lean 4 formal proofs; peer review
  pending at the time of writing.

\bibitem[{OpenAI}(2026{\natexlab{c}})]{openai2026unitdistance}
{OpenAI}.
\newblock An {OpenAI} model has disproved a central conjecture in discrete
  geometry.
\newblock
  \url{https://openai.com/index/model-disproves-discrete-geometry-conjecture/},
  May 2026{\natexlab{c}}.
\newblock Counterexample to the Erd\H{o}s unit distance conjecture, verified by
  external mathematicians with a companion paper.

\bibitem[Pearl(2009)]{pearl2009causality}
J.~Pearl.
\newblock \emph{Causality}.
\newblock Cambridge university press, 2009.

\bibitem[Pelleriti et~al.(2026)]{pelleriti2026evotrace}
N.~Pelleriti et~al.
\newblock What do evolutionary coding agents evolve?
\newblock \emph{arXiv preprint arXiv:2605.20086}, 2026.

\bibitem[Perry et~al.(2022)Perry, Von~K{\"u}gelgen, and
  Sch{\"o}lkopf]{perry2022causal}
R.~Perry, J.~Von~K{\"u}gelgen, and B.~Sch{\"o}lkopf.
\newblock Causal discovery in heterogeneous environments under the sparse
  mechanism shift hypothesis.
\newblock \emph{Advances in Neural Information Processing Systems},
  35:\penalty0 10904--10917, 2022.

\bibitem[Plaat et~al.(2025)Plaat, van Duijn, van Stein, Preuss, van~der Putten,
  and Batenburg]{plaat2025agentic}
A.~Plaat, M.~van Duijn, N.~van Stein, M.~Preuss, P.~van~der Putten, and K.~J.
  Batenburg.
\newblock Agentic large language models, a survey.
\newblock \emph{arXiv preprint arXiv:2503.23037}, 2025.

\bibitem[Quinonero-Candela et~al.(2008)Quinonero-Candela, Sugiyama,
  Schwaighofer, and Lawrence]{datasetshift}
J.~Quinonero-Candela, M.~Sugiyama, A.~Schwaighofer, and N.~D. Lawrence.
\newblock \emph{Dataset shift in machine learning}.
\newblock Mit Press, 2008.

\bibitem[Roch et~al.(2020)]{roch2020chemos}
L.~M. Roch et~al.
\newblock Chemos: An orchestration software to democratize autonomous
  discovery.
\newblock \emph{PLOS ONE}, 15\penalty0 (4):\penalty0 e0229862, 2020.

\bibitem[Romera-Paredes et~al.(2024)]{romeraparedes2024funsearch}
B.~Romera-Paredes et~al.
\newblock Mathematical discoveries from program search with large language
  models.
\newblock \emph{Nature}, 625:\penalty0 468--475, 2024.

\bibitem[Rubenstein et~al.(2017)Rubenstein, Weichwald, Bongers, Mooij, Janzing,
  Grosse-Wentrup, and Sch{\"o}lkopf]{rubenstein2017causal}
P.~K. Rubenstein, S.~Weichwald, S.~Bongers, J.~M. Mooij, D.~Janzing,
  M.~Grosse-Wentrup, and B.~Sch{\"o}lkopf.
\newblock Causal consistency of structural equation models.
\newblock In \emph{Proceedings of the Thirty-Third Conference on Uncertainty in
  Artificial Intelligence}, 2017.

\bibitem[Sch{\"o}lkopf et~al.(2021)Sch{\"o}lkopf, Locatello, Bauer, Ke,
  Kalchbrenner, Goyal, and Bengio]{scholkopf2021toward}
B.~Sch{\"o}lkopf, F.~Locatello, S.~Bauer, N.~R. Ke, N.~Kalchbrenner, A.~Goyal,
  and Y.~Bengio.
\newblock Toward causal representation learning.
\newblock \emph{Proceedings of the IEEE}, 109\penalty0 (5):\penalty0 612--634,
  2021.

\bibitem[Sharma(2025)]{openevolve}
A.~Sharma.
\newblock Openevolve: an open-source evolutionary coding agent, 2025.
\newblock URL \url{https://github.com/algorithmicsuperintelligence/openevolve}.

\bibitem[Sheth et~al.(2025)Sheth, Jin, Wilder, Janzing, and
  Fritz]{sheth2025can}
I.~Sheth, Z.~Jin, B.~Wilder, D.~Janzing, and M.~Fritz.
\newblock Can llms propose instrumental variables for causal reasoning?
\newblock In \emph{NeurIPS 2025 Workshop on CauScien: Uncovering Causality in
  Science}, 2025.

\bibitem[Shimizu et~al.(2006)Shimizu, Hoyer, Hyv{\"a}rinen, Kerminen, and
  Jordan]{shimizu2006linear}
S.~Shimizu, P.~O. Hoyer, A.~Hyv{\"a}rinen, A.~Kerminen, and M.~Jordan.
\newblock A linear non-gaussian acyclic model for causal discovery.
\newblock \emph{Journal of Machine Learning Research}, 7\penalty0 (10), 2006.

\bibitem[Shojaee et~al.(2025{\natexlab{a}})Shojaee, Meidani, Gupta,
  Barati~Farimani, and Reddy]{shojaee2025llmsr}
P.~Shojaee, K.~Meidani, S.~Gupta, A.~Barati~Farimani, and C.~K. Reddy.
\newblock {LLM-SR}: Scientific equation discovery via programming with large
  language models.
\newblock In \emph{ICLR}, 2025{\natexlab{a}}.

\bibitem[Shojaee et~al.(2025{\natexlab{b}})Shojaee, Nguyen, Meidani, Farimani,
  Doan, and Reddy]{shojaeellm}
P.~Shojaee, N.-H. Nguyen, K.~Meidani, A.~B. Farimani, K.~D. Doan, and C.~K.
  Reddy.
\newblock Llm-srbench: A new benchmark for scientific equation discovery with
  large language models.
\newblock In \emph{ICML}, 2025{\natexlab{b}}.

\bibitem[Si et~al.(2026)Si, Hashimoto, and Yang]{si2026gap}
C.~Si, T.~Hashimoto, and D.~Yang.
\newblock The ideation--execution gap: Execution outcomes of {LLM}-generated
  versus human research ideas.
\newblock In \emph{International Conference on Learning Representations
  (ICLR)}, 2026.

\bibitem[Soare et~al.(2014)Soare, Lazaric, and Munos]{soare2014bestarm}
M.~Soare, A.~Lazaric, and R.~Munos.
\newblock Best-arm identification in linear bandits.
\newblock In \emph{Advances in Neural Information Processing Systems
  (NeurIPS)}, 2014.

\bibitem[Spirtes et~al.(1995)Spirtes, Meek, and Richardson]{spirtes1995causal}
P.~Spirtes, C.~Meek, and T.~Richardson.
\newblock Causal inference in the presence of latent variables and selection
  bias.
\newblock In \emph{Proceedings of the Eleventh conference on Uncertainty in
  artificial intelligence}, pages 499--506, 1995.

\bibitem[Spirtes et~al.(2000)Spirtes, Glymour, and
  Scheines]{spirtes2000causation}
P.~Spirtes, C.~N. Glymour, and R.~Scheines.
\newblock \emph{Causation, prediction, and search}.
\newblock MIT press, 2000.

\bibitem[Starace et~al.(2025)Starace, Jaffe, Sherburn, Aung, Chan, Maksin,
  Dias, Mays, Kinsella, Thompson, et~al.]{staracepaperbench}
G.~Starace, O.~Jaffe, D.~Sherburn, J.~Aung, J.~S. Chan, L.~Maksin, R.~Dias,
  E.~Mays, B.~Kinsella, W.~Thompson, et~al.
\newblock Paperbench: Evaluating ai’s ability to replicate ai research.
\newblock In \emph{Proceedings of the 42nd International Conference on Machine
  Learning}, volume 267, pages 56843--56873, 2025.

\bibitem[{\v{S}}urina et~al.(2025){\v{S}}urina, Mansouri, Quaedvlieg, Seddas,
  Viazovska, Abbe, and Gulcehre]{surina2025evotune}
A.~{\v{S}}urina, A.~Mansouri, L.~C. P.~M. Quaedvlieg, A.~Seddas, M.~Viazovska,
  E.~Abbe, and C.~Gulcehre.
\newblock Algorithm discovery with {LLMs}: Evolutionary search meets
  reinforcement learning.
\newblock In \emph{Conference on Language Modeling (COLM)}, 2025.
\newblock arXiv:2504.05108.

\bibitem[Swanson et~al.(2025)Swanson, Wu, Bulaong, Pak, and
  Zou]{swanson2025virtual}
K.~Swanson, W.~Wu, N.~L. Bulaong, J.~E. Pak, and J.~Zou.
\newblock The virtual lab of ai agents designs new sars-cov-2 nanobodies.
\newblock \emph{Nature}, 646\penalty0 (8085):\penalty0 716--723, 2025.

\bibitem[Tie et~al.(2025)Tie, Zhou, and Sun]{tie2025survey}
G.~Tie, P.~Zhou, and L.~Sun.
\newblock A survey of {AI} scientists.
\newblock \emph{arXiv preprint arXiv:2510.23045}, 2025.

\bibitem[Tom et~al.(2024)Tom, Schmid, Baird, Cao, Darvish, Hao, Lo,
  Pablo-Garc{\'\i}a, Rajaonson, Skreta, et~al.]{tom2024self}
G.~Tom, S.~P. Schmid, S.~G. Baird, Y.~Cao, K.~Darvish, H.~Hao, S.~Lo,
  S.~Pablo-Garc{\'\i}a, E.~M. Rajaonson, M.~Skreta, et~al.
\newblock Self-driving laboratories for chemistry and materials science.
\newblock \emph{Chemical Reviews}, 124\penalty0 (16):\penalty0 9633--9732,
  2024.

\bibitem[Truhn et~al.(2026)Truhn, Azizi, Zou, Cerda-Alberich, Mahmood, and
  Kather]{truhn2026artificial}
D.~Truhn, S.~Azizi, J.~Zou, L.~Cerda-Alberich, F.~Mahmood, and J.~N. Kather.
\newblock Artificial intelligence agents in cancer research and oncology.
\newblock \emph{Nature Reviews Cancer}, 26\penalty0 (4):\penalty0 256--269,
  2026.

\bibitem[Volkov et~al.(2026)Volkov, Khrulkov, Galichin, Sivtsov, Glazkov,
  Volkova, Pchelin, Bespalov, Dylov, Anokhin, and Oseledets]{volkov2026evomem}
V.~Volkov, V.~Khrulkov, A.~V. Galichin, D.~Sivtsov, N.~Glazkov, O.~Volkova,
  K.~Pchelin, I.~Bespalov, D.~V. Dylov, P.~Anokhin, and I.~Oseledets.
\newblock {EvoMem}: Memory-augmented evolution for code optimization.
\newblock \emph{arXiv preprint arXiv:2608.10795}, 2026.

\bibitem[Wallace et~al.(2026)Wallace, Griffin, O'Neill, Luong, and
  Larter]{wallace2026conjecture}
D.~Wallace, C.~Griffin, S.~O'Neill, T.~Luong, and O.~Larter.
\newblock Conjecture machines: {AI} agents and the new validation bottleneck in
  science.
\newblock Google DeepMind,
  \url{https://deepmind.google/public-policy/conjecture-machines-ai-agents-and-the-new-validation-bottleneck-in-science/},
  July 2026.

\bibitem[Wallace(1981)]{wallace1981causality}
W.~Wallace.
\newblock \emph{Causality and Scientific Explanation}.
\newblock Number v. 2 in Causality and Scientific Explanation. University Press
  of America, 1981.
\newblock ISBN 9780819114815.

\bibitem[Wan et~al.(2026)Wan, Yang, Yu, Tu, Lu, Yu, Cao, Gao, Xie, Wang,
  et~al.]{wan2025deepresearch}
H.~Wan, C.~Yang, J.~Yu, M.~Tu, J.~Lu, D.~Yu, J.~Cao, B.~Gao, J.~Xie, A.~Wang,
  et~al.
\newblock Deepresearch arena: The first exam of llms' research abilities via
  seminar-grounded tasks.
\newblock In \emph{Proceedings of the AAAI Conference on Artificial
  Intelligence}, volume~40, pages 33341--33349, 2026.

\bibitem[Wang et~al.(2025{\natexlab{a}})Wang, He, Coelho, Bucci, Nazir, Chen,
  Trinh, Zhang, Lu, Huang, et~al.]{wang2025spatialagent}
H.~Wang, Y.~He, P.~P. Coelho, M.~Bucci, A.~Nazir, B.~Chen, L.~T. Trinh,
  S.~Zhang, Z.~Lu, K.~Huang, et~al.
\newblock {SpatialAgent}: An autonomous {AI} agent for spatial biology.
\newblock \emph{bioRxiv}, 2025{\natexlab{a}}.
\newblock \doi{10.1101/2025.04.03.646459}.

\bibitem[Wang et~al.(2024)Wang, Downey, Ji, and Hope]{wang2024scimon}
Q.~Wang, D.~Downey, H.~Ji, and T.~Hope.
\newblock Scimon: Scientific inspiration machines optimized for novelty.
\newblock In \emph{Proceedings of the 62nd Annual Meeting of the Association
  for Computational Linguistics (Volume 1: Long Papers)}, pages 279--299, 2024.

\bibitem[Wang et~al.(2025{\natexlab{b}})Wang, Pi{\k{e}}kos, Nanbo, Laakom,
  Chen, Ostaszewski, Zhuge, and Schmidhuber]{wang2025hgm}
W.~Wang, P.~Pi{\k{e}}kos, L.~Nanbo, F.~Laakom, Y.~Chen, M.~Ostaszewski,
  M.~Zhuge, and J.~Schmidhuber.
\newblock {Huxley-G{\"o}del Machine}: Human-level coding agent development by
  an approximation of the optimal self-improving machine.
\newblock \emph{arXiv preprint arXiv:2510.21614}, 2025{\natexlab{b}}.

\bibitem[Wang et~al.(2026)Wang, Su, Zeng, Xu, Ren, Yang, Huang, He, Ma, Peng,
  Cheng, He, Chen, Wang, Du, and Shen]{Wang2025ThetaEvolveTL}
Y.~Wang, S.-R. Su, Z.~Zeng, E.~Xu, L.~Ren, X.~Yang, Z.~Huang, X.~He, L.~Ma,
  B.~Peng, H.~Cheng, P.~He, W.~Chen, S.~Wang, S.~S. Du, and Y.~Shen.
\newblock Thetaevolve: Test-time learning on open problems.
\newblock In \emph{International Conference on Machine Learning (ICML)}, 2026.
\newblock arXiv:2511.23473.

\bibitem[Wang et~al.(2025{\natexlab{c}})Wang, Danek, and Sun]{wang2025biodsa}
Z.~Wang, B.~Danek, and J.~Sun.
\newblock Biodsa-1k: Benchmarking data science agents for biomedical research.
\newblock \emph{arXiv preprint arXiv:2505.16100}, 2025{\natexlab{c}}.

\bibitem[Woodruff et~al.(2026)Woodruff, Cohen-Addad, Jain, Mao, Zuo, Bateni,
  Br{\^a}nzei, Brenner, Chen, Feng, Fortnow, Fu, Guan, Hadizadeh, Hajiaghayi,
  JafariRaviz, Javanmard, {Karthik C. S.}, ichi Kawarabayashi, Kumar, Lattanzi,
  Lee, Li, Panageas, Paparas, Przybocki, Subercaseaux, Svensson, Taherijam, Wu,
  Yogev, Zadimoghaddam, Zhou, Matias, Manyika, and
  Mirrokni]{Woodruff2026AcceleratingSR}
D.~P. Woodruff, V.~Cohen-Addad, L.~Jain, J.~Mao, S.~Zuo, M.~Bateni,
  S.~Br{\^a}nzei, M.~P. Brenner, L.~Chen, Y.~Feng, L.~Fortnow, G.~Fu, Z.~Guan,
  Z.~Hadizadeh, M.~T. Hajiaghayi, M.~JafariRaviz, A.~Javanmard, {Karthik C.
  S.}, K.~ichi Kawarabayashi, R.~Kumar, S.~Lattanzi, E.~Lee, Y.~Li,
  I.~Panageas, D.~Paparas, B.~Przybocki, B.~Subercaseaux, O.~Svensson,
  S.~Taherijam, X.~Wu, E.~Yogev, M.~Zadimoghaddam, S.~Zhou, Y.~Matias,
  J.~Manyika, and V.~S. Mirrokni.
\newblock Accelerating scientific research with gemini: Case studies and common
  techniques.
\newblock \emph{ArXiv}, abs/2602.03837, 2026.

\bibitem[Wu et~al.(2026)Wu, Fu, Si, Huang, Jiang, Li, Xia, Sun, Xu, Xiao, Lu,
  Cai, Ye, Zhu, Hu, and Liu]{wu2025innovatorbench}
Y.~Wu, D.~Fu, W.~Si, Z.~Huang, M.~Jiang, K.~Li, S.~Xia, J.~Sun, T.~Xu, Y.~Xiao,
  P.~Lu, X.~Cai, L.~Ye, W.~Zhu, X.~Hu, and P.~Liu.
\newblock {InnovatorBench}: Evaluating agents' ability to conduct innovative
  {AI} research.
\newblock In \emph{ICLR}, 2026.
\newblock arXiv:2510.27598.

\bibitem[Yamada et~al.(2025)Yamada, Lange, Lu, Hu, Lu, Foerster, Clune, and
  Ha]{yamada2025ai}
Y.~Yamada, R.~T. Lange, C.~Lu, S.~Hu, C.~Lu, J.~Foerster, J.~Clune, and D.~Ha.
\newblock The ai scientist-v2: Workshop-level automated scientific discovery
  via agentic tree search.
\newblock \emph{arXiv preprint arXiv:2504.08066}, 2025.

\bibitem[Yang et~al.(2026{\natexlab{a}})Yang, Lu, Wan, Yu, and
  Qin]{yang2025multi}
C.~Yang, J.~Lu, H.~Wan, J.~Yu, and F.~Qin.
\newblock From what to why: A multi-agent system for evidence-based chemical
  reaction condition reasoning.
\newblock In \emph{ICLR}, 2026{\natexlab{a}}.

\bibitem[Yang et~al.(2026{\natexlab{b}})Yang, Jiang, Fu, Luo, Ren, Wang, Zhao,
  Liu, Zuo, Wang, et~al.]{yang2026frontis}
J.~Yang, C.~Jiang, Y.~Fu, T.~Luo, C.~Ren, W.~Wang, K.~Zhao, H.~Liu, Y.~Zuo,
  Y.~Wang, et~al.
\newblock Frontis-ma1: Training an ai4ai model towards recursive
  self-improvement in machine learning engineering.
\newblock \emph{arXiv preprint arXiv:2607.28568}, 2026{\natexlab{b}}.

\bibitem[Yang et~al.(2018)Yang, Katcoff, and Uhler]{yang2018characterizing}
K.~Yang, A.~Katcoff, and C.~Uhler.
\newblock Characterizing and learning equivalence classes of causal dags under
  interventions.
\newblock In \emph{International Conference on Machine Learning}, pages
  5541--5550. PMLR, 2018.

\bibitem[Yang et~al.(2024)Yang, Du, Li, Zheng, Poria, and
  Cambria]{yang2024large}
Z.~Yang, X.~Du, J.~Li, J.~Zheng, S.~Poria, and E.~Cambria.
\newblock Large language models for automated open-domain scientific hypotheses
  discovery.
\newblock In \emph{Findings of the Association for Computational Linguistics:
  ACL 2024}, pages 13545--13565, 2024.

\bibitem[Yang et~al.(2025)Yang, Liu, Gao, Xie, Li, Ouyang, Poria, Cambria, and
  Zhou]{yangmoose}
Z.~Yang, W.~Liu, B.~Gao, T.~Xie, Y.~Li, W.~Ouyang, S.~Poria, E.~Cambria, and
  D.~Zhou.
\newblock Moose-chem: Large language models for rediscovering unseen chemistry
  scientific hypotheses.
\newblock In \emph{ICLR}, 2025.

\bibitem[Ye et~al.(2024)Ye, Wang, Cao, Berto, Hua, Kim, Park, and
  Song]{ye2024reevo}
H.~Ye, J.~Wang, Z.~Cao, F.~Berto, C.~Hua, H.~Kim, J.~Park, and G.~Song.
\newblock {ReEvo}: Large language models as hyper-heuristics with reflective
  evolution.
\newblock In \emph{Advances in Neural Information Processing Systems
  (NeurIPS)}, volume~37, pages 43571--43608, 2024.

\bibitem[Yu et~al.(2025{\natexlab{a}})Yu, Liang, Ho, and Ren]{yu2025satlution}
C.~Yu, R.~Liang, C.-T. Ho, and H.~Ren.
\newblock Autonomous code evolution meets {NP}-completeness.
\newblock \emph{arXiv preprint arXiv:2509.07367}, 2025{\natexlab{a}}.

\bibitem[Yu et~al.(2025{\natexlab{b}})Yu, Feng, Zhao, He, Zhang, and
  Cohan]{yu2025alpharesearch}
Z.~Yu, K.~Feng, Y.~Zhao, S.~He, X.-P. Zhang, and A.~Cohan.
\newblock {AlphaResearch}: Accelerating new algorithm discovery with language
  models.
\newblock \emph{arXiv preprint arXiv:2511.08522}, 2025{\natexlab{b}}.

\bibitem[Yuksekgonul et~al.(2026)Yuksekgonul, Koceja, Li, Bianchi, McCaleb,
  Wang, Kautz, Choi, Zou, Guestrin, and Sun]{yuksekgonul2026ttt}
M.~Yuksekgonul, D.~Koceja, X.~Li, F.~Bianchi, J.~McCaleb, X.~Wang, J.~Kautz,
  Y.~Choi, J.~Zou, C.~Guestrin, and Y.~Sun.
\newblock Learning to discover at test time.
\newblock In \emph{International Conference on Machine Learning (ICML)}, 2026.
\newblock arXiv:2601.16175.

\bibitem[Zhang et~al.(2026)Zhang, Hu, Lu, Lange, and Clune]{zhang2025dgm}
J.~Zhang, S.~Hu, C.~Lu, R.~T. Lange, and J.~Clune.
\newblock Darwin g{\"o}del machine: Open-ended evolution of self-improving
  agents.
\newblock In \emph{International Conference on Learning Representations
  (ICLR)}, 2026.
\newblock arXiv:2505.22954.

\bibitem[Zhang et~al.(2025)Zhang, Yang, Xie, Xue, Hu, Li, Qu, Yin, Fu, Hu,
  et~al.]{zhang2025position}
S.~Zhang, S.~Yang, T.~Xie, X.~Xue, Z.~Hu, R.~Li, W.~Qu, Z.~Yin, T.~Fu, D.~Hu,
  et~al.
\newblock Position: Intelligent science laboratory requires the integration of
  cognitive and embodied ai.
\newblock \emph{arXiv preprint arXiv:2506.19613}, 2025.

\bibitem[Zheng et~al.(2025{\natexlab{a}})Zheng, Deng, Tsang, Wang, Bai, Wang,
  and Song]{ZHENG2025FromAT}
T.~Zheng, Z.~Deng, H.~T. Tsang, W.~Wang, J.~Bai, Z.~Wang, and Y.~Song.
\newblock From automation to autonomy: A survey on large language models in
  scientific discovery.
\newblock In \emph{Proceedings of the 2025 Conference on Empirical Methods in
  Natural Language Processing}, pages 17733--17750, 2025{\natexlab{a}}.

\bibitem[Zheng et~al.(2026)Zheng, Tam, Nguyen, Xu, Wang, Cheng, Tsang, Wang,
  Bai, Fang, Song, Wong, and See]{zheng2025newtonbench}
T.~Zheng, K.~K.-W. Tam, N.~H.-N.~K. Nguyen, B.~Xu, Z.~Wang, J.~Cheng, H.~T.
  Tsang, W.~Wang, J.~Bai, T.~Fang, Y.~Song, G.~Y. Wong, and S.~See.
\newblock {NewtonBench}: Benchmarking generalizable scientific law discovery in
  {LLM} agents.
\newblock In \emph{ICLR}, 2026.
\newblock arXiv:2510.07172.

\bibitem[Zheng et~al.(2025{\natexlab{b}})Zheng, Xie, Wang, and
  Hooi]{zheng2025mctsahd}
Z.~Zheng, Z.~Xie, Z.~Wang, and B.~Hooi.
\newblock Monte carlo tree search for comprehensive exploration in {LLM}-based
  automatic heuristic design.
\newblock In \emph{International Conference on Machine Learning (ICML)},
  2025{\natexlab{b}}.
\newblock arXiv:2501.08603.

\bibitem[Zhou et~al.(2025)Zhou, Ling, Fu, Huang, Sun, Yu, Wang, Li, Su, Zhang,
  Chen, Liang, Qian, Ji, Wang, Zitnik, and Ji]{Zhou2025AutonomousAF}
L.~Zhou, H.~Ling, C.~Fu, Y.~Huang, M.~Sun, W.~Yu, X.~Wang, X.~Li, X.~Su,
  J.~Zhang, X.~Chen, C.~Liang, X.~Qian, H.~Ji, W.~Wang, M.~Zitnik, and S.~Ji.
\newblock Autonomous agents for scientific discovery: Orchestrating scientists,
  language, code, and physics.
\newblock \emph{ArXiv}, abs/2510.09901, 2025.

\bibitem[Zhu et~al.(2022)Zhu, Zhang, Huang, Xiao, Zhao, Zhang, Song, Tang, Li,
  He, et~al.]{zhu2022all}
Q.~Zhu, F.~Zhang, Y.~Huang, H.~Xiao, L.~Zhao, X.~Zhang, T.~Song, X.~Tang,
  X.~Li, G.~He, et~al.
\newblock An all-round ai-chemist with a scientific mind.
\newblock \emph{National Science Review}, 9\penalty0 (10):\penalty0 nwac190,
  2022.

\end{thebibliography}

\clearpage
\newpage
\appendix

\section{Limitations and Future directions}
Due to the resource constraints, the tasks studied in this work remain limited to small-scale verification and are relatively simple. There remains a broader scope of scientific problems to be explored where verification signals can be expensive, such as biology, medicine, etc.. Meanwhile, the factor discovery can also be guided with additional background knowledge of the task.

\section{Related Work}
\label{appdx:related_work}

In this section, we present more detailed discussions of related works in the literature.

\textbf{AI Scientist agents.} With the significant advancement in LLM capacity and the development of agentic systems, a rising number of works develop agents to help scientific discovery~\citep{ZHENG2025FromAT,Zhou2025AutonomousAF,tie2025survey}. Some systems automate the whole research workflow, from idea generation to experimentation and paper writing~\citep{lu2026aiscientist,yamada2025ai,gottweis2026coscientist,ghareeb2026robin,Mitchener2025KosmosAA,internscience2026internagent,lyu2026evoscientist}. One research line automates specific scientific activities, including literature review and deep research~\citep{asai2026openscholar,huang2025deep}, hypothesis generation~\citep{li2024chain,yang2024large,wang2024scimon,yangmoose}, hypothesis verification~\citep{li2024critical,huangautomated}, and feedback on scientific papers~\citep{liang2024can}. Another line targets data-driven discovery and machine learning engineering~\citep{majumder2025discoverybench,wang2025biodsa,jing2024dsbenchfardatascience,chan2024mle-bench}, as well as embodied agents in self-driving laboratories that conduct real-world experiments~\citep{roch2020chemos,zhu2022all,tom2024self,zhang2025position,mandal2025evaluating}. A growing set of benchmarks evaluates the research abilities of these agents, e.g., replicating and conducting AI research~\citep{staracepaperbench,wu2025innovatorbench,lyu2026mlsbench}, generating hypotheses from inspirations~\citep{liu2025researchbench}, discovering scientific laws and equations~\citep{zheng2025newtonbench,shojaeellm}, and conducting research tasks in biology~\citep{laurent2024lab}, together with broader evaluations of research abilities~\citep{wan2025deepresearch}. These lines of work have made impact in mathematics~\citep{openai2026unitdistance,openai2026tenadvances,Bubeck2025EarlySA,Woodruff2026AcceleratingSR}, chemistry~\citep{yang2025multi,boiko2023autonomous}, materials science~\citep{ghafarollahi2025sciagents}, earth science~\citep{feng2025earth}, biology~\citep{swanson2025virtual,wang2025spatialagent,truhn2026artificial,anthropic2026protein}, and the development of LLMs themselves~\citep{karpathy2026autoresearch,openai2026researchacceleration,favaro2026rsi}. As generation scales up, validating the generated hypotheses becomes the main cost~\citep{wallace2026conjecture}, and ideas that look promising often degrade once executed~\citep{si2026gap}, which motivates us to study how to reuse experience to reduce the evaluations needed in evolutionary discovery.

\textbf{Evolutionary coding agents.} FunSearch~\citep{romeraparedes2024funsearch} pairs an LLM with an evaluator and a program database to discover new constructions in extremal combinatorics, and \alevolve~\citep{novikov2025alphaevolve} extends the loop to full codebases with an ensemble of LLMs, with applications to a large set of mathematical problems~\citep{Georgiev2025MathematicalEA}. Follow-up frameworks improve the search process: OpenEvolve provides an open-source implementation~\citep{openevolve}; \shinka improves sample efficiency with parent sampling, novelty rejection, bandit-based LLM selection, and a meta-scratchpad that summarizes insights of recent programs~\citep{lange2025shinkaevolve}; CodeEvolve and GigaEvo maintain diversity with island models and quality-diversity archives~\citep{assumpcao2025codeevolve,khrulkov2025gigaevo}; MCTS-AHD organizes heuristics in a Monte Carlo search tree~\citep{zheng2025mctsahd}; and AdaEvolve adapts search intensity and resource allocation from the history of improvements~\citep{cemri2026adaevolve}. The paradigm has also been applied to equation discovery~\citep{shojaee2025llmsr} and systems research~\citep{Cheng2025BarbariansAT}. \scs shows that simply conditioning on previous solutions is competitive with more elaborate pipelines~\citep{Gideoni2026SimpleBA}, which we adopt as the backbone and main baseline. Classical evolutionary computation also reuses experience through credit assignment over predefined units: adaptive operator selection allocates budget across operators with multi-armed bandits~\citep{dacosta2008aos}, quality-diversity methods maintain elites across behavior niches~\citep{mouret2015mapelites} or search for novelty alone~\citep{lehman2011novelty}, and MuEvo revises the priorities of heuristic components with short-budget probing~\citep{lv2026muevo}. In \ours, the units receiving credit are ideas discovered during the search, and plateau handling targets the ideas shared by top programs.

\textbf{Learned mutators and self-improving agents.} Another line trains the LLM mutator with reinforcement learning, either during evolution~\citep{surina2025evotune} or at test time on a single problem~\citep{Wang2025ThetaEvolveTL,yuksekgonul2026ttt}, where experience is stored in model weights. Self-improving agents such as the Darwin G\"odel Machine and the Huxley-G\"odel Machine evolve the code of the agent itself~\citep{zhang2025dgm,wang2025hgm}, SATLUTION evolves a solver repository together with its own rulebase~\citep{yu2025satlution}, and AlphaResearch alternates between idea proposal and program verification~\citep{yu2025alpharesearch}. These directions are orthogonal to ours, which uses a fixed LLM and a fixed harness.

\textbf{Causality for scientific discovery.} There is a long history of discussions on how to understand the world through observations~\citep{greenland1999causal,spirtes2000causation,pearl2009causality}. One research line is causal discovery, which learns the causal structure among variables from observational data under assumptions such as conditional independence or constrained functional forms~\citep{spirtes1995causal,shimizu2006linear,hoyer2008nonlinear}, from interventional data and data of multiple domains~\citep{yang2018characterizing,brouillard2020differentiable,mooij2020joint,huang2020causal,perry2022causal}, and in the presence of latent variables~\citep{li2025recovery,li2025efficient}, while experiment design selects the interventions that make the causal structure identifiable~\citep{hyttinen2013experiment,Li2025CanLL}. Recently, LLMs have been used to construct structured variables out of unstructured data~\citep{coat2025discoveringreasoningcausalityhidden,lirevealing}, to estimate treatment effects from unstructured data~\citep{dhawan2024end}, and to propose instrumental variables~\citep{sheth2025can}. Compared with these works, the variables in \ours, i.e., ideas, are constructed from unstructured programs, and are evaluated through interventions that the search itself chooses adaptively; Sec.~\ref{sec:ideas} characterizes what such an archive can identify.

\section{A Causal View of the Formulation}
\label{appdx:formulation}

In this section, we present a more detailed discussion of the causal view of evolutionary discovery in Sec.~\ref{sec:pomdp}.

\subsection{Formulation of Scientific Discovery}
\label{appdx:causal_formulation}

Scientific discovery aims to uncover the underlying scientific knowledge or the causal mechanisms from interactions with the world~\citep{sci_revolution}, which can be formulated as a Partially Observed Markov Decision Process (POMDP)~\citep{kaelbling1998planning} (Fig.~\ref{fig:pomdp_appdx}).

\begin{figure}[t]
    \centering

    \begin{minipage}[b]{0.505\textwidth}
        \centering
        \resizebox{\linewidth}{!}{
        \begin{tikzpicture}[
            node distance=1.2cm and 1.2cm,
            every node/.style={font=\sffamily},
            block/.style={
                rectangle,
                draw=lineColor,
                thick,
                rounded corners=4pt,
                minimum width=3.5cm,
                minimum height=1.2cm,
                align=center,
                drop shadow={opacity=0.2, shadow xshift=1mm, shadow yshift=-1mm}
            },
            world/.style={
                circle,
                draw=lineColor,
                thick,
                double,
                fill=boxWorld,
                minimum size=1.2cm,
                align=center
            },
            line/.style={->, >=LaTeX, thick, lineColor, rounded corners=5pt}
        ]

        \node[block, fill=boxMem] (mem) {
            \textbf{Scratchpad Memory}\\
            ($m_t \to m_{t+1}$)\\
            \footnotesize \textit{"Integrate evidence"}
        };

        \node[block, fill=boxAct, below=of mem] (propose) {
            \textbf{Propose Candidate}\\
            \textbf{Program} ($p_t$)\\
            \footnotesize \textit{"Triggers experiment"}
        };

        \node[block, fill=boxEnv, below=of propose] (outcome) {
            \textbf{Observe Outcome}\\
            ($y_t$)\\
            \footnotesize \textit{"Yielding outcome"}
        };

        \node[world, left=1.0cm of outcome] (theta) {$\theta_{\mathrm{sci}}$};

        \draw[line] (mem) -- node[right, font=\footnotesize] {Guide} (propose);

        \draw[line] (propose) -- node[right, font=\footnotesize] {Execute} (outcome);

        \draw[line, dashed] (theta) -- node[above, font=\footnotesize] {} (outcome);

        \draw[line] (outcome.east) -- ++(0.5,0) |- node[pos=0.25, right, font=\footnotesize] {Provide Evidence} (mem.east);

        \begin{scope}[on background layer]
            \node[
                draw=lineColor!30, dashed,
                fit=(mem)(propose),
                inner sep=15pt,
                rounded corners=10pt,
                label={[text=lineColor!80]above:\textbf{The AI Scientist Agent}}
            ] {};
        \end{scope}

        \end{tikzpicture}
        }
    \end{minipage}
    \hfill
    \begin{minipage}[b]{0.475\textwidth} % 稍微宽一点给 PGM
        \centering
        \resizebox{\linewidth}{!}{
        \begin{tikzpicture}[
            >={LaTeX[width=2mm,length=2mm]},
            node distance=1.4cm and 2.2cm,
            every node/.style={font=\sffamily},
            state/.style={circle, draw=borderGray, thick, minimum size=1.1cm, inner sep=0pt},
            mem/.style={state, fill=memBlue},
            prog/.style={state, fill=progGreen},
            obs/.style={state, fill=obsGray},
            global/.style={state, fill=white, double},
            plate_style/.style={draw=borderGray!50, dashed, rounded corners, inner sep=8pt, label={[text=borderGray]above:#1}},
            edge_style/.style={->, thick, borderGray}
        ]

        \node[mem] (mt1) at (0,0) {$m_{t+1}$};
        \node[prog, below=of mt1] (pt1) {$p_{t+1}$};
        \node[obs, below=of pt1] (yt1) {$y_{t+1}$};

        \node[mem, left=of mt1] (mt) {$m_t$};
        \node[prog, below=of mt] (pt) {$p_t$};
        \node[obs, below=of pt] (yt) {$y_t$};

        \node[mem, right=of mt1] (mt2) {$m_{t+2}$};
        \node[prog, below=of mt2] (pt2) {$p_{t+2}$};
        \node[obs, below=of pt2] (yt2) {$y_{t+2}$};

        \node[global, below=1.8cm of yt1] (theta) {$\theta_{\mathrm{sci}}$};

        \node[plate_style={$t$}, fit=(mt)(yt)] (plate_t) {};
        \node[plate_style={$t+1$}, fit=(mt1)(yt1)] (plate_tp1) {};
        \node[plate_style={$t+2$}, fit=(mt2)(yt2)] (plate_tp2) {};

        \foreach \t in {t, t1, t2} {
            \draw[edge_style] (m\t) -- (p\t);
            \draw[edge_style] (p\t) -- (y\t);
        }

        \draw[edge_style] (theta) edge[out=150, in=270] (yt);
        \draw[edge_style] (theta) edge[out=90,  in=270] (yt1);
        \draw[edge_style] (theta) edge[out=30,  in=270] (yt2);

        \draw[edge_style] (mt) -- (mt1);
        \draw[edge_style] (pt.east) to[out=0, in=225, looseness=1.1] (mt1.south west);
        \draw[edge_style] (yt.east) to[out=0, in=270, looseness=1.3] (mt1.south);

        \draw[edge_style] (mt1) -- (mt2);
        \draw[edge_style] (pt1.east) to[out=0, in=225, looseness=1.1] (mt2.south west);
        \draw[edge_style] (yt1.east) to[out=0, in=270, looseness=1.3] (mt2.south);

        \node[left=0.8cm of mt] (start_dots) {$\cdots$};
        \draw[edge_style, dashed] (start_dots) -- (mt);
        \node[right=0.8cm of mt2] (end_dots) {$\cdots$};
        \draw[edge_style, dashed] (mt2) -- (end_dots);

        \end{tikzpicture}
        }
    \end{minipage}
    \caption{The iterative scientific discovery loop. \textbf{Left:} Conceptual flow of the agent. The agent maintains a scratchpad memory ($m$), proposes a program ($p$), and observes the outcome ($y$) which is constrained by the unknown world state ($\theta_{\mathrm{sci}}$). The outcome feeds back into the memory for the next step. \textbf{Right:} The diagram illustrates how the AI Scientist probes the unknown world state $\theta_{\mathrm{sci}}$. By proposing a candidate program $p_t$, the agent triggers an experiment yielding outcome $y_t$. This observation provides evidence about $\theta_{\mathrm{sci}}$, which is integrated into the agent's scratchpad memory $m_{t+1}$. Over time steps $t, t+1, \dots$, this recurrent process allows the agent to navigate the performance landscape and converge towards optimal programs despite the static but unknown nature of $\theta_{\mathrm{sci}}$.}
    \label{fig:pomdp_appdx}
\end{figure}
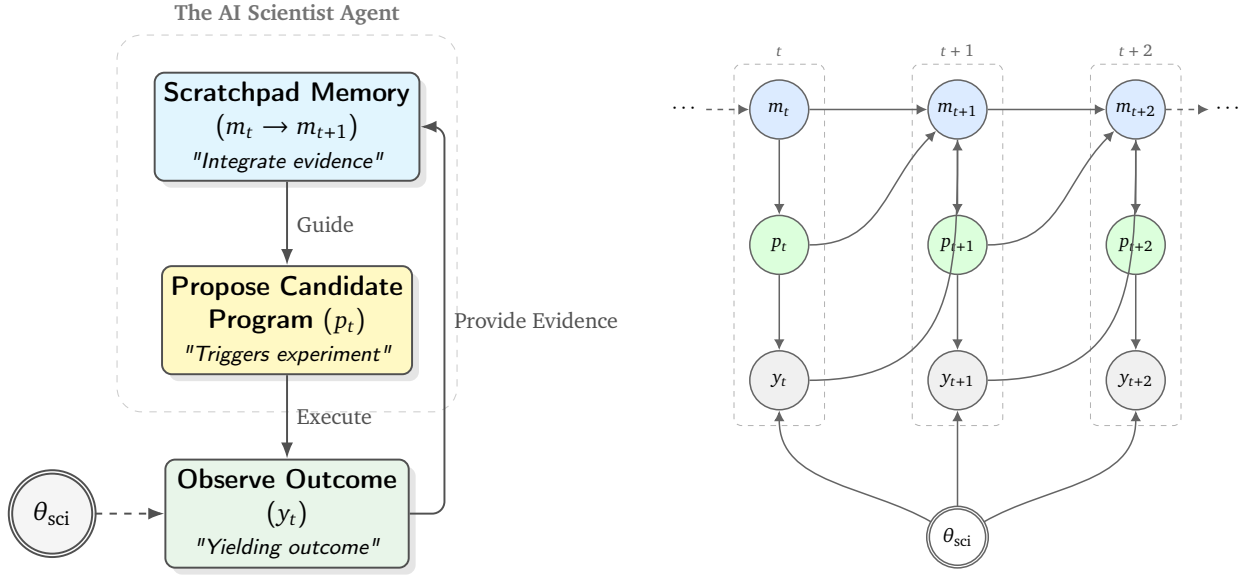

\paragraph{Scientific knowledge.} The primary objective of an AI Scientist is to uncover the \emph{underlying scientific knowledge} about the task-world, represented by a latent variable $\Theta_{\mathrm{sci}} \in \Theta$, where $\Theta$ may encode causal structure, mechanisms, inductive biases, constraints, etc. Specifically, $\Theta_{\mathrm{sci}}=\theta_{\mathrm{sci}}$ can be parameterized as a Structural Causal Model (SCM) $\theta_{\mathrm{sci}} = (\mathcal{G}, \mathcal{F}, P_U)$~\citep{spirtes2000causation}, where $\mathcal{G} = (V, E)$ is a directed graph whose nodes $V$ represent variables of interest and whose edges $E$ encode direct causal dependencies; $\mathcal{F} = \{f_v\}_{v \in V}$ is a collection of structural equations $v = f_v(\mathrm{Pa}(v), u_v)$, where $\mathrm{Pa}(v)$ denotes the parents of $v$ in $\mathcal{G}$ and $u_v$ is an exogenous noise variable; $P_U$ is a distribution over the exogenous variables $U = \{u_v\}_{v \in V}$.

\paragraph{Evaluation as intervention on SCM.} Given the SCM parametrization of $\theta_\mathrm{sci}$, we can consider that a program $p \in \mathcal{P}$ is encoded as a particular configuration of design variables $X = x_p$. Then, $F$ can be implemented as
\[
    F(p; \theta_{\mathrm{sci}})
:= \mathbb{E}\big[\, Y \,\big|\, \mathrm{do}(X = x_p),\,
       \theta_{\mathrm{sci}} \big],
\]
i.e., the expected outcome under the intervention $\mathrm{do}(X=x_p)$ in the true causal model $\theta_{\mathrm{sci}}$.
Typical implementations of $F$ can be the objective value of a combinatorial optimization problem, the efficiency of a kernel program, or the performance of a machine learning model~\citep{novikov2025alphaevolve}.

\paragraph{Belief as revealed knowledge.} As in Sec.~\ref{sec:pomdp}, the AI Scientist maintains a memory $m_t=\Phi(h_t)$ as an approximate belief about $\theta_{\mathrm{sci}}$. Each evaluation step $(p_t, y_t)$ thus updates $m_t$, which in turn updates the agent's effective belief about $\theta_{\mathrm{sci}}$. In this sense, each step \emph{reveals part of the underlying scientific knowledge}, which in turn determines the next action $p_{t+1}$.

\subsection{Ideas and Causal Representation Learning}
\label{appdx:ideas_crl}

\paragraph{Ideas as high-level causal variables.}
In the SCM view above, a program specifies a configuration $x_p$ of design variables $X$, which are low-level and high-dimensional, e.g., every line of the code. An idea $z$ (Definition~\ref{def:idea}) is a function of the program, and hence of $x_p$, that summarizes the design variables at a higher level. A faithful idea representation (Sec.~\ref{sec:ideas}) requires that $\E[Y\mid\mathrm{do}(X=x_p),\theta_{\mathrm{sci}}]$ depends on $x_p$ only through $Z(p)$, i.e., interventions on the design variables that keep all ideas fixed leave the expected outcome unchanged. Hence, an intervention on the ideas is well defined at the level of expected outcomes. This condition is related to causal feature learning~\citep{chalupka2015visual}, which constructs macro-level causal variables by grouping the micro-level states that induce the same outcome distribution under intervention, and to causal abstraction~\citep{rubenstein2017causal}, which studies when a high-level causal model is consistent with a low-level one. Faithfulness is a mean-level version of these conditions, restricted to the score of interest.

\paragraph{Relation to causal representation learning.}
Causal representation learning (CRL) aims to recover high-level causal variables from low-level observations~\citep{scholkopf2021toward}, where identifiability usually relies on additional information such as interventions on the latent variables~\citep{ahuja2023interventional}. Ideas can be viewed as the high-level causal variables of programs, and the two settings differ in three aspects. (i) In CRL, the representation is usually an encoder learned from the observations, while an idea is an interpretable function of the program that an LLM can annotate, similar to how LLMs construct high-level variables from unstructured data~\citep{coat2025discoveringreasoningcausalityhidden,lirevealing}. (ii) CRL usually aims to recover all latent variables, while ideas only need to be faithful to the score. (iii) The agent decides which programs to evaluate, and an edit that changes one idea and keeps the others fixed acts as an intervention on that idea, although an LLM edit may also change other ideas. Theorem~\ref{thm:degenerated_exploration} shows that when the evaluated programs do not vary the candidate ideas, the ideas that matter cannot be identified from the scores.

\subsection{Essentiality of Causal Knowledge for AI Scientists}
\label{appdx:essentiality}

If the evaluation revealed the scientific knowledge completely, then with more experiment turns, the agent's revealed scientific knowledge would converge to $\theta_{\mathrm{sci}}$.
However, the observation from the evaluation is usually only a proxy of the scientific knowledge $\Theta_{\mathrm{sci}}$: each evaluation returns a scalar score of the program that the agent chooses to evaluate.
Different hypotheses about $\theta_{\mathrm{sci}}$, e.g., ones under which different ideas determine the score, can produce the same scores on all evaluated programs.
For example, if all evaluated programs in circle packing share the same initialization and the same refinement, the scores cannot tell which of the two accounts for the performance.
Therefore, the causal knowledge often cannot be identified from the scores of the programs explored so far, and it is essential for an AI Scientist to explore sufficiently, i.e., to evaluate programs that distinguish the competing hypotheses.
We denote by $e_{\mathrm{src}}$ the environment in which the evaluation is conducted.
\begin{definition}[Causal AI Scientist]
\label{def:causal-ai-scientist}
A \emph{Causal AI Scientist} is an agent specified by:
(i) a policy $\pi_t(\cdot\mid \theta_t,e_{\mathrm{src}})$ selecting $p_t$,
(ii) a counterfactual / explanatory operator $\mathrm{CF}$, that answer interventional queries $(e,p)$ via $\mathrm{CF}(\theta_t;e,p)$ as an ``explanation'' of predicted performance, where $\theta_t$ is the knowledge revealed at turn $t$.
\end{definition}
Without the revealing of the causal knowledge, the discovery process suffers from significant inefficiency and suboptimality issues. We discuss the two issues more concretely below.

\textbf{Evolutionary efficiency of Causal AI Scientist.}
The efficiency issue is characterized by Theorem~\ref{thm:structured_reuse}. In the linear setting of Theorem~\ref{thm:structured_reuse}, a Causal AI Scientist can be implemented via estimating the $w^\star$ and optimizing for $\hat p$ from the history, while the black-box baseline does not consider the interactions between the historical observations.
Theorem~\ref{thm:structured_reuse} shows that, when $K\gg d$, which is usually the case as the space for all programs is significantly larger than the underlying SCM, encoding (correct) causal structure, e.g., a faithful idea representation, yields an exponential (or at least multiplicative) gain in sample efficiency under finite budgets. For example, if $K$ grows exponentially with $d$, e.g., $K=2^d$, the upper bound in Theorem~\ref{thm:structured_reuse} grows polynomially in $d$, as $O\big(d\,(d+\log\frac{1}{\delta})\frac{\sigma^2}{\epsilon^2}\big)$, while the lower bound for the black-box baseline grows exponentially in $d$, as $\Omega\big(2^d\log\frac{1}{\delta}\frac{\sigma^2}{\epsilon^2}\big)$.

\textbf{Sufficient exploration of Causal AI Scientist.}
The suboptimality issue is characterized by Theorem~\ref{thm:degenerated_exploration}.
Let $S\subseteq\mathcal{P}$ be the set of programs evaluated by the agent. Two hypotheses $\theta_0$ and $\theta_1$ can be distinguished from the history only if their observation models differ on some program in $S$; otherwise, the history has the same distribution under both hypotheses.
Under the idea-level mechanisms in Theorem~\ref{thm:degenerated_exploration}, i.e., $F(p;\theta_i)=g(Z(p),z_i(p))$ for $i\in\{0,1\}$, the two hypotheses agree on every program where $z_0$ and $z_1$ agree, so only the programs where the two ideas differ can distinguish them.
Intuitively, Theorem~\ref{thm:degenerated_exploration} implies that if the evaluated programs do not distinguish the corresponding $\theta_\mathrm{sci}$ among $\{\theta_0,\theta_1\}$, then the solution $\hat p$ incurs an expected regret of at least $\Delta/2$ under one of them, regardless of the budget.
This situation is common in evolutionary search: as the search concentrates around its current best programs, the ideas shared by these programs stop varying among the evaluated programs, and the scores cannot tell whether these ideas are the ones that matter.
Hence, a Causal AI Scientist needs to explore the programs in which the competing ideas differ, so that the scores can distinguish the hypotheses.
This barrier can be removed by exploration within the same evaluator, which differs from the transfer across environments discussed below.

\subsection{Transfer of Scientific Knowledge across Environments}
\label{appdx:transfer}

Beyond a single evaluator, the revealed scientific knowledge is also expected to transfer to other environments.
If the objective function $F$ is static universally, then with more experiment turns, the optimized solution $p_t$ and the agent's revealed scientific knowledge can also be applied universally.
However, the evaluation is usually conducted at some specific environment $e\in\mathcal{E}$, which only gives a proxy knowledge $\theta_e$ about the scientific knowledge $\Theta_\text{sci}$.
For example, the performance of a machine learning model is usually assessed on finite samples from the test distribution, and there also exist distribution shifts from the test distribution when deploying the model in the real world~\citep{datasetshift}.
Different from $\Theta_{\mathrm{sci}}$ that characterizes the complete causal structure about the scientific problem, optimization under environment $\theta_e$ may introduce some spurious correlations that maximize the objective value $F_e$~\citep{feat}.
Therefore, to retain the optimality of $\hat p$ beyond the source environment $e_\mathrm{src}$ to some target $e_\mathrm{tgt}$, it is essential to reveal the causal knowledge and answer causal questions for an AI Scientist, e.g., through the operator $\mathrm{CF}$ in Definition~\ref{def:causal-ai-scientist}. Let $P_e(\cdot\mid p,\theta)$ denote the observation model of the evaluator in environment $e$, with mean $F_e(p;\theta)$. We have the following:
\begin{theorem}\label{thm:nonidentifiability-barrier-rev}
    Consider $e_{\mathrm{src}},e_{\mathrm{tgt}}\in\mathcal{E}$ and $\theta_0,\theta_1\in\Theta$ such that $P_{e_{\mathrm{src}}}(\cdot\mid p,\theta_0)=P_{e_{\mathrm{src}}}(\cdot\mid p,\theta_1)$ for all $p\in\mathcal{P}$, and the optimal programs in $e_{\mathrm{tgt}}$ under $\theta_0$ and $\theta_1$ are different and unique, each better than all other programs by at least $\Delta>0$. Then for any policy $\pi$ that can interact only with $e_\mathrm{src}$, there exists $i\in\{0,1\}$ such that for every budget $T$, $\E\big[\max_{p\in\mathcal{P}}F_{e_{\mathrm{tgt}}}(p;\theta_i)-F_{e_{\mathrm{tgt}}}(\hat p;\theta_i)\big]\geq \Delta/2$.
\end{theorem}
The formal description and the proof are given in Appendix~\ref{appdx:thm_shifts}.
Intuitively, Theorem~\ref{thm:nonidentifiability-barrier-rev} implies that if the source environment does not distinguish the corresponding $\theta_\mathrm{sci}$ among $\{\theta_0,\theta_1\}$, then the solution $\hat p$ solved given source environment is suboptimal in the target environment under one of them, no matter how many evaluations are made in the source environment.
In the real world, it is usually the case that two machine learning models will have similar performances under the public test benchmarks, but exhibit significantly different behaviors when generalizing to distributions from other environments.
Different from Theorem~\ref{thm:degenerated_exploration}, this barrier cannot be removed by more exploration in the source environment, since the two hypotheses agree on every program there. We leave a full study of the transfer of scientific knowledge across environments to future work.

\section{Additional Technical Details}
\label{appdx:technical}

\subsection{Notation}
\label{appdx:notation}

\begin{table}[ht]
\centering
\small
\caption{Notation used in the formulation and theorems.}
\label{tab:notation-unified}
\begin{tabular}{ll}
\hline
Symbol & Meaning \\
\hline
$\mathcal{P}$ & Program / pipeline / model space (candidate designs) \\
$K$ & Number of candidate programs, $K:=|\mathcal{P}|$ (finite in Theorem~\ref{thm:structured_reuse}) \\
$p\in\mathcal{P}$ & A program to evaluate (action) \\
$\mathcal{X}$ & Design-variable space (encoding of programs) \\
$x_p\in\mathcal{X}$ & Encoding of program $p$ (e.g., design variables $X=x_p$) \\
$z$, $Z(p)$ & An idea and the idea representation of program $p$ (Definition~\ref{def:idea}) \\
$r$ & Number of ideas in an idea representation \\
$\phi(p)\in\mathbb{R}^d$ & Known representation of program $p$ in Theorem~\ref{thm:structured_reuse}, e.g., the one-hot encoding of $Z(p)$ \\
$d$ & Dimension of the representation, e.g., the number of value combinations of the ideas \\
$\Theta$ & Hypothesis space of scientific knowledge (e.g., SCMs / mechanisms) \\
$\Theta_{\mathrm{sci}}$ & Latent RV taking values in $\Theta$ (Bayesian view) \\
$\theta^\star\in\Theta$ & True (fixed but unknown) scientific knowledge instance (realization) \\
$\mu_0$ & Prior over $\Theta$ (i.e., $\Theta_{\mathrm{sci}}\sim\mu_0$) \\
$\mathcal{E}$ & Environment / protocol index set (evaluation regimes, deployments) \\
$e\in\mathcal{E}$ & Environment index; $e_{\mathrm{src}}$ source, $e_{\mathrm{tgt}}$ target \\
$F_e(p;\theta)$ & True performance in env $e$ (scalar objective) \\
$P_e(\cdot\mid p,\theta)$ & Observation model (likelihood) for evaluator output in env $e$ \\
$y_t$ & Observed evaluator outcome at round $t$ \\
$h_t$ & History $\{(p_0,y_0),\dots,(p_{t-1},y_{t-1})\}$ \\
$b_t$ & Bayesian belief/posterior over $\theta$: $b_t(\cdot)=\Pr(\Theta_{\mathrm{sci}}\in\cdot\mid h_t,e_{\mathrm{src}})$ \\
$T$ & Evaluation budget / horizon (number of program evaluations) \\
\hline
\end{tabular}
\end{table}

\subsection{Random variable, space, and realization (to avoid notation confusion)}
\label{appdx:rv-space-realization}

We use the following (standard) convention.

\textbf{(i) Hypothesis space.} $\Theta$ is a set that contains all candidate scientific-knowledge hypotheses.

\textbf{(ii) True but unknown instance.} The real world is governed by a fixed but unknown
$\theta^\star\in\Theta$.

\textbf{(iii) Bayesian view (optional but convenient).} A Bayesian agent models uncertainty by
treating $\theta^\star$ as a realization of a latent random variable $\Theta_{\mathrm{sci}}$ with prior
$\mu_0$, i.e.\ $\Theta_{\mathrm{sci}}\sim\mu_0$ and $\theta^\star$ is one draw from it.
The belief $b_t$ is simply the posterior distribution of $\Theta_{\mathrm{sci}}$ after seeing history $h_t$.

\textbf{(iv) Does scientific knowledge change across environments?}
In our formulation, the \emph{underlying} scientific knowledge $\theta^\star$ is static across rounds.
Different environments $e\in\mathcal{E}$ represent different evaluation/deployment protocols
(distribution shifts, constraint changes, measurement noise, private vs public tests, etc.).
Formally, environments affect either the true performance map $F_e(\cdot;\theta)$ and/or the observation
kernel $P_e(\cdot\mid p,\theta)$, while $\theta^\star$ itself remains the same hidden instance.

\subsection{Evaluator as an observation model (covers deterministic and stochastic evaluators)}
\label{appdx:evaluator-kernel}

Fix an environment $e\in\mathcal{E}$.
When the agent evaluates program $p$, it receives an observation $y\in\mathcal{Y}$ drawn from
\[
y \sim P_e(\cdot\mid p,\theta^\star),
\]
where $P_e(\cdot\mid p,\theta)$ is a conditional distribution on $\mathcal{Y}$.

\paragraph{Deterministic evaluator.}
A deterministic evaluator is the special case where there exists a function $g_e$ such that
\[
P_e(\cdot\mid p,\theta)=\delta_{g_e(p;\theta)}(\cdot),
\quad\text{i.e.,}\quad
y = g_e(p;\theta^\star)\ \text{a.s.}
\]
In many program-evolution settings, the evaluator is designed to deterministically
check validity and compute an objective score (e.g., via a verifier and a scoring routine).

\paragraph{Stochastic/noisy evaluator.}
A common instantiation is additive noise:
\[
y = F_e(p;\theta^\star) + \varepsilon,\qquad \varepsilon\sim\mathcal{N}(0,\sigma^2),
\]
but only the proof of Theorem~\ref{thm:structured_reuse} relies on the Gaussian form; the other results hold for any noise distribution shared by the hypotheses.

\subsection{Belief and Bayes update: kernel form and undergraduate-friendly special cases}
\label{appdx:bayes-update}

Let $h_t=\{(p_0,y_0),\dots,(p_{t-1},y_{t-1})\}$ be the history.
The Bayesian belief (posterior) is
\[
b_t(B)=\Pr(\Theta_{\mathrm{sci}}\in B\mid h_t,e_{\mathrm{src}}),\qquad B\subseteq\Theta.
\]

\paragraph{General Bayes update (kernel form).}
After choosing $p_t$ and observing $y_t$ in $e_{\mathrm{src}}$, the posterior is
\begin{equation}
b_{t+1}(B)
=
\frac{\int_{B} P_{e_{\mathrm{src}}}(dy_t\mid p_t,\theta)\, b_t(d\theta)}
{\int_{\Theta} P_{e_{\mathrm{src}}}(dy_t\mid p_t,\theta)\, b_t(d\theta)}.
\label{eq:bayes-kernel-update-unified}
\end{equation}

\paragraph{Finite hypothesis space (sum form).}
If $\Theta=\{\theta_1,\dots,\theta_N\}$ is finite and the likelihood has a pmf
$P_{e_{\mathrm{src}}}(y_t\mid p_t,\theta_i)$, then
\[
b_{t+1}(\theta_i)
=
\frac{b_t(\theta_i)\, P_{e_{\mathrm{src}}}(y_t\mid p_t,\theta_i)}
{\sum_{j=1}^N b_t(\theta_j)\, P_{e_{\mathrm{src}}}(y_t\mid p_t,\theta_j)}.
\]

\paragraph{Continuous hypothesis space (density form).}
If $P_{e_{\mathrm{src}}}(dy\mid p,\theta)$ has a density $p_{e_{\mathrm{src}}}(y\mid p,\theta)$,
then
\[
b_{t+1}(\theta)
=
\frac{b_t(\theta)\, p_{e_{\mathrm{src}}}(y_t\mid p_t,\theta)}
{\int_{\Theta} b_t(\theta')\, p_{e_{\mathrm{src}}}(y_t\mid p_t,\theta')\, d\theta'}.
\]

\paragraph{Deterministic evaluator (indicator/filter form).}
If $y=g_{e_{\mathrm{src}}}(p;\theta)$ deterministically, then the update becomes
\[
b_{t+1}(d\theta)\ \propto\ \mathbf{1}\{g_{e_{\mathrm{src}}}(p_t;\theta)=y_t\}\, b_t(d\theta),
\]
i.e.\ the posterior is the prior restricted to hypotheses consistent with the observed outcome.

\section{Proof of Theorem~\ref{thm:structured_reuse} (Structured experience reuse)}
\label{appdx:thm1-static}

Throughout this section we fix a \emph{single} static environment (drop $e$ from notation),
and assume $\mathcal{P}=\{p_1,\dots,p_K\}$ is finite.

\subsection{Protocol and performance criterion}

\paragraph{Experiment--then--commit protocol.}
A policy $\pi$ interacts for $T$ rounds. At each round $t=0,\dots,T-1$ it selects a program $p_t\in\mathcal{P}$
(possibly randomized) based on the past history $h_t$, then observes $y_t\in\mathbb{R}$.
After $T$ evaluations it outputs a final recommendation $\hat p\in\mathcal{P}$.

\paragraph{Simple regret.}
Let $f(p)$ denote the true mean performance of program $p$ in this environment.
Define the (random) simple regret
\begin{equation}
\mathrm{SR}_T
:= \max_{p\in\mathcal{P}} f(p)\;-\; f(\hat p).
\label{eq:sr-def-unified}
\end{equation}

\paragraph{$(\epsilon,\delta)$-correctness (uniform).}
Fix $\epsilon>0$ and $\delta\in(0,1)$.
We say a policy $\pi$ is \emph{$(\epsilon,\delta)$-correct uniformly on a hypothesis class $\mathcal{H}$}
if for every instance in $\mathcal{H}$,
\[
\Pr\big(\mathrm{SR}_T \le \epsilon\big)\ge 1-\delta.
\]
``Uniformly'' means the guarantee must hold for \emph{all} instances in the class, not only on average.

\subsection{Two hypothesis classes}

\textbf{(1) Structured class.}
Each program $p$ has a known representation $\phi(p)\in\mathbb{R}^d$ with $\|\phi(p)\|_2\le 1$, e.g., a faithful linear idea representation $\phi(p)=Z(p)$ (Sec.~\ref{sec:ideas}).
The unknown instance is a weight vector $w^\star\in\mathbb{R}^d$ and
\begin{equation}
f(p)=\langle \phi(p), w^\star\rangle.
\label{eq:linear-class-unified}
\end{equation}
Observations follow a Gaussian noise model
\begin{equation}
y_t = f(p_t)+\varepsilon_t,\qquad \varepsilon_t\sim\mathcal{N}(0,\sigma^2)\ \text{i.i.d.}
\label{eq:noise-unified}
\end{equation}
\emph{Coverage condition.} Assume there exist $d$ basis programs $p^{(1)},\dots,p^{(d)}$ whose representations are the standard basis:
\begin{equation}
\phi(p^{(i)})=e_i,\qquad i=1,\dots,d.
\label{eq:basis-unified}
\end{equation}

\textbf{(2) Unstructured black-box class (baseline).}
The unknown instance is an arbitrary vector of means
\[
\mu=(\mu_1,\dots,\mu_K)\in\mathbb{R}^K,\qquad f(p_i)=\mu_i,
\]
and observations are
\begin{equation}
y_t=\mu_{I_t}+\varepsilon_t,\qquad \varepsilon_t\sim\mathcal{N}(0,\sigma^2)\ \text{i.i.d.},
\label{eq:blackbox-noise-unified}
\end{equation}
where $I_t\in\{1,\dots,K\}$ is the index of the chosen program $p_t=p_{I_t}$.
Crucially, there is \emph{no assumed relation} between $\mu_i$ and $\mu_j$ for $i\neq j$.

\subsection{Formal statement and proof}

\begin{theorem}[Formal version of Theorem~\ref{thm:structured_reuse}]
\label{thm:structured-reuse-formal}
Fix $\epsilon>0$ and $\delta\in(0,1/4)$.

\begin{enumerate}
\item \textbf{(Upper bound under the structured linear class).}
Under \eqref{eq:linear-class-unified}--\eqref{eq:basis-unified} and \eqref{eq:noise-unified},
there exists a policy $\pi_{\mathrm{lin}}$ such that
\[
\Pr\big(\mathrm{SR}_T \le 2\epsilon\big)\ge 1-\delta
\quad\text{whenever}\quad
T \ \ge\ 2d\,\frac{\sigma^2}{\epsilon^2}\log\!\Big(\frac{2K}{\delta}\Big).
\]

\item \textbf{(Lower bound for the unstructured black-box class).}
For the black-box class \eqref{eq:blackbox-noise-unified}, any policy that is $(\epsilon,\delta)$-correct
\emph{uniformly for all} $\mu\in\mathbb{R}^K$ must satisfy
\[
T\ \ge\ (K-1)\,\frac{\sigma^2}{8\epsilon^2}\,\log\!\Big(\frac{1}{2\delta}\Big).
\]
\end{enumerate}
\end{theorem}
Replacing $\epsilon$ with $\epsilon/2$ in part (1) gives the informal statement in Theorem~\ref{thm:structured_reuse}.

\begin{proof}
We prove the two parts separately.

\paragraph{Part (1): constructive upper bound (estimate $w^\star$ then commit).}
Evaluate each basis program $p^{(i)}$ exactly $n$ times (total $T=nd$).
Let $y^{(i)}_1,\dots,y^{(i)}_n$ be the observations for basis $i$, and define
\[
\hat w_i := \frac{1}{n}\sum_{j=1}^n y^{(i)}_j.
\]
By \eqref{eq:linear-class-unified}--\eqref{eq:basis-unified}, $f(p^{(i)})=w^\star_i$.
By \eqref{eq:noise-unified}, $\hat w_i\sim \mathcal{N}(w^\star_i,\sigma^2/n)$ and these coordinates are independent.

Define for any program $p$:
\[
\widehat f(p):=\langle \phi(p),\hat w\rangle,\qquad \hat w=(\hat w_1,\dots,\hat w_d).
\]
Then
\[
\widehat f(p)-f(p)=\langle \phi(p),\hat w-w^\star\rangle
\sim \mathcal{N}\!\Big(0,\ \frac{\sigma^2}{n}\|\phi(p)\|_2^2\Big),
\]
so since $\|\phi(p)\|_2\le 1$,
\[
\Pr\big(|\widehat f(p)-f(p)|\ge \epsilon\big)
\le 2\exp\!\Big(-\frac{n\epsilon^2}{2\sigma^2}\Big).
\]
Union bound over $K$ programs gives
\[
\Pr\Big(\max_{p\in\mathcal{P}}|\widehat f(p)-f(p)|\ge \epsilon\Big)
\le 2K\exp\!\Big(-\frac{n\epsilon^2}{2\sigma^2}\Big).
\]
Choose
\[
n\ \ge\ 2\frac{\sigma^2}{\epsilon^2}\log\!\Big(\frac{2K}{\delta}\Big),
\]
so that with probability at least $1-\delta$ we have
$\max_{p}|\widehat f(p)-f(p)|\le\epsilon$.

Now output $\hat p:=\arg\max_{p\in\mathcal{P}}\widehat f(p)$.
Let $p^\star:=\arg\max_p f(p)$.
On the above high-probability event,
\[
f(p^\star)-f(\hat p)
\le \big(f(p^\star)-\widehat f(p^\star)\big)+\big(\widehat f(\hat p)-f(\hat p)\big)
\le \epsilon+\epsilon = 2\epsilon.
\]
Thus $\Pr(\mathrm{SR}_T\le 2\epsilon)\ge 1-\delta$ for $T=nd$ as stated.

\paragraph{Part (2): lower bound for the black-box class.}
We construct $K$ hard instances and lower bound any uniformly $(\epsilon,\delta)$-correct policy.

Let the programs be $p_1,\dots,p_K$.
Define a base instance $\mu^{(0)}\in\mathbb{R}^K$:
\[
\mu^{(0)}_1 = 0,\qquad \mu^{(0)}_i=-2\epsilon\ \ (i=2,\dots,K).
\]
For each $i\in\{2,\dots,K\}$, define an alternative instance $\mu^{(i)}$:
\[
\mu^{(i)}_1 = 0,\qquad \mu^{(i)}_i=+2\epsilon,\qquad
\mu^{(i)}_j=-2\epsilon\ \ (j\notin\{1,i\}).
\]
Under $\mu^{(0)}$, the unique best program is $p_1$, and choosing any $p_i$ with $i\ge 2$ incurs regret $2\epsilon>\epsilon$.
Under $\mu^{(i)}$, the unique best program is $p_i$, and choosing $p_1$ incurs regret $2\epsilon>\epsilon$.

Let $P_0$ be the distribution of the full transcript
$\mathcal{T}:=(p_{0:T-1},y_{0:T-1},\hat p)$ under $\mu^{(0)}$,
and $P_i$ the analogous distribution under $\mu^{(i)}$.
Uniform $(\epsilon,\delta)$-correctness implies
\[
P_0(\hat p=p_1)\ge 1-\delta,
\qquad
P_i(\hat p=p_1)\le \delta \quad (i=2,\dots,K).
\]

\medskip
\noindent\textbf{Step 1: a KL lower bound from an event.}
For any event $A$ and distributions $P,Q$, one has
\[
\mathrm{KL}(P\|Q)
\ge
P(A)\log\frac{P(A)}{Q(A)} + (1-P(A))\log\frac{1-P(A)}{1-Q(A)}.
\]
Apply it with $A=\{\hat p=p_1\}$, $P=P_0$, $Q=P_i$.
Let $p:=P_0(A)\ge 1-\delta$ and $q:=P_i(A)\le\delta$.
For $\delta\in(0,1/4)$ this yields
\begin{equation}
\mathrm{KL}(P_0\|P_i)\ \ge\ \log\!\Big(\frac{1}{2\delta}\Big).
\label{eq:kl-lb-unified}
\end{equation}

\medskip
\noindent\textbf{Step 2: compute $\mathrm{KL}(P_0\|P_i)$ via number of pulls of arm $i$.}
Under $\mu^{(0)}$ and $\mu^{(i)}$, the policy is identical; only observations when playing $p_i$ differ:
\[
y\sim \mathcal{N}(-2\epsilon,\sigma^2)\ \text{under }\mu^{(0)},\qquad
y\sim \mathcal{N}(+2\epsilon,\sigma^2)\ \text{under }\mu^{(i)}.
\]
For Gaussians with equal variance,
$\mathrm{KL}(\mathcal{N}(m_0,\sigma^2)\|\mathcal{N}(m_1,\sigma^2))=\frac{(m_0-m_1)^2}{2\sigma^2}$,
so each pull of $p_i$ contributes KL $\frac{(4\epsilon)^2}{2\sigma^2}=\frac{8\epsilon^2}{\sigma^2}$.

Let $N_i$ be the (random) number of times $p_i$ is evaluated in $T$ rounds.
Additivity of log-likelihood ratios over independent Gaussian samples yields
\begin{equation}
\mathrm{KL}(P_0\|P_i)
=
\frac{8\epsilon^2}{\sigma^2}\ \mathbb{E}_{P_0}[N_i].
\label{eq:kl-to-count-unified}
\end{equation}

\medskip
\noindent\textbf{Step 3: conclude the lower bound on $T$.}
Combine \eqref{eq:kl-lb-unified} and \eqref{eq:kl-to-count-unified}:
\[
\mathbb{E}_{P_0}[N_i]
\ \ge\
\frac{\sigma^2}{8\epsilon^2}\log\!\Big(\frac{1}{2\delta}\Big),
\qquad i=2,\dots,K.
\]
Summing over $i=2,\dots,K$ gives
\[
T=\sum_{i=1}^K N_i
\ \ge\
\sum_{i=2}^K \mathbb{E}_{P_0}[N_i]
\ \ge\
(K-1)\frac{\sigma^2}{8\epsilon^2}\log\!\Big(\frac{1}{2\delta}\Big).
\]
This completes the proof.
\end{proof}

\paragraph{Remark (deterministic evaluator).}
If $\sigma=0$, the structured linear class can recover $w^\star$ exactly from $d$ basis evaluations and achieve $\mathrm{SR}_T=0$,
while in the unstructured black-box class a uniform worst-case guarantee requires evaluating all $K$ programs at least once.

\paragraph{Relation to best-arm identification.}
The lower bound follows the standard change-of-measure argument for best-arm identification in $K$-armed Gaussian bandits~\citep{kaufmann2016complexity}, and the upper bound is a simple instance of best-arm identification in linear bandits~\citep{soare2014bestarm}.

\section{Proof of Theorem~\ref{thm:degenerated_exploration} (Degenerated exploration of ideas)}
\label{appdx:prop2}

\subsection{Formal statement and proof}

For a policy that interacts for $T$ rounds and outputs $\hat p$, define the simple regret under mechanism $\theta$ as $\mathrm{SR}_T(\theta)\coloneqq\max_{p\in\mathcal{P}}F(p;\theta)-F(\hat p;\theta)$.

\begin{theorem}[Formal version of Theorem~\ref{thm:degenerated_exploration}]
\label{thm:degenerated-exploration-formal}
Let $z_0,z_1:\mathcal{P}\to\mathcal{V}$ be two ideas, and let $A\coloneqq\{p\in\mathcal{P}: z_0(p)=z_1(p)\}$ be the set of programs on which they agree. Assume there exist $\theta_0,\theta_1\in\Theta$ such that:
\begin{enumerate}[label=(\roman*), leftmargin=*]
\item \textbf{(Idea-level mechanisms)} $F(p;\theta_i)=g(Z(p),z_i(p))$ for all $p\in\mathcal{P}$ and $i\in\{0,1\}$, for a common representation $Z$ and function $g$, and each evaluation returns $y=F(p;\theta_i)+\varepsilon$ with the same noise distribution under $\theta_0$ and $\theta_1$;
\item \textbf{(Optimal program flips with margin $\Delta$)} there exist $p_0\neq p_1$ and $\Delta>0$ such that $F(p_i;\theta_i)-F(p;\theta_i)\ge\Delta$ for all $p\neq p_i$ and $i\in\{0,1\}$.
\end{enumerate}
Then for any policy that only evaluates programs in $A$, there exists $i\in\{0,1\}$ such that for every budget $T$,
\[
\E\big[\mathrm{SR}_T(\theta_i)\big]\ \ge\ \Delta/2.
\]
Moreover, at least one of $p_0$ and $p_1$ lies outside $A$.
\end{theorem}

\begin{proof}
For every $p\in A$, $z_0(p)=z_1(p)$ gives $F(p;\theta_0)=g(Z(p),z_0(p))=g(Z(p),z_1(p))=F(p;\theta_1)$. Together with the common noise distribution in (i), the observation model $O(\cdot\mid\theta_i,p)$ is identical under $\theta_0$ and $\theta_1$ for every $p\in A$.

Let $\mathbb{P}_i$ be the distribution of the transcript $(p_{0:T-1},y_{0:T-1},\hat p)$ when the true mechanism is $\theta_i$. Since the policy only evaluates programs in $A$, for any history and any chosen $p_t\in A$, the conditional distribution of $y_t$ is identical under $\theta_0$ and $\theta_1$, and the policy chooses $p_{t+1}$ and $\hat p$ with the same randomization given the same history. By induction on $t$, $\mathbb{P}_0=\mathbb{P}_1$. In particular, the output $\hat p$ has the same distribution $Q$ on $\mathcal{P}$ under both mechanisms.

By (ii), any output $\hat p\neq p_i$ incurs regret at least $\Delta$ under $\theta_i$, so
\[
\E[\mathrm{SR}_T(\theta_0)]\ \ge\ \Delta\,\big(1-Q(\hat p=p_0)\big),\qquad
\E[\mathrm{SR}_T(\theta_1)]\ \ge\ \Delta\,\big(1-Q(\hat p=p_1)\big).
\]
Since $p_0\neq p_1$, $Q(\hat p=p_0)+Q(\hat p=p_1)\le 1$, so at least one of the two probabilities is at most $1/2$, and the corresponding expected regret is at least $\Delta/2$.

For the last claim, suppose both $p_0,p_1\in A$. Then $F(p_1;\theta_0)=F(p_1;\theta_1)\ge F(p_0;\theta_1)+\Delta=F(p_0;\theta_0)+\Delta$, which contradicts the optimality of $p_0$ under $\theta_0$.
\end{proof}

The theorem holds for any budget, since the barrier comes from the programs that the policy evaluates. To identify which of $z_0$ and $z_1$ matters, the policy has to evaluate a program outside $A$, i.e., a program on which the two ideas differ. In \ours, this is the role of plateau-aware revision, which changes the proposal distribution to evaluate programs where the ideas shared by the leading programs differ.

\section{Proof of Theorem~\ref{thm:nonidentifiability-barrier-rev} (Non-identifiability under environment shifts)}
\label{appdx:thm_shifts}

\subsection{Setup: source interaction, target evaluation, and target regret}

The agent can only interact with the \emph{source} environment $e_{\mathrm{src}}$:
\[
y_t\sim P_{e_{\mathrm{src}}}(\cdot\mid p_t,\theta^\star).
\]
After $T$ rounds it outputs a final program $\hat p$.
Performance is judged in a \emph{target} environment $e_{\mathrm{tgt}}$ via $F_{e_{\mathrm{tgt}}}(p;\theta^\star)$.
Define the target (simple) regret:
\[
\mathrm{GR}_T(\theta^\star)
:= \max_{p\in\mathcal{P}}F_{e_{\mathrm{tgt}}}(p;\theta^\star)
 - F_{e_{\mathrm{tgt}}}(\hat p;\theta^\star).
\]

\subsection{Formal statement and proof}

\begin{theorem}[Non-identifiability barrier under shifts]
\label{thm:nonidentifiability-barrier-formal}
Fix $e_{\mathrm{src}},e_{\mathrm{tgt}}\in\mathcal{E}$.
Assume there exist two hypotheses $\theta_0,\theta_1\in\Theta$ such that:

\begin{align}
&\textbf{(Source indistinguishability)}\quad
P_{e_{\mathrm{src}}}(\cdot\mid p,\theta_0)=P_{e_{\mathrm{src}}}(\cdot\mid p,\theta_1),
\quad \forall p\in\mathcal{P}. \label{eq:src-indist-unified}\\[2mm]
&\textbf{(Target optimal action flips with margin $\Delta$)}\quad
\exists\, p_0,p_1\in\mathcal{P}\ \text{and}\ \Delta>0\ \text{s.t.}\\
&\qquad p_0\in\arg\max_{p\in\mathcal{P}}F_{e_{\mathrm{tgt}}}(p;\theta_0),
\quad
p_1\in\arg\max_{p\in\mathcal{P}}F_{e_{\mathrm{tgt}}}(p;\theta_1), \nonumber\\
&\qquad F_{e_{\mathrm{tgt}}}(p_0;\theta_0)-F_{e_{\mathrm{tgt}}}(p;\theta_0)\ge \Delta,\ \forall p\neq p_0,
\quad
F_{e_{\mathrm{tgt}}}(p_1;\theta_1)-F_{e_{\mathrm{tgt}}}(p;\theta_1)\ge \Delta,\ \forall p\neq p_1.
\label{eq:tgt-flip-gap-unified}
\end{align}

Then for any policy $\pi$ that can interact only with $e_{\mathrm{src}}$,
there exists $i\in\{0,1\}$ such that for every budget $T$,
\[
\mathbb{E}\big[\mathrm{GR}_T(\theta_i)\big]\ \ge\ \Delta/2.
\]
This impossibility holds whether the evaluator is stochastic or deterministic, since
\eqref{eq:src-indist-unified} is stated at the level of the full observation model $P_{e_{\mathrm{src}}}$.
\end{theorem}

\begin{proof}
Let $\mathbb{P}_i$ be the distribution over the full transcript
\[
\mathcal{T}:=(p_{0:T-1},y_{0:T-1},\hat p)
\]
when the true hypothesis is $\theta_i$ and interaction is only with $e_{\mathrm{src}}$.

By \eqref{eq:src-indist-unified}, for any history and any chosen action $p_t$,
the conditional distribution of $y_t$ is identical under $\theta_0$ and $\theta_1$.
By induction on $t$, the entire transcript distribution is identical:
\[
\mathbb{P}_0=\mathbb{P}_1.
\]
In particular, the marginal distribution of the final output $\hat p$ is the same under $\theta_0$ and $\theta_1$.
Let this common distribution be denoted by $Q$ on $\mathcal{P}$.

Now consider the expected target regret under $\theta_0$:
by \eqref{eq:tgt-flip-gap-unified}, any output $\hat p\neq p_0$ incurs regret at least $\Delta$ under $\theta_0$:
\[
\mathrm{GR}_T(\theta_0)=F_{e_{\mathrm{tgt}}}(p_0;\theta_0)-F_{e_{\mathrm{tgt}}}(\hat p;\theta_0)\ \ge\ \Delta\cdot \mathbf{1}\{\hat p\neq p_0\}.
\]
Taking expectation w.r.t.\ $Q$ yields
\[
\mathbb{E}[\mathrm{GR}_T(\theta_0)]\ \ge\ \Delta\cdot (1-Q(\hat p=p_0)).
\]
Similarly,
\[
\mathbb{E}[\mathrm{GR}_T(\theta_1)]\ \ge\ \Delta\cdot (1-Q(\hat p=p_1)).
\]
Since $Q(\hat p=p_0)+Q(\hat p=p_1)\le 1$, at least one of these probabilities is at most $1/2$,
so at least one of the two expected regrets is at least $\Delta/2$:
\[
\max\{\mathbb{E}[\mathrm{GR}_T(\theta_0)],\ \mathbb{E}[\mathrm{GR}_T(\theta_1)]\}\ \ge\ \Delta/2.
\]
This proves the claim.
\end{proof}

\subsection{Concrete examples satisfying the conditions}
\label{appdx:thm2-examples}

We give two illustrative examples where source data cannot distinguish two hypotheses, yet the target-optimal decision differs.

\paragraph{Example 1: public test vs private (distribution shift / shortcut feature).}
Let $\theta\in\{\theta_0,\theta_1\}$ encode which feature is truly stable/causal.
Programs correspond to two model families:
$p_0$ uses a stable causal feature; $p_1$ uses a shortcut feature.
In the source environment (public benchmark), the shortcut feature is perfectly correlated with labels,
so both hypotheses yield the same evaluator distribution for every program, satisfying \eqref{eq:src-indist-unified}.
In the target environment (deployment/private), the shortcut correlation breaks:
under $\theta_0$, $p_0$ is uniquely optimal; under $\theta_1$, $p_1$ is uniquely optimal, with margin $\Delta$,
satisfying \eqref{eq:tgt-flip-gap-unified}.
No amount of interaction with $e_{\mathrm{src}}$ can identify which world holds.

\paragraph{Example 2: relaxed verification (slack) vs exact verification (constraint shift).}
In combinatorial optimization, it is common to evaluate candidate programs using a \emph{relaxed} verifier
(e.g., allowing numerical slack), then validate with an \emph{exact} verifier.
For instance, in circle packing, one may verify non-overlap with a numerical slack such as $10^{-6}$,
and later validate with an exact checker; converting a relaxed-feasible solution into an exact-feasible one
may require tiny but nonzero modifications, and rankings can change when switching verifiers.
This is explicitly discussed in the context of circle packing verification with slack vs exact validation.
The source environment $e_{\mathrm{src}}$ can correspond to the relaxed evaluator, while the target
environment $e_{\mathrm{tgt}}$ corresponds to the exact evaluator.
Then two hypotheses $\theta_0,\theta_1$ can be constructed so that they are indistinguishable under the relaxed evaluator
for all queried programs, yet the exact evaluator reverses which program is truly best (with gap $\Delta$),
matching Theorem~\ref{thm:nonidentifiability-barrier-formal}.

\section{Implementation Details of \ours}
\label{appdx:method}

This section describes the implementation of \ours used in our experiments. Alg.~\ref{alg:expevolve} summarizes the search loop, and Alg.~\ref{alg:discover} details the discovery of idea factors. We implement \ours on the backbone of \scs~\citep{Gideoni2026SimpleBA}. In each generation, \ours generates $B$ programs in parallel slots. The prompt of each slot contains the task instruction, an optional strategy instruction from an idea-guided action, a summary of the idea memory, and a few example programs, and the mutator LLM $G$ writes a new program from the prompt. All LLM calls use DeepSeek-V4-Flash, and the hyperparameters are listed in Table~\ref{tab:expevolve_hparams}. Each generated program is evaluated once. A program is treated as failed if the evaluation fails or returns a non-finite objective, and failed programs are neither repaired nor retried. The LLM calls for generation are not retried in our experiments either, so a failed call, e.g., due to rate limits, yields a failed program that still counts toward the evaluation budget.

\definecolor{IEteal}{HTML}{087F7B}\definecolor{IEamber}{HTML}{C97A1A}\definecolor{IEplum}{HTML}{7A4E8C}
\providecommand{\proc}[1]{\textsc{#1}}
\providecommand{\algblock}[2]{\item[]{\color{#1}\textbf{#2}}}

\begin{algorithm}[t]
\caption{\ours}
\label{alg:expevolve}
\small
\begin{algorithmic}[1]
\REQUIRE initial archive $\mathcal{A}_0$ of the task; mutator LLM $G$; generations $N$; batch size $B$; random-exploration fraction $\rho$; plateau window $\kappa$
\STATE $\mathcal{M}_0\gets\varnothing$
\FOR{$t=0,\dots,N-1$}
\algblock{IEteal}{Exploit idea factors (Sec.~\ref{sec:method_guide})}
\STATE $(u_{t,1},\dots,u_{t,B})\gets\proc{Allocate}(\mathcal{M}_t,B,\rho)$ \COMMENT{$\bot$: random exploration}
\FOR{$b=1,\dots,B$}
\STATE $p_{t,b}\sim G\big(\cdot\mid\proc{BuildContext}(u_{t,b},\mathcal{A}_t,\mathcal{M}_t)\big)$;\ $y_{t,b}\gets\proc{Evaluate}(p_{t,b})$
\ENDFOR
\algblock{IEamber}{Discover ideas from experience (Sec.~\ref{sec:method_extract})}
\STATE $\mathcal{A}_{t+1}\gets\mathcal{A}_t\cup\{(p_{t,b},y_{t,b},u_{t,b})\}_{b=1}^{B}$
\STATE $\mathcal{M}_{t+1}\gets\proc{UpdateMemory}(\mathcal{M}_t,\mathcal{A}_{t+1})$ \COMMENT{annotate; merge; update rewards}
\IF{$\proc{DiscoverScheduled}(t)$}
\STATE $\mathcal{M}_{t+1}\gets\proc{DiscoverIdeas}(\mathcal{A}_{t+1},\mathcal{M}_{t+1})$ \COMMENT{Alg.~\ref{alg:discover}}
\ENDIF
\algblock{IEplum}{Revise at plateaus (Sec.~\ref{sec:method_plateau})}
\IF{$\proc{EntersPlateau}(\mathcal{A}_{t+1},\kappa)$}
\STATE $\mathcal{M}_{t+1}\gets\proc{ReviseAtPlateau}(\mathcal{A}_{t+1},\mathcal{M}_{t+1})$ \COMMENT{probe; judge; mask}
\ENDIF
\ENDFOR
\STATE \textbf{return} $\arg\max_{(p,y,\cdot)\in\mathcal{A}_N}y$ and $\mathcal{M}_N$
\end{algorithmic}
\end{algorithm}

\begin{algorithm}[t]
\caption{\proc{DiscoverIdeas}}
\label{alg:discover}
\small
\begin{algorithmic}[1]
\REQUIRE archive $\mathcal{A}$; idea memory $\mathcal{M}$; pair budget $\ell$; minimum support $n_{\min}$; admission threshold $\eta$
\ENSURE updated idea memory $\mathcal{M}$
\algblock{IEteal}{Select unexplained contrasts}
\IF{$\mathcal{M}=\varnothing$}
\STATE $Q\gets\proc{LargestScoreGaps}(\mathcal{A},\ell)$
\ELSE
\STATE $Q\gets\proc{MisrankedPairs}(\mathcal{A},\mathcal{M},\ell)$ \COMMENT{pairwise ranker on $Z(p)$}
\ENDIF
\algblock{IEamber}{Propose candidate ideas}
\STATE $\widetilde{\mathcal{M}}\gets\proc{ProposeIdeas}(Q)$ \COMMENT{name, description, type, values}
\algblock{IEplum}{Annotate, test, and admit}
\FOR{each $z\in\widetilde{\mathcal{M}}$}
\STATE $z(p)\gets\proc{Annotate}(p,z)$ for all $p\in\mathcal{A}$ \COMMENT{in $\mathcal{V}\cup\{\mathsf{unknown}\}$}
\IF{$\proc{Supported}(z,\mathcal{A},n_{\min})$ \AND $|\proc{ScoreContrast}(z,\mathcal{A})|\ge\eta$ \AND \textbf{not} $\proc{Duplicate}(z,\mathcal{M})$}
\STATE $\mathcal{M}\gets\mathcal{M}\cup\{z\}$
\ENDIF
\ENDFOR
\STATE \textbf{return} $\mathcal{M}$
\end{algorithmic}
\end{algorithm}

\subsection{Idea Memory and Annotation}
\label{appdx:method_memory}

\textbf{Idea factors.} Each idea factor $z$ in the idea memory $\mathcal{M}_t$ is stored with a name, a description of its code pattern, a type, i.e., binary or categorical, and a value set $\mathcal{V}$. The description is also used by the annotator LLM to recognize the idea factor from the source code. The annotator LLM labels each program with $z(p)\in\mathcal{V}\cup\{\mathsf{unknown}\}$, where the label is $\mathsf{unknown}$ if the idea factor cannot be recognized. The annotation of a program is a separate LLM call, independent of the proposal of idea factors, with the source code truncated to $4000$ characters.

\textbf{Score contrast.} For a binary idea factor, the score contrast is the difference of the mean scores between the successful programs labeled $\mathsf{true}$ and those labeled $\mathsf{false}$. For a categorical idea factor, the score contrast is the difference of the mean scores between the value with the highest mean score and the value with the lowest mean score. Programs labeled $\mathsf{unknown}$ are excluded from the score contrast.

\textbf{Updating the memory.} After each generation, \textsc{UpdateMemory} labels the new programs with all existing idea factors, recomputes the score contrasts, and merges two idea factors if the Pearson correlation of their labels on the archive exceeds $0.85$ (Alg.~\ref{alg:expevolve}, line~8). Degenerate evaluations, whose scores collapse to extreme values, are excluded from the statistics. It also updates the statistics of the actions (Appendix~\ref{appdx:method_actions}).

\textbf{Memory summary in the prompt.} Every slot, including random exploration, receives a summary of the idea memory (line~5). The summary lists the six idea factors with the largest absolute score contrasts, each with its name, description, direction of effect, numbers of programs with and without it, and score contrast, and gives the mean score of each value for categorical idea factors. It also contains the outlier programs that the current idea factors fail to explain, i.e., one program with a large positive residual and one with a large negative residual ($|z|\ge1.5$), with their source code truncated to $1500$ characters.

\subsection{Idea Discovery}
\label{appdx:method_discover}

\textbf{Trigger.} Discovery is triggered every $n_{\mathrm{disc}}$ generations after the first generation, when the archive contains at least $15$ successful programs and the idea memory contains fewer than $12$ idea factors, excluding the background idea factors (Alg.~\ref{alg:expevolve}, line~9).

\textbf{Pairwise ranker.} Let $\sigma_Y$ be the standard deviation of the scores of the successful programs in the archive. \ours builds all pairs of successful programs $(p,p')$ with $|y(p)-y(p')|\ge0.25\,\sigma_Y$, and labels each pair by $\mathrm{sign}(y(p)-y(p'))$. The feature of a pair is the difference of the binary encodings of the idea factors, where a factor is encoded as $1$ if it is labeled $\mathsf{true}$ and $0$ otherwise, together with the standardized difference of the code lengths as a covariate. The ranker is a ridge least-squares model without intercept ($\lambda=10^{-3}$), refit in each discovery round.

\textbf{Selecting unexplained pairs.} \ours selects the pairs that the ranker orders incorrectly (Alg.~\ref{alg:discover}, line~4), and excludes the pairs on which the current idea factors take identical values. The selected pairs are ranked by $\mathrm{acc}_0(p,p')\cdot w(p,p')$, where $w(p,p')=1.5$ if the two programs come from the same generation and $1$ otherwise. The null accuracy $\mathrm{acc}_0(p,p')$ is the fraction of $500$ random permutations of the scores over programs, under which a ranker refit on the permuted labels still predicts the true order of the pair. A pair with a high null accuracy is easy to order from the idea factors, so that its misranking indicates a missing idea factor. The null accuracy uses the differences of the idea factors without the code-length covariate. Up to $\ell=6$ pairs are used, with each program appearing in at most two pairs. If fewer than $3$ pairs are available, the discovery round is skipped. When the idea memory is empty, \ours selects the $\ell$ pairs with the largest score gaps under the same constraints (lines~1--2).

\textbf{Proposing candidates.} For each selected pair, the prompt of the proposer LLM contains the score gap in raw units and in units of $\sigma_Y$, the scores, the predicted scores and the residuals of both programs, the order predicted by the ranker, the idea factors on which the two programs differ, and the source code of both programs truncated to $1500$ characters each. The proposer LLM is asked to explain the unexpected order, and proposes $8$ candidates in each round, each with a name, a description, a type, a value set, and a predicted direction (line~6).

\textbf{Annotation and admission.} Each candidate is annotated on the whole archive (line~8). \ours admits a candidate if the two values compared by its score contrast are each taken by at least $n_{\min}=5$ programs, the absolute score contrast is at least $\eta=\max(0.25\,\sigma_Y,\,10^{-3})$, and the candidate does not duplicate an idea factor in the memory, an idea factor removed before, or another selected candidate (line~9). Two idea factors are considered duplicates if their names are identical, the token Jaccard similarity of their names is at least $0.6$, or the token overlap of their descriptions is at least $0.45$. Since the archive is collected adaptively by the search, the score contrast is an association within the archive and does not identify the effect of the idea factor. Admission therefore makes an idea factor available for testing, and the search examines it through the actions (Appendix~\ref{appdx:method_actions}).

\subsection{Idea-Guided Actions}
\label{appdx:method_actions}

\textbf{Actions.} \textsc{Allocate} compiles the idea factors into actions, i.e., $\mathsf{ADD}(z)$ and $\mathsf{REMOVE}(z)$ for a binary idea factor, and $\mathsf{SET}(z{=}v)$ for each value $v$ of a categorical idea factor (Alg.~\ref{alg:expevolve}, line~3). The action set is rebuilt in every generation, and keeps at most three actions of the idea factors with the largest absolute score contrasts, together with a baseline action without instruction. An action that has stayed in the action set for $N_{\mathrm{rot}}$ consecutive generations is rotated out to give way to other actions, and its records are kept for when it enters the action set again. Before an admitted idea factor is compiled into actions, its absolute score contrast must exceed the $95$th percentile of a null distribution obtained from $200$ permutations of the scores, and a threshold on the absolute score contrast is used instead when the permutation test cannot be run. Background idea factors (Appendix~\ref{appdx:method_plateau}) are not compiled into actions, until they satisfy the admission conditions of Appendix~\ref{appdx:method_discover} and are promoted to regular idea factors.

\textbf{Prompt of an action.} An idea-guided slot differs from random exploration in two parts of the prompt (line~5). First, a strategy instruction asks the mutator LLM to use or avoid the pattern of $z$, or to set $z$ to $v$, with the description of $z$ and its estimated effect. Second, the example programs consist of the best program in the archive as an anchor, followed by the best programs in the target state of the action, where ties between programs with similar scores are broken in favor of programs with rare combinations of idea factors, and $k=3$ example programs are used in total. There is no explicit parent program, since \scs conditions the mutator LLM on a set of example programs.

\textbf{Random exploration.} A random exploration slot uses the prompt of \scs without strategy instruction, with $3$ example programs sampled uniformly at random from the successful programs, and the summary of the idea memory.

\textbf{Diversification in the warm-up.} In the warm-up generations, the slots of the baseline action in the second half of each generation receive a diversification instruction, which asks the mutator LLM to first list three to five distinct families of well-established approaches to the task as code comments, and to note which families the example programs use.

\textbf{Allocating slots.} In each generation, \textsc{Allocate} assigns $\max(2,\mathrm{round}(\rho B))$ slots to random exploration, and the remaining slots to idea-guided actions, where $u_{t,b}$ denotes the action of slot $b$ and $u_{t,b}=\bot$ denotes random exploration (line~3). An idea-guided slot falls back to random exploration if no action is available, e.g., before any idea factor is admitted, or if the actions have reached the limit of slots per generation. The fraction of random exploration is thus at least $\rho$.

\textbf{Bandit.} The bandit selects the actions of idea-guided slots with a soft UCB rule,
\begin{equation*}
    s(a)=\tilde r_{\max}(a)+c\sqrt{\frac{\ln(N_{\mathrm{tot}}+1)}{n_a}},\qquad \Pr(a)\propto\exp\big(s(a)\big),
\end{equation*}
where $n_a$ is the number of times that $a$ has been selected, $N_{\mathrm{tot}}=\sum_a n_a$, and $c=1$. Actions that have not been selected are tried first. The value $\tilde r_{\max}(a)$ is the maximum reward of $a$, minus a bootstrap estimate of the expected maximum of $n_a$ rewards drawn from all actions, which corrects the growth of the maximum with the number of samples. The reward of a program generated under an action follows the normalized reward in Sec.~\ref{sec:method_guide}, where the standard deviation is lower bounded as $\sigma_{\mathrm{rand}}=\max\big(\mathrm{std}(y_{\mathrm{rand}}),\,0.05\,|\bar y_{\mathrm{rand}}|,\,10^{-6}\big)$ (line~8). If fewer than two random exploration programs are available in the generation, the adjacent generations are included. A failed program receives a penalty score below the worst successful program.

\subsection{Idea-Guided Exploration at Plateaus}
\label{appdx:method_plateau}

\textbf{Entering and leaving a plateau.} \ours records the best-so-far score at the end of each generation, and the search is in a plateau if the best-so-far scores of the last $\kappa+1$ generations are equal up to a relative tolerance of $10^{-4}$, after the warm-up generations. When the search enters a plateau, \textsc{ReviseAtPlateau} runs the background probe and the dependency judge below (Alg.~\ref{alg:expevolve}, lines~12--13). The plateau ends as soon as the best-so-far score improves, and all the changes during the plateau are reverted in the same generation. Each time the search enters a plateau again, the probe and the judge are run again, since the background of a new plateau can differ from the previous one.

\textbf{Background probe.} An LLM reads the $8$ successful programs with the highest scores, with the source code truncated to $1500$ characters and the names of the current idea factors, and proposes $2$ to $4$ binary properties that these programs share. Each property is annotated on the archive, and is recorded as a background idea factor if at least $5$ of the $8$ top programs are annotated and at least $80\%$ of them realize the property. A background idea factor stays in the memory after the plateau, and is promoted to a regular idea factor once enough programs without it appear in the archive and its score contrast passes the admission threshold.

\textbf{Dependency judge.} The background is described by the three idea factors with the largest absolute score contrasts. For each idea factor in the memory, and for each value of a categorical idea factor, an LLM judge decides, with three calls, whether it depends on the background, i.e., loses its meaning or reverses its effect once the defining property of the background no longer holds.

\textbf{Changes during a plateau.} While the search is in a plateau, \ours applies the following changes.
\begin{itemize}
    \item \emph{Reversing and masking actions.} For a binary idea factor judged dependent, $\mathsf{ADD}(z)$ is replaced by $\mathsf{REMOVE}(z)$ in the action set, with the same description and score contrast, so that the bandit can only move away from $z$. For a categorical idea factor, the actions of the values judged dependent are masked, while the other values are kept as escape actions. The dependent idea factors are also removed from the summary of the idea memory. After the plateau, the original actions are restored, and the records of the reversed and masked actions are kept.
    \item \emph{Stall note.} A note on the current state is prepended to the prompt of every slot. The note reports the best score and the number of generations without improvement, lists the dependent idea factors with their names and shortened descriptions, states that the search has stalled on this combination, and asks whether other constructions or combinations can reach a higher score.
    \item \emph{More random exploration.} Each idea-guided slot is reassigned to random exploration with probability $0.9$. Since the judge runs at the end of the first plateau generation, the probability is lowered to $0.3$ from the next generation.
    \item \emph{Diversified demonstrations.} Random exploration uses the best program of the background as the first example, and fills the other examples with programs outside the background, which are selected greedily to diversify their idea factors. A program is considered in the background if it realizes any dependent idea factor.
    \item \emph{Exploration of the bandit.} In the first $2$ generations of a plateau, the bandit selects the actions by the exploration bonus only.
\end{itemize}

\begin{table}[t]
\centering
\caption{Hyperparameters of \ours in the main experiments. The budget of $200$ evaluations is run with two configurations of generations $N$ and batch size $B$. Both configurations are run on all $11$ tasks, each with $3$ seeds.}
\label{tab:expevolve_hparams}
\small
\begin{tabular}{lcc}
\toprule
Hyperparameter & $N=20$, $B=10$ & $N=10$, $B=20$ \\
\midrule
Fraction of random exploration $\rho$ & $0.6$ & $0.4$ \\
Discovery interval $n_{\mathrm{disc}}$ (generations) & $2$ & $1$ \\
Plateau window $\kappa$ (generations) & $4$ & $2$ \\
Warm-up generations & $2$ & $1$ \\
Action rotation $N_{\mathrm{rot}}$ (generations) & $6$ & $3$ \\
\midrule
\multicolumn{3}{l}{\emph{Shared by both configurations}} \\
Minimum archive size for discovery & \multicolumn{2}{c}{$15$} \\
Maximum number of idea factors & \multicolumn{2}{c}{$12$} \\
Score gap of candidate pairs & \multicolumn{2}{c}{$0.25\,\sigma_Y$} \\
Pair budget $\ell$ & \multicolumn{2}{c}{$6$} \\
Weight of same-generation pairs & \multicolumn{2}{c}{$1.5$} \\
Permutations for the null accuracy & \multicolumn{2}{c}{$500$} \\
Permutations and quantile of the action gate & \multicolumn{2}{c}{$200$, $0.95$} \\
Candidates per discovery round & \multicolumn{2}{c}{$8$} \\
Minimum support $n_{\min}$ & \multicolumn{2}{c}{$5$} \\
Admission threshold $\eta$ & \multicolumn{2}{c}{$\max(0.25\,\sigma_Y,10^{-3})$} \\
Actions in the action set & \multicolumn{2}{c}{$\le 3$ and a baseline action} \\
Example programs per prompt $k$ & \multicolumn{2}{c}{$3$} \\
UCB coefficient $c$ & \multicolumn{2}{c}{$1$} \\
Top programs in the background probe & \multicolumn{2}{c}{$8$} \\
Reassignment probability at plateaus & \multicolumn{2}{c}{$0.9$, then $0.3$} \\
Bandit exploration window (generations) & \multicolumn{2}{c}{$2$} \\
LLM of all calls & \multicolumn{2}{c}{DeepSeek-V4-Flash} \\
Temperature of proposer / annotator / judge & \multicolumn{2}{c}{$0.7$ / $0.2$ / $0.3$} \\
\bottomrule
\end{tabular}
\end{table}

\section{Experimental Details}
\label{appdx:exp}

\subsection{Tasks}
\label{appdx:tasks}

Table~\ref{tab:tasks} summarizes the $11$ tasks. Nine tasks are the mathematical problems studied by \scs~\citep{Gideoni2026SimpleBA}, which are taken from \alevolve~\citep{novikov2025alphaevolve}, and we report their reference values as collected by \citet{Gideoni2026SimpleBA}. All methods use the same evaluator for each task, and each evaluation is limited to $60$ seconds, except the uncertainty inequality with $300$ seconds.

\begin{table}[ht]
\centering
\caption{Tasks, optimization directions, and reference values. The previous best values and the values of \alevolve are collected by \citet{Gideoni2026SimpleBA}. For the uncertainty inequality, all methods in \citet{Gideoni2026SimpleBA} reach $0.3482$ under their reformulated parameterization, which we use in our experiments. For the Hadamard matrix, the previous best is the normalized score of the best-known construction.}
\label{tab:tasks}
\small
\begin{tabular}{lccc}
\toprule
Task & Direction & Previous best & \alevolve \\
\midrule
First autocorrelation inequality & $\downarrow$ & $1.5098$ & $1.5053$ \\
Second autocorrelation inequality & $\uparrow$ & $0.88922$ & $0.8962$ \\
Uncertainty inequality & $\downarrow$ & $0.3523$ & $0.3521$ \\
Erd\H{o}s minimum overlap & $\downarrow$ & $0.380927$ & $0.380924$ \\
Sums and differences of sets & $\uparrow$ & $1.14465$ & $1.1584$ \\
Max--min distance ratio ($16$ points) & $\downarrow$ & $12.890$ & $12.88927$ \\
Heilbronn for triangles ($n=11$) & $\uparrow$ & $0.036$ & $0.0365$ \\
Kissing number ($11$ dimensions) & $\uparrow$ & $592$ & $593$ \\
Circle packing ($26$ circles) & $\uparrow$ & $2.634$ & $2.63586$ \\
Circle packing via LP ($26$ circles) & $\uparrow$ & -- & -- \\
Hadamard matrix ($n=29$) & $\uparrow$ & $0.9357$ & -- \\
\bottomrule
\end{tabular}
\end{table}

\textbf{First autocorrelation inequality.} Let $C_1$ be the largest constant such that $\max_{t}(f*f)(t)\ge C_1\big(\int f\big)^2$ for all non-negative functions $f$ supported on $[-1/4,1/4]$, where $f*f$ is the autoconvolution. A non-negative step function $f$ gives an upper bound $C_1\le\max_t(f*f)(t)/(\int f)^2$, which the task minimizes. The number of steps is chosen by the program, and for a step function with $n$ steps and values $h$, the score is computed as $2n\max(h*h)/(\sum h)^2$, which is invariant to the scale of $h$.

\textbf{Second autocorrelation inequality.} Let $C_2$ be the smallest constant such that $\|f*f\|_2^2\le C_2\,\|f*f\|_1\,\|f*f\|_\infty$ for all non-negative functions $f$. A non-negative step function $f$ gives a lower bound $C_2\ge\|f*f\|_2^2/(\|f*f\|_1\|f*f\|_\infty)$, which the task maximizes. The evaluator rejects step functions with more than $10^4$ steps.

\textbf{Uncertainty inequality.} For a function $g$, let $A(g)=\inf\{r>0:g(x)\ge0\text{ for all }|x|\ge r\}$. For even functions $f$ with $f(0)\le0$ and $\hat f(0)\le0$, the task minimizes $A(f)A(\hat f)$, which upper bounds the largest constant $C$ such that $A(f)A(\hat f)\ge C$. Following the reformulation of \citet{Gideoni2026SimpleBA}, $f$ is a Hermite--Gaussian expansion over the monic (probabilist's) Hermite polynomials $H_0,H_4,\dots,H_{28}$ with $7$ free coefficients, which represents a larger class of functions than the original formulation over $H_0,H_4,H_8,H_{12}$ with $3$ free coefficients. The sign conditions are checked on a grid of $2\times10^{5}$ points in $[10^{-6},30]$, which is refined to $2\times10^{6}$ points with extended precision when the score is below $0.6$, and each evaluation is limited to $300$ seconds. Since the reformulation enlarges the search space, the resulting values are not comparable with the values of the original formulation in Table~\ref{tab:tasks}. In the implementation, the score is $r^2/(2\pi)$, where $r$ is the largest root at which $f$ changes sign.

\textbf{Erd\H{o}s minimum overlap problem.} For a partition of $\{1,\dots,2n\}$ into two sets $A$ and $B$ of equal size, let $M(n)=\min_{A,B}\max_k|\{(a,b)\in A\times B:a-b=k\}|$. The task constructs a step function $h$ with values in $[0,1]$, whose number of steps is chosen by the program, and minimizes the upper bound on $\lim_{n\to\infty}M(n)/n$ given by twice the maximum correlation between $h$ and $1-h$ divided by the number of steps.

\textbf{Sums and differences of sets.} Following the framework of Gyarmati, Hennecart, and Ruzsa, the task constructs a set $U\subset\mathbb{N}_{\ge0}$ with $0\in U$ that maximizes $\theta(U)=1+\log(|U-U|/|U+U|)/\log(2\max(U)+1)$, which lower bounds the largest exponent $\theta$ such that $|A-B|\ge c(K)|A+B|^{\theta}$ for arbitrarily large sets $A,B$ with $|A+B|\le K|A|$.

\textbf{Max--min distance ratio.} The task places $16$ points in the plane to minimize the squared ratio between the maximum and the minimum pairwise distances. For example, the squared ratios $12.889$ and $13.0$ in Table~\ref{tab:main_results} correspond to the ratios $3.5902$ and $3.6056$.

\textbf{Heilbronn problem for triangles.} The task places $11$ points in the equilateral triangle with vertices $(0,0)$, $(1,0)$, and $(1/2,\sqrt3/2)$ to maximize the smallest area of the triangles formed by any three of the points, normalized by the area $\sqrt3/4$ of the equilateral triangle.

\textbf{Kissing number in $11$ dimensions.} A set of points $C\subset\mathbb{R}^{11}$ with $0\notin C$ and $\min_{x\neq y}\|x-y\|\ge\max_{x}\|x\|$ over $x,y\in C$ yields a kissing configuration of $|C|$ unit spheres. The task constructs such a set, and the score is the number of points $|C|$. The distances and norms are only used to check the validity of the configuration, and a configuration with more points than the known upper bound $868$ is rejected.

\textbf{Circle packing.} The task places $26$ non-overlapping circles in the unit square to maximize the sum of their radii. The evaluator checks the non-overlap and the containment of all circles, including that every radius is at most $0.5$ and every center lies in the unit square, and checks that the sum of radii reported by the program agrees with the radii up to a relative and absolute tolerance of $10^{-3}$.

\textbf{Circle packing via linear programming.} A variant of circle packing where the program proposes only the $26$ centers $C_i\in[0,1]^2$, and the radii are recovered by a linear program that maximizes $\sum_i r_i$ subject to $r_i+r_j\le\|C_i-C_j\|$ and $r_i\le\min(C_i^x,C_i^y,1-C_i^x,1-C_i^y)$. The variant isolates the layout of the circles from the assignment of the radii. The starting code in the prompt solves the linear program with the HiGHS solver of SciPy.

\textbf{Hadamard matrix.} The task constructs a matrix $H\in\{\pm1\}^{29\times29}$ that maximizes $|\det(H)|/(2^{28}\cdot7^{12}\cdot342)$. Following the convention of \citet{Wang2025ThetaEvolveTL}, the best-known construction with $|\det(H)|=2^{28}\cdot7^{12}\cdot320$ scores $320/342\approx0.9357$.

\subsection{Baselines}
\label{appdx:baselines}

\textbf{\scs.} We use the implementation of \scs~\citep{Gideoni2026SimpleBA}, which generates each program from a prompt containing the task instruction and $3$ example programs sampled uniformly at random without replacement from the archive. \scs does not maintain an idea memory, and its prompt does not contain the summary of the idea memory used by \ours.

\textbf{SCS+meta.} SCS+meta adds the meta-summary of \shinka to the prompt of \scs, i.e., a textual summary of the insights of recent programs with at most $5$ recommendations, which is rewritten by DeepSeek-V4-Flash at temperature $0$ after every generation, without the idea memory, the idea-guided actions, and the exploration at plateaus of \ours. It is run on First autocorrelation, Heilbronn, and Kissing with $10$ generations of $20$ programs, and on Uncertainty inequality with $20$ generations of $10$ programs.

\textbf{\shinka.} We run \shinka~\citep{lange2025shinkaevolve} with $5$ islands and migration every $10$ generations, weighted parent sampling, an archive of $40$ programs, a meta-summary every $10$ generations with $5$ recommendations, and both diff-based and full rewrites. \shinka generates one program per generation, and does not distinguish the two configurations of \ours and \scs.

\subsection{Budget and Resources}
\label{appdx:budget}

All methods evaluate each generated program once, and a failed program consumes the budget. The LLM calls of \ours for proposing, annotating, probing, and judging idea factors do not evaluate programs and do not consume the budget. Each run of \shinka starts from $5$ copies of the initial program, one for each island, and then generates $199$ programs, so its final performance is reported at $199$ generated programs. The evaluations of \ours and \scs are not limited in CPUs, while the runs of \shinka on First autocorrelation, Second autocorrelation, Kissing, and Max--min distance ratio evaluate each program with at most $8$ CPUs. The LLM is accessed either through the official API of DeepSeek or through OpenRouter, depending on the time of the run.

\section{Additional Results}
\label{appdx:results}

\textbf{Per-seed results.} Tables~\ref{tab:perseed} and~\ref{tab:perseed_cont} report the final best score of each seed for the cells of Table~\ref{tab:main_results}, and Tables~\ref{tab:budget} and~\ref{tab:budget_cont} report the best-so-far scores along the budget.

\begin{table}[p]
\centering
\caption{Final best score of each seed at the budget of $200$, and the mean $\pm$ standard deviation over seeds.}
\label{tab:perseed}
\resizebox{\linewidth}{!}{%
\begin{tabular}{@{}lllcccc@{}}
\toprule
Task & Method & Config. & Seed 1 & Seed 2 & Seed 3 & Mean \\
\midrule
\multirow{5}{*}{Hadamard $\uparrow$} & \ours & $10\times20$ & 0.9211 & 0.9357 & 0.9211 & 0.9259 $\pm$ 0.0084 \\
 & \ours & $20\times10$ & 0.9211 & 0.9211 & 0.9211 & 0.9211 $\pm$ 0.0000 \\
 & \scs & $10\times20$ & 0.8596 & 0.9211 & 0.7240 & 0.8349 $\pm$ 0.1008 \\
 & \scs & $20\times10$ & 0.8596 & 0.5665 & 0.5662 & 0.6641 $\pm$ 0.1693 \\
 & \shinka & -- & 0.5603 & 0.5733 & 0.8066 & 0.6467 $\pm$ 0.1386 \\
\midrule
\multirow{5}{*}{Pack $\uparrow$} & \ours & $10\times20$ & 2.63304 & 2.63598 & 2.63595 & 2.63499 $\pm$ 0.00169 \\
 & \ours & $20\times10$ & 2.63598 & 2.63429 & 2.63598 & 2.63542 $\pm$ 0.00098 \\
 & \scs & $10\times20$ & 2.62760 & 2.63104 & 2.61050 & 2.62305 $\pm$ 0.01100 \\
 & \scs & $20\times10$ & 2.63017 & 2.63572 & 2.58501 & 2.61697 $\pm$ 0.02781 \\
 & \shinka & -- & 2.61408 & 2.54142 & 2.59052 & 2.58201 $\pm$ 0.03707 \\
\midrule
\multirow{5}{*}{Linprog $\uparrow$} & \ours & $10\times20$ & 2.63598 & 2.63429 & 2.63598 & 2.63542 $\pm$ 0.00098 \\
 & \ours & $20\times10$ & 2.63429 & 2.63598 & 2.63598 & 2.63542 $\pm$ 0.00098 \\
 & \scs & $10\times20$ & 2.55342 & 2.54533 & 2.55608 & 2.55161 $\pm$ 0.00560 \\
 & \scs & $20\times10$ & 2.63598 & 2.63429 & 2.63598 & 2.63542 $\pm$ 0.00098 \\
 & \shinka & -- & 2.51192 & 2.54142 & 2.48969 & 2.51434 $\pm$ 0.02595 \\
\midrule
\multirow{5}{*}{SumDiff $\uparrow$} & \ours & $10\times20$ & 1.12281 & 1.12142 & 1.12250 & 1.12224 $\pm$ 0.00073 \\
 & \ours & $20\times10$ & 1.11937 & 1.10954 & 1.11937 & 1.11609 $\pm$ 0.00567 \\
 & \scs & $10\times20$ & 1.11033 & 1.12006 & 1.11786 & 1.11608 $\pm$ 0.00510 \\
 & \scs & $20\times10$ & 1.11143 & 1.11123 & 1.10995 & 1.11087 $\pm$ 0.00080 \\
 & \shinka & -- & 1.10954 & 1.10954 & 1.10954 & 1.10954 $\pm$ 0.00000 \\
\midrule
\multirow{5}{*}{Heilbronn $\uparrow$} & \ours & $10\times20$ & 0.0359 & 0.0365 & 0.0351 & 0.0358 $\pm$ 0.0007 \\
 & \ours & $20\times10$ & 0.0357 & 0.0359 & 0.0342 & 0.0353 $\pm$ 0.0009 \\
 & \scs & $10\times20$ & 0.0343 & 0.0333 & 0.0331 & 0.0336 $\pm$ 0.0006 \\
 & \scs & $20\times10$ & 0.0340 & 0.0288 & 0.0312 & 0.0313 $\pm$ 0.0026 \\
 & \shinka & -- & 0.0303 & 0.0324 & 0.0283 & 0.0303 $\pm$ 0.0021 \\
\midrule
\multirow{5}{*}{Min overlap $\downarrow$} & \ours & $10\times20$ & 0.38096 & 0.38102 & 0.38095 & 0.38098 $\pm$ 0.00004 \\
 & \ours & $20\times10$ & 0.38096 & 0.38095 & 0.38100 & 0.38097 $\pm$ 0.00003 \\
 & \scs & $10\times20$ & 0.38130 & 0.38094 & 0.38097 & 0.38107 $\pm$ 0.00020 \\
 & \scs & $20\times10$ & 0.38104 & 0.38210 & 0.38154 & 0.38156 $\pm$ 0.00053 \\
 & \shinka & -- & 0.38239 & 0.38243 & 0.38246 & 0.38242 $\pm$ 0.00003 \\
\bottomrule
\end{tabular}}
\end{table}

\begin{table}[p]
\centering
\caption{Final best score of each seed (continued) at the budget of $200$ evaluated programs ($199$ for \shinka), and the mean $\pm$ standard deviation over seeds.}
\label{tab:perseed_cont}
\resizebox{\linewidth}{!}{%
\begin{tabular}{@{}lllcccc@{}}
\toprule
Task & Method & Config. & Seed 1 & Seed 2 & Seed 3 & Mean \\
\midrule
\multirow{5}{*}{Kissing $\uparrow$} & \ours & $10\times20$ & 582 & 582 & 582 & 582 $\pm$ 0 \\
 & \ours & $20\times10$ & 414 & 582 & 582 & 526 $\pm$ 97 \\
 & \scs & $10\times20$ & 582 & 582 & 582 & 582 $\pm$ 0 \\
 & \scs & $20\times10$ & 406 & 438 & 394 & 413 $\pm$ 23 \\
 & \shinka & -- & 414 & 406 & 566 & 462 $\pm$ 90 \\
\midrule
\multirow{5}{*}{Second autocorr. $\uparrow$} & \ours & $10\times20$ & 0.8825 & 0.8825 & 0.8825 & 0.8825 $\pm$ 0.0000 \\
 & \ours & $20\times10$ & 0.8825 & 0.8825 & 0.8825 & 0.8825 $\pm$ 0.0000 \\
 & \scs & $10\times20$ & 0.8825 & 0.8825 & 0.8825 & 0.8825 $\pm$ 0.0000 \\
 & \scs & $20\times10$ & 0.8499 & 0.8485 & 0.8210 & 0.8398 $\pm$ 0.0163 \\
 & \shinka & -- & 0.8739 & 0.8825 & 0.8671 & 0.8745 $\pm$ 0.0077 \\
\midrule
\multirow{5}{*}{First autocorr. $\downarrow$} & \ours & $10\times20$ & 1.5151 & 1.5176 & 1.5257 & 1.5195 $\pm$ 0.0055 \\
 & \ours & $20\times10$ & 1.5104 & 1.5098 & 1.5137 & 1.5113 $\pm$ 0.0021 \\
 & \scs & $10\times20$ & 1.5587 & 1.5560 & 1.5378 & 1.5508 $\pm$ 0.0113 \\
 & \scs & $20\times10$ & 1.5151 & 1.5354 & 1.5202 & 1.5235 $\pm$ 0.0106 \\
 & \shinka & -- & 1.5203 & 1.5708 & 1.5708 & 1.5540 $\pm$ 0.0292 \\
\midrule
\multirow{5}{*}{Max--min ratio $\downarrow$} & \ours & $10\times20$ & 12.8892 & 12.8892 & 12.8892 & 12.8892 $\pm$ 0.0000 \\
 & \ours & $20\times10$ & 13.0000 & 12.8892 & 12.8892 & 12.9262 $\pm$ 0.0640 \\
 & \scs & $10\times20$ & 12.9425 & 13.0000 & 13.0000 & 12.9808 $\pm$ 0.0332 \\
 & \scs & $20\times10$ & 13.0229 & 12.8892 & 12.8892 & 12.9338 $\pm$ 0.0772 \\
 & \shinka & -- & 13.0000 & 13.0000 & 13.0000 & 13.0000 $\pm$ 0.0000 \\
\midrule
\multirow{5}{*}{Hermite $\downarrow$} & \ours & $10\times20$ & 0.34820 & 0.35172 & 0.35274 & 0.35089 $\pm$ 0.00238 \\
 & \ours & $20\times10$ & 0.35321 & 0.35263 & 0.35834 & 0.35473 $\pm$ 0.00314 \\
 & \scs & $10\times20$ & 0.35394 & 0.36109 & 0.36365 & 0.35956 $\pm$ 0.00503 \\
 & \scs & $20\times10$ & 0.46232 & 0.35945 & 0.36116 & 0.39431 $\pm$ 0.05890 \\
 & \shinka & -- & 6.91428 & 0.35429 & 6.91480 & 4.72779 $\pm$ 3.78756 \\
\bottomrule
\end{tabular}}
\end{table}

\begin{table}[p]
\centering
\caption{Best-so-far scores along the budget, as the mean $\pm$ standard deviation over $3$ seeds after $40$, $80$, $120$, $160$, and $200$ evaluated programs. }
\label{tab:budget}
\resizebox{\linewidth}{!}{%
\begin{tabular}{@{}lllccccc@{}}
\toprule
Task & Method & Config. & $40$ & $80$ & $120$ & $160$ & $200$ \\
\midrule
\multirow{5}{*}{Hadamard $\uparrow$} & \ours & $10\times20$ & 0.8801 $\pm$ 0.0355 & 0.9211 $\pm$ 0.0000 & 0.9259 $\pm$ 0.0084 & 0.9259 $\pm$ 0.0084 & 0.9259 $\pm$ 0.0084 \\
 & \ours & $20\times10$ & 0.9006 $\pm$ 0.0355 & 0.9006 $\pm$ 0.0355 & 0.9211 $\pm$ 0.0000 & 0.9211 $\pm$ 0.0000 & 0.9211 $\pm$ 0.0000 \\
 & \scs & $10\times20$ & 0.5047 $\pm$ 0.0113 & 0.6446 $\pm$ 0.1866 & 0.6494 $\pm$ 0.1828 & 0.7144 $\pm$ 0.1502 & 0.8349 $\pm$ 0.1008 \\
 & \scs & $20\times10$ & 0.6506 $\pm$ 0.1821 & 0.6506 $\pm$ 0.1821 & 0.6556 $\pm$ 0.1771 & 0.6641 $\pm$ 0.1693 & 0.6641 $\pm$ 0.1693 \\
 & \shinka & -- & 0.4909 $\pm$ 0.0171 & 0.5093 $\pm$ 0.0093 & 0.5297 $\pm$ 0.0258 & 0.5588 $\pm$ 0.0131 & 0.6467 $\pm$ 0.1386 \\
\midrule
\multirow{5}{*}{Pack $\uparrow$} & \ours & $10\times20$ & 2.62275 $\pm$ 0.00820 & 2.62721 $\pm$ 0.00567 & 2.63201 $\pm$ 0.00090 & 2.63499 $\pm$ 0.00169 & 2.63499 $\pm$ 0.00169 \\
 & \ours & $20\times10$ & 2.63046 $\pm$ 0.00347 & 2.63289 $\pm$ 0.00397 & 2.63326 $\pm$ 0.00336 & 2.63485 $\pm$ 0.00098 & 2.63542 $\pm$ 0.00098 \\
 & \scs & $10\times20$ & 2.60036 $\pm$ 0.03729 & 2.60307 $\pm$ 0.03920 & 2.60423 $\pm$ 0.04016 & 2.61452 $\pm$ 0.02569 & 2.62305 $\pm$ 0.01100 \\
 & \scs & $20\times10$ & 2.60046 $\pm$ 0.02223 & 2.61200 $\pm$ 0.02384 & 2.61383 $\pm$ 0.02605 & 2.61697 $\pm$ 0.02781 & 2.61697 $\pm$ 0.02781 \\
 & \shinka & -- & 2.51452 $\pm$ 0.04659 & 2.53382 $\pm$ 0.01316 & 2.56349 $\pm$ 0.03437 & 2.57166 $\pm$ 0.03783 & 2.58201 $\pm$ 0.03707 \\
\midrule
\multirow{5}{*}{Linprog $\uparrow$} & \ours & $10\times20$ & 2.63043 $\pm$ 0.00962 & 2.63435 $\pm$ 0.00282 & 2.63542 $\pm$ 0.00098 & 2.63542 $\pm$ 0.00098 & 2.63542 $\pm$ 0.00098 \\
 & \ours & $20\times10$ & 2.63165 $\pm$ 0.00241 & 2.63444 $\pm$ 0.00148 & 2.63542 $\pm$ 0.00098 & 2.63542 $\pm$ 0.00098 & 2.63542 $\pm$ 0.00098 \\
 & \scs & $10\times20$ & 2.52974 $\pm$ 0.01579 & 2.54272 $\pm$ 0.00225 & 2.54672 $\pm$ 0.00612 & 2.55161 $\pm$ 0.00560 & 2.55161 $\pm$ 0.00560 \\
 & \scs & $20\times10$ & 2.63328 $\pm$ 0.00332 & 2.63486 $\pm$ 0.00098 & 2.63486 $\pm$ 0.00098 & 2.63542 $\pm$ 0.00098 & 2.63542 $\pm$ 0.00098 \\
 & \shinka & -- & 2.46507 $\pm$ 0.09226 & 2.48623 $\pm$ 0.06618 & 2.48853 $\pm$ 0.06724 & 2.51167 $\pm$ 0.02988 & 2.51434 $\pm$ 0.02595 \\
\midrule
\multirow{5}{*}{SumDiff $\uparrow$} & \ours & $10\times20$ & 1.11026 $\pm$ 0.00102 & 1.11364 $\pm$ 0.00470 & 1.11948 $\pm$ 0.00220 & 1.12089 $\pm$ 0.00065 & 1.12224 $\pm$ 0.00073 \\
 & \ours & $20\times10$ & 1.10968 $\pm$ 0.00023 & 1.10968 $\pm$ 0.00023 & 1.11413 $\pm$ 0.00480 & 1.11473 $\pm$ 0.00484 & 1.11609 $\pm$ 0.00567 \\
 & \scs & $10\times20$ & 1.10981 $\pm$ 0.00023 & 1.10990 $\pm$ 0.00009 & 1.11387 $\pm$ 0.00473 & 1.11564 $\pm$ 0.00497 & 1.11608 $\pm$ 0.00510 \\
 & \scs & $20\times10$ & 1.10968 $\pm$ 0.00023 & 1.11031 $\pm$ 0.00099 & 1.11031 $\pm$ 0.00099 & 1.11031 $\pm$ 0.00099 & 1.11087 $\pm$ 0.00080 \\
 & \shinka & -- & 1.10954 $\pm$ 0.00000 & 1.10954 $\pm$ 0.00000 & 1.10954 $\pm$ 0.00000 & 1.10954 $\pm$ 0.00000 & 1.10954 $\pm$ 0.00000 \\
\midrule
\multirow{5}{*}{Heilbronn $\uparrow$} & \ours & $10\times20$ & 0.0315 $\pm$ 0.0008 & 0.0349 $\pm$ 0.0012 & 0.0350 $\pm$ 0.0014 & 0.0354 $\pm$ 0.0014 & 0.0358 $\pm$ 0.0007 \\
 & \ours & $20\times10$ & 0.0336 $\pm$ 0.0016 & 0.0339 $\pm$ 0.0013 & 0.0342 $\pm$ 0.0013 & 0.0349 $\pm$ 0.0010 & 0.0353 $\pm$ 0.0009 \\
 & \scs & $10\times20$ & 0.0310 $\pm$ 0.0025 & 0.0325 $\pm$ 0.0003 & 0.0326 $\pm$ 0.0005 & 0.0331 $\pm$ 0.0002 & 0.0336 $\pm$ 0.0006 \\
 & \scs & $20\times10$ & 0.0294 $\pm$ 0.0021 & 0.0309 $\pm$ 0.0025 & 0.0309 $\pm$ 0.0025 & 0.0311 $\pm$ 0.0022 & 0.0313 $\pm$ 0.0026 \\
 & \shinka & -- & 0.0261 $\pm$ 0.0007 & 0.0268 $\pm$ 0.0012 & 0.0290 $\pm$ 0.0024 & 0.0297 $\pm$ 0.0024 & 0.0303 $\pm$ 0.0021 \\
\midrule
\multirow{5}{*}{Min overlap $\downarrow$} & \ours & $10\times20$ & 0.38109 $\pm$ 0.00012 & 0.38107 $\pm$ 0.00010 & 0.38101 $\pm$ 0.00007 & 0.38098 $\pm$ 0.00004 & 0.38098 $\pm$ 0.00004 \\
 & \ours & $20\times10$ & 0.38121 $\pm$ 0.00009 & 0.38116 $\pm$ 0.00016 & 0.38106 $\pm$ 0.00014 & 0.38105 $\pm$ 0.00013 & 0.38097 $\pm$ 0.00003 \\
 & \scs & $10\times20$ & 0.38141 $\pm$ 0.00029 & 0.38114 $\pm$ 0.00017 & 0.38108 $\pm$ 0.00019 & 0.38107 $\pm$ 0.00020 & 0.38107 $\pm$ 0.00020 \\
 & \scs & $20\times10$ & 0.38351 $\pm$ 0.00315 & 0.38344 $\pm$ 0.00323 & 0.38344 $\pm$ 0.00323 & 0.38156 $\pm$ 0.00053 & 0.38156 $\pm$ 0.00053 \\
 & \shinka & -- & 0.38285 $\pm$ 0.00065 & 0.38242 $\pm$ 0.00003 & 0.38242 $\pm$ 0.00003 & 0.38242 $\pm$ 0.00003 & 0.38242 $\pm$ 0.00003 \\
\bottomrule
\end{tabular}}
\end{table}

\begin{table}[p]
\centering
\caption{Best-so-far scores along the budget (continued), as the mean $\pm$ standard deviation over $3$ seeds after $40$, $80$, $120$, $160$, and $200$ evaluated programs. For \shinka, the last column reports $199$ programs.}
\label{tab:budget_cont}
\resizebox{\linewidth}{!}{%
\begin{tabular}{@{}lllccccc@{}}
\toprule
Task & Method & Config. & $40$ & $80$ & $120$ & $160$ & $200$ \\
\midrule
\multirow{5}{*}{Kissing $\uparrow$} & \ours & $10\times20$ & 578 $\pm$ 6 & 582 $\pm$ 0 & 582 $\pm$ 0 & 582 $\pm$ 0 & 582 $\pm$ 0 \\
 & \ours & $20\times10$ & 348 $\pm$ 0 & 453 $\pm$ 113 & 523 $\pm$ 103 & 523 $\pm$ 103 & 526 $\pm$ 97 \\
 & \scs & $10\times20$ & 564 $\pm$ 3 & 582 $\pm$ 0 & 582 $\pm$ 0 & 582 $\pm$ 0 & 582 $\pm$ 0 \\
 & \scs & $20\times10$ & 394 $\pm$ 38 & 394 $\pm$ 38 & 402 $\pm$ 34 & 402 $\pm$ 33 & 413 $\pm$ 23 \\
 & \shinka & -- & 426 $\pm$ 68 & 459 $\pm$ 92 & 459 $\pm$ 92 & 461 $\pm$ 91 & 462 $\pm$ 90 \\
\midrule
\multirow{5}{*}{Second autocorr. $\uparrow$} & \ours & $10\times20$ & 0.8825 $\pm$ 0.0000 & 0.8825 $\pm$ 0.0000 & 0.8825 $\pm$ 0.0000 & 0.8825 $\pm$ 0.0000 & 0.8825 $\pm$ 0.0000 \\
 & \ours & $20\times10$ & 0.8649 $\pm$ 0.0154 & 0.8823 $\pm$ 0.0004 & 0.8824 $\pm$ 0.0002 & 0.8825 $\pm$ 0.0000 & 0.8825 $\pm$ 0.0000 \\
 & \scs & $10\times20$ & 0.8825 $\pm$ 0.0000 & 0.8825 $\pm$ 0.0000 & 0.8825 $\pm$ 0.0000 & 0.8825 $\pm$ 0.0000 & 0.8825 $\pm$ 0.0000 \\
 & \scs & $20\times10$ & 0.8492 $\pm$ 0.0010 & 0.8398 $\pm$ 0.0163 & 0.8398 $\pm$ 0.0163 & 0.8398 $\pm$ 0.0163 & 0.8398 $\pm$ 0.0163 \\
 & \shinka & -- & 0.8577 $\pm$ 0.0215 & 0.8650 $\pm$ 0.0172 & 0.8732 $\pm$ 0.0096 & 0.8732 $\pm$ 0.0096 & 0.8745 $\pm$ 0.0077 \\
\midrule
\multirow{5}{*}{First autocorr. $\downarrow$} & \ours & $10\times20$ & 1.5528 $\pm$ 0.0113 & 1.5316 $\pm$ 0.0233 & 1.5217 $\pm$ 0.0093 & 1.5216 $\pm$ 0.0091 & 1.5195 $\pm$ 0.0055 \\
 & \ours & $20\times10$ & 1.5196 $\pm$ 0.0021 & 1.5169 $\pm$ 0.0053 & 1.5124 $\pm$ 0.0023 & 1.5119 $\pm$ 0.0019 & 1.5113 $\pm$ 0.0021 \\
 & \scs & $10\times20$ & 1.5674 $\pm$ 0.0056 & 1.5626 $\pm$ 0.0076 & 1.5549 $\pm$ 0.0165 & 1.5508 $\pm$ 0.0113 & 1.5508 $\pm$ 0.0113 \\
 & \scs & $20\times10$ & 1.5482 $\pm$ 0.0284 & 1.5262 $\pm$ 0.0143 & 1.5262 $\pm$ 0.0143 & 1.5235 $\pm$ 0.0106 & 1.5235 $\pm$ 0.0106 \\
 & \shinka & -- & 1.5664 $\pm$ 0.0076 & 1.5664 $\pm$ 0.0076 & 1.5622 $\pm$ 0.0149 & 1.5554 $\pm$ 0.0268 & 1.5540 $\pm$ 0.0292 \\
\midrule
\multirow{5}{*}{Max--min ratio $\downarrow$} & \ours & $10\times20$ & 12.8892 $\pm$ 0.0000 & 12.8892 $\pm$ 0.0000 & 12.8892 $\pm$ 0.0000 & 12.8892 $\pm$ 0.0000 & 12.8892 $\pm$ 0.0000 \\
 & \ours & $20\times10$ & 12.9266 $\pm$ 0.0636 & 12.9262 $\pm$ 0.0640 & 12.9262 $\pm$ 0.0640 & 12.9262 $\pm$ 0.0640 & 12.9262 $\pm$ 0.0640 \\
 & \scs & $10\times20$ & 14.8410 $\pm$ 1.5646 & 13.8205 $\pm$ 1.4212 & 13.0557 $\pm$ 0.0965 & 12.9808 $\pm$ 0.0332 & 12.9808 $\pm$ 0.0332 \\
 & \scs & $20\times10$ & 13.4169 $\pm$ 0.9139 & 13.2822 $\pm$ 0.6807 & 13.2420 $\pm$ 0.6110 & 12.9629 $\pm$ 0.1277 & 12.9338 $\pm$ 0.0772 \\
 & \shinka & -- & 13.0000 $\pm$ 0.0000 & 13.0000 $\pm$ 0.0000 & 13.0000 $\pm$ 0.0000 & 13.0000 $\pm$ 0.0000 & 13.0000 $\pm$ 0.0000 \\
\midrule
\multirow{5}{*}{Hermite $\downarrow$} & \ours & $10\times20$ & 0.39730 $\pm$ 0.06957 & 0.35841 $\pm$ 0.00503 & 0.35738 $\pm$ 0.00659 & 0.35117 $\pm$ 0.00203 & 0.35089 $\pm$ 0.00238 \\
 & \ours & $20\times10$ & 2.81478 $\pm$ 2.14897 & 1.51139 $\pm$ 1.88727 & 0.40779 $\pm$ 0.06331 & 0.35505 $\pm$ 0.00285 & 0.35473 $\pm$ 0.00314 \\
 & \scs & $10\times20$ & 1.14201 $\pm$ 1.25089 & 0.35956 $\pm$ 0.00503 & 0.35956 $\pm$ 0.00503 & 0.35956 $\pm$ 0.00503 & 0.35956 $\pm$ 0.00503 \\
 & \scs & $20\times10$ & 2.42124 $\pm$ 2.13014 & 0.96193 $\pm$ 0.95033 & 0.39580 $\pm$ 0.05769 & 0.39431 $\pm$ 0.05890 & 0.39431 $\pm$ 0.05890 \\
 & \shinka & -- & 4.72792 $\pm$ 3.78744 & 4.72791 $\pm$ 3.78744 & 4.72788 $\pm$ 3.78741 & 4.72783 $\pm$ 3.78749 & 4.72779 $\pm$ 3.78756 \\
\bottomrule
\end{tabular}}
\end{table}

\textbf{Sample efficiency.} For each setting, $n_{\ours}$ is the first number of evaluated programs at which the best-so-far score of \ours, averaged over the seeds that have produced a successful program, reaches the final best score of \scs averaged over seeds, and $n_{\scs}$ is the first number of evaluated programs at which the best-so-far score of \scs, averaged in the same way, reaches its final value. The programs within a generation are ordered by their slots. Table~\ref{tab:sample_eff} reports the four tasks in Fig.~\ref{fig:curves}.

\begin{table}[h]
\centering
\caption{Numbers of evaluated programs that \ours and \scs need to reach the final best score of \scs, averaged over $3$ seeds.}
\label{tab:sample_eff}
\begin{tabular}{llcccc}
\toprule
Task & Config. & Final of \scs & $n_{\scs}$ & $n_{\ours}$ & $n_{\scs}/n_{\ours}$ \\
\midrule
Hadamard & $10\times20$ & $0.8349$ & $186$ & $10$ & $18.6$ \\
Hadamard & $20\times10$ & $0.6641$ & $131$ & $5$ & $26.2$ \\
SumDiff & $10\times20$ & $1.11608$ & $180$ & $81$ & $2.2$ \\
SumDiff & $20\times10$ & $1.11087$ & $181$ & $97$ & $1.9$ \\
Heilbronn & $10\times20$ & $0.0336$ & $199$ & $52$ & $3.8$ \\
Heilbronn & $20\times10$ & $0.0313$ & $169$ & $27$ & $6.3$ \\
Min overlap & $10\times20$ & $0.38107$ & $141$ & $84$ & $1.7$ \\
Min overlap & $20\times10$ & $0.38156$ & $144$ & $13$ & $11.1$ \\
\bottomrule
\end{tabular}
\end{table}

\section{Prompts}
\label{appdx:prompts}

\textbf{Prompt of a slot.} The prompt of each slot is assembled as follows, where the strategy instruction is present only for idea-guided slots, and the summary of the idea memory is present for all slots.
\begin{tcolorbox}[colback=gray!5,colframe=gray!50,boxrule=0.4pt,breakable]
\scriptsize
\begin{verbatim}
{base_instruction}

# Strategy for this attempt
{directive}

{causal_info_block}

# Successful examples (to improve on):

## Solution {i+1} ({metric} = {value}):
```
{code}
```

Please now provide your new, better implementation of the required
function (exact signature from the task statement above). Address the
strategy above and avoid the failure patterns shown.
\end{verbatim}
\end{tcolorbox}

\textbf{Strategy instructions.} The strategy instruction of an action on a binary idea factor and on a categorical idea factor are as follows, with lines wrapped for display.
\begin{tcolorbox}[colback=gray!5,colframe=gray!50,boxrule=0.4pt,breakable]
\scriptsize
\begin{verbatim}
STRATEGY (this batch is targeting a discovered code factor):
{USE|AVOID} the `{feature_name}` pattern{ ({desc})}. Past data on this
run estimates its effect on {conf_metric_name} at ATE={+}{ate:.3f}.
The first example below is the highest-Y program (anchor, regardless
of this pattern); subsequent examples {USE this pattern|DO NOT use this
pattern}. Keep the pattern direction while improving on the anchor's
overall quality.
\end{verbatim}
\end{tcolorbox}
\begin{tcolorbox}[colback=gray!5,colframe=gray!50,boxrule=0.4pt,breakable]
\scriptsize
\begin{verbatim}
STRATEGY (this batch is targeting a categorical code factor): SET the
`{feature_name}` factor to value `{level}`{ ({desc})}. Past data shows
programs with this exact level gain {+}{ate:.3f} {conf_metric_name} vs
other levels of the factor. The first example below is the highest-Y
program (anchor); the other examples USE this exact level. Adopt the
`{level}` setting and improve on the anchor.
\end{verbatim}
\end{tcolorbox}

\textbf{Stall note.} At plateaus, the following note is prepended to the prompt of every slot, where \texttt{\{blur\}} lists the dependent idea factors as their names followed by the first $90$ characters of their descriptions, separated by semicolons, and \texttt{\{stall\}} is ``best has stayed at \texttt{\{best\}} for \texttt{\{streak\}} generations''.
\begin{tcolorbox}[colback=gray!5,colframe=gray!50,boxrule=0.4pt,breakable]
\scriptsize
\begin{verbatim}
CURRENT-STATE NOTE (read first): the most common combination among the
current top programs is: {blur}. Progress with this combination has
stalled — {stall}. Consider: are there OTHER constructions or
combinations, or variants of the current one, that could reach higher?
If you know one, this attempt is a good moment to try it.
\end{verbatim}
\end{tcolorbox}

\textbf{Annotator.} The annotator LLM labels the idea factors of a program with the following prompt, where the descriptions are truncated to $200$ characters and the source code to $4000$ characters.
\begin{tcolorbox}[colback=gray!5,colframe=gray!50,boxrule=0.4pt,breakable]
\scriptsize
\begin{verbatim}
[system]
You annotate program code for the presence of specific algorithmic
factors. Output STRICT JSON only with key 'assignments'.

[user]
Factors to annotate:
- <name> (binary): <description>
- <name> (categorical, values: [...]): <description>

Program code:
```python
<code>
```

Output JSON: { "assignments": { "<factor_name>": <value>, ... } }
For binary factors, value should be true or false. For categorical,
value should be one of the listed values. Use null only if unknown.
\end{verbatim}
\end{tcolorbox}

\textbf{Background probe.} At the entry of a plateau, the background probe reads the $8$ successful programs with the highest scores, with the source code truncated to $1500$ characters, and the names of the current idea factors.
\begin{tcolorbox}[colback=gray!5,colframe=gray!50,boxrule=0.4pt,breakable]
\scriptsize
\begin{verbatim}
These are the TOP-SCORING programs of a search that has plateaued —
they are the stuck family.

## Program {i} ({metric}={score}):
```
{code}
```
...

Factors already tracked (do NOT repeat these or trivial rephrasings):
- {name}
...

Name the structural / algorithmic properties that these top programs
SHARE — especially the construction or representation they all rely
on. These will be used to describe the stuck family so the search can
deliberately depart from it. Binary properties only.
Return ONLY a JSON array: [{"name":"<snake_case>","description":
"<one sentence, concrete>"}] (2 to 4 entries)
\end{verbatim}
\end{tcolorbox}

\subsection{More Details of Experiments on Machine Learning Engineering}
\label{appdx:mle-experiments}

We further study whether the effectiveness of \ours transfers to machine learning engineering (MLE). We integrate its idea discovery and search control into OpenRSI~\citep{yang2026frontis} and evaluate the resulting agent on MLS-Bench~\citep{lyu2026mlsbench}. This setting examines the use of idea factors within a harness that supports evolution on ML algorithms towards a specific task objective.

\textbf{Tasks and data.}
MLS-Bench evaluates agents on improving components of ML algorithms under fixed task interfaces and evaluation protocols. Here we consider \textbf{6} tasks: \emph{Online bandit}, \emph{Anomaly detection}, \emph{Clustering}, \emph{Multi-objective optimization}, \emph{Subgroup calibration shift}, and \emph{Evolution strategy}.
Search leverages the scores and execution feedback exposed by each task.

\textbf{Backbone and budget.}
We use the evolutionary harness from OpenRSI, which organizes search through \textsc{Draft}, \textsc{Improve}, \textsc{Debug}, and \textsc{Crossover} operators.
For each task, we retain a fixed history containing 20 completed search slots from an earlier OpenRSI run. The OpenRSI baseline and OpenRSI equipped with \ours continue from this same history for \textbf{50} additional evaluator calls, with \textbf{3} seeds.  Both methods use \textbf{5} search slots per generation.
All program generation uses Qwen3.6-35B-A3B.
The same checkpoint performs factor discovery and annotation in \ours.

\textbf{Integrating idea factors with OpenRSI.}
\ours initializes its idea memory from the shared history and updates it with newly evaluated programs. Consistent with Sec.~\ref{sec:method}, Discovery examines program pairs whose score differences are insufficiently explained by the current factors. The LLM reads their source and proposes reusable mechanisms.
Factors $z$ enter the action catalog when both states have sufficient support and exhibit a score contrast in the archive.
Each generation reserves two slots for \textsc{Open} exploration and schedules three for \textsc{Guided} search.
Specifically, A guided slot first selects an action, such as $\mathsf{ADD}(z)$ or $\mathsf{REMOVE}(z)$, and then the harness selects a high-scoring parent compatible with that change.
We use improvement over the best selected parent as action feedback, adapting the reward reference to OpenRSI's parent--child structure.
When the cumulative best score remains unchanged within relative tolerance $10^{-4}$ across three generation boundaries, \ours  switches to idea-guided exploration at plateaus, aligning with Sec.~\ref{sec:method_guide}.

\textbf{Evaluation and results.} We evaluate search performance using the best-so-far (BSF) score.
For each task and method, we average the final BSF score over three seeds, starting from the same 20-slot history and continuing for 50 additional evaluations. Fig.~\ref{fig:mle-final} compares OpenRSI with OpenRSI+\textsc{ExpEvolve} across the six tasks. The reported configurations improve the mean final BSF score on all six tasks, with the largest gains on Online bandit and Anomaly detection.

\clearpage
\section*{Ethics Statement}
This work studies evolutionary coding agents on open-ended mathematical discovery tasks and on machine learning tasks from MLS-Bench. It does not involve human subjects or personally identifiable information, and all tasks come from public benchmarks or prior work.

\section*{Reproducibility Statement}
The full procedure of \ours is given in Alg.~\ref{alg:expevolve} and Alg.~\ref{alg:discover}, and its implementation details and hyperparameters are given in Appendix~\ref{appdx:method} and Table~\ref{tab:expevolve_hparams}. The tasks, evaluators, and reference values are described in Appendix~\ref{appdx:tasks}, the configurations of the baselines in Appendix~\ref{appdx:baselines}, the budget and resources in Appendix~\ref{appdx:budget}, the setting of MLS-Bench in Appendix~\ref{appdx:mle-experiments}, and all prompts in Appendix~\ref{appdx:prompts}. The per-seed results are given in Appendix~\ref{appdx:results}. The formal statements and proofs of the theoretical results are given in Appendix~\ref{appdx:thm1-static} and Appendix~\ref{appdx:prop2}. All experiments on mathematical discovery use DeepSeek-V4-Flash unless stated otherwise, and each setting is run with $3$ seeds.

\section*{AI Use Statement}
We used LLM assistance to revise the manuscript's organization and wording, identify and check related references, and examine the interpretation of reported experiments. The proposed framework itself uses LLMs to generate programs, propose idea factors, and assess program properties.

\end{document}